\documentclass{article} % For LaTeX2e
\usepackage{iclr2027_conference,times}

\usepackage{amsmath,amsfonts,bm}

\def\eqref#1{equation~\ref{#1}}
\def\1{\bm{1}}

\def\rmA{{\mathbf{A}}}

\def\rmX{{\mathbf{X}}}

\DeclareMathAlphabet{\mathsfit}{\encodingdefault}{\sfdefault}{m}{sl}
\SetMathAlphabet{\mathsfit}{bold}{\encodingdefault}{\sfdefault}{bx}{n}

\def\gA{{\mathcal{A}}}

\def\gE{{\mathcal{E}}}

\def\gG{{\mathcal{G}}}

\def\gL{{\mathcal{L}}}

\def\gN{{\mathcal{N}}}

\def\gP{{\mathcal{P}}}

\def\gT{{\mathcal{T}}}

\def\gV{{\mathcal{V}}}

\def\gY{{\mathcal{Y}}}

\newcommand{\E}{\mathbb{E}}

\newcommand{\R}{\mathbb{R}}

\newcommand{\KL}{D_{\mathrm{KL}}}
\newcommand{\Var}{\mathrm{Var}}

\DeclareMathOperator*{\argmin}{arg\,min}

\iclrfinalcopy
\usepackage{amsthm}
\usepackage{algorithm}
\usepackage{algpseudocode}

\newcommand{\TU}{\mathrm{TU}}
\newcommand{\AU}{\mathrm{AU}}
\newcommand{\EU}{\mathrm{EU}}

\newcommand{\I}{\mathbb{I}}

\newcommand{\Prob}{\mathbb{P}}
\newcommand{\Dir}{\mathrm{Dir}}

\newtheorem{theorem}{Theorem}[section]

\newtheorem{lemma}[theorem]{Lemma}
\newtheorem{definition}[theorem]{Definition}

\usepackage{hyperref}
\usepackage{url}

\usepackage{enumitem}

\usepackage{placeins}
\usepackage{booktabs}
\usepackage{longtable}
\usepackage{multirow}
\usepackage{array}

\usepackage{graphicx}
\usepackage{caption}
\usepackage{subcaption}

\usepackage{lipsum}

\title{Rethinking Epistemic Uncertainty in Node Classification through Information Growth}

\author{
Emma Meneghini\textsuperscript{1,2},
Francesco Ferrini\textsuperscript{2},
Bruno Lepri\textsuperscript{2},
Andrea Passerini\textsuperscript{1},
Veronica Lachi\textsuperscript{3} \\[0.5em]
\textsuperscript{1}University of Trento, 
\textsuperscript{2}Fondazione Bruno Kessler, FBK, 
\textsuperscript{3}University of Troms{\o}
}

\begin{document}

\maketitle
\lhead{Preprint}

\begin{abstract}
Epistemic uncertainty should decrease as additional information about the data-generating process (DGP) becomes available to the predictor. Yet, existing graph evidential deep learning (EDL) methods for node classification typically construct epistemic uncertainty from graph-specific properties and evaluate it on downstream tasks such as out-of-distribution detection, which do not test its reducibility as information about the DGP increases. To make reducibility directly testable, we introduce a statistical framework for studying epistemic uncertainty under information growth. Our framework specifies an information-growth experimental protocol and a consistency criterion for epistemic predictors, while using projective graph DGPs to ensure that growing graphs, which in general need not provide increasing information about the same DGP, constitute coherent observations of the same underlying process. We show that EDL methods do not explicitly estimate data uncertainty arising from a single finite graph observation and instead regulate epistemic uncertainty through model hyperparameters, precluding consistency, as corroborated by controlled information-growth experiments. As an alternative, we propose graph bootstrap ensembles, capturing both data and procedural uncertainty through graph resampling and randomized training. Under the same experimental protocol, these ensembles exhibit epistemic uncertainty reduction beyond standard deep ensembles. These findings support bootstrap ensembles as candidate consistent epistemic predictors under information growth.
\end{abstract}

\section{Introduction}
Machine learning models are increasingly expected to provide not only accurate predictions, but also reliable estimates of predictive uncertainty~\citep{Abdar-2021, Galikowski-2023, Wang-2026}. Predictive uncertainty is commonly decomposed into \emph{aleatoric} uncertainty, arising from intrinsic ambiguity in the data-generating process (DGP) and thus irreducible, and \emph{epistemic} uncertainty, arising from limited knowledge of that process and thus reducible with additional data~\citep{Hullermeier-2021, Sale-2026b}. In classical supervised learning with independent and identically distributed (i.i.d.)\ observations, under suitable regularity conditions, epistemic uncertainty is operationally characterized to decrease monotonically as information increases and eventually vanish~\citep{Bengs-2022, Jurgens-2024, Shen-2024}.

Extending this characterization to node classification on graphs is non-trivial. While increasing sample size provides a natural information-growth mechanism in the i.i.d.\ setting, no analogous mechanism is readily available for node classification. Intuitively, information growth would involve progressively larger graph observations and training sets becoming available, allowing predictors to be retrained on more informative observations for the same node classification problem. However, nodes are dependent, and observing more of them does not necessarily yield more information about the same underlying DGP \citep{Shalizi-2013}. Consequently, node classification currently lacks a framework for assessing the reducibility of epistemic uncertainty under information growth.

Graph uncertainty quantification has been studied using a variety of approaches (see~\citet{Wang-2024, Chen-2026} for reviews). Among these, evidential deep learning (EDL) methods~\citep{Zhao-2020, Stadler-2021, Hart-2023, Damke-2024a, Damke-2024b} are particularly relevant because they make the epistemic interpretation a central modeling objective, prescribing appropriate uncertainty behavior through graph-specific principles, including reducibility through \emph{feature support} and \emph{structural support}, and the reducibility or irreducibility due to \emph{neighborhood conflict}. However, these principles are defined with respect to the features and topology of a fixed observed graph rather than under information growth. Empirical evaluations likewise assess epistemic uncertainty through its utility for out-of-distribution (OOD) detection, whose optimal predictors need not coincide with predictors of reducible epistemic uncertainty~\citep{Paplham-2026}. While informative about task-specific downstream utility, such evaluations therefore cannot establish whether an uncertainty score has this epistemic interpretation.

Together, these limitations raise fundamental questions about how epistemic uncertainty should be defined and modeled in node classification. We investigate these issues through three research questions: \textbf{(RQ1)} What constitutes a principled statistical framework for epistemic predictors in node classification? \textbf{(RQ2)} Are graph EDL models consistent epistemic predictors, and, if not, what are their fundamental limitations? \textbf{(RQ3)} Which alternative approach can address these limitations? 

\textbf{Contributions.} We address these questions through three main contributions. 

\noindent\textbf{(1)} We establish the first statistical framework for epistemic uncertainty under information growth in node classification, comprising an information-growth experimental protocol and a definition of consistent epistemic predictors. The former naturally requires \emph{projective} graph DGPs to preserve the underlying node-classification problem across nested graphs and learning tasks~\citep{Shalizi-2013}, while the latter requires decreasing and asymptotically vanishing uncertainty. Together, these provide a statistical basis for assessing epistemic reducibility beyond conventional evaluations.

\noindent\textbf{(2)} We unify graph EDL objectives and characterize their population optima, showing how model hyperparameters artificially regulate epistemic uncertainty in the absence of explicit estimation of data uncertainty~\citep{Shen-2024}, arising from observing a single finite realization of the DGP, ultimately precluding epistemic consistency. Controlled information-growth experiments on projective graphs confirm insufficient epistemic concentration and persistent parameterization dependence, corroborating the structural limitations that undermine their use as consistent epistemic predictors.

\noindent\textbf{(3)} Motivated by evidence that bootstrap-based ensembles capture data and procedural uncertainty~\citep{Jimenez-2026}, thus enabling decreasing epistemic uncertainty in the i.i.d.\ setting~\citep{ Jain-2026}, we extend this approach to graphs through \emph{graph bootstrap ensembles}. Under the same protocol as in (2), the experiments support these ensembles as a family of candidate consistent epistemic predictors, demonstrating the ability to recover epistemic concentration absent from standard deep ensembles.

Together, these contributions establish a statistical foundation for understanding epistemic uncertainty in node classification, reveal the limitations of existing evidential approaches when assessed through information growth, and introduce graph bootstrap ensembles as a theoretically motivated, empirically supported alternative for modeling epistemic uncertainty that contracts under information growth.

\section{Background on Predictive Uncertainty and Motivation}
\label{sec:unc-background}
Let \( \gG=(\gV,\gE) \) be an undirected attributed graph with \( |\gV|=N \), adjacency matrix \( A \in \{0,1\}^{N\times N} \), and node feature matrix \( X \in \R^{N\times D} \). Let \( \gY = \{ 1,\ldots,C \} \) denote the set of class labels. In semi-supervised node classification, a set of labeled nodes \( \gT^{(\mathrm{tr})} \subset \gV \) is used to predict the labels of the remaining nodes \( \gT^{(\mathrm{te})} = \gV \setminus \gT^{(\mathrm{tr})} \). We consider quantifying the uncertainty of these predictions. A probabilistic classifier associates each node \( i \in \gV \) with a predictive categorical distribution \( \theta^{(i)} \in \Delta_C \), where \( \Delta_C = \{ \theta \in \R^C: \theta_c \ge 0, \; \sum_{c=1}^C \theta_c = 1 \} \) denotes the probability simplex over \( \gY \). We adopt a second-order probabilistic framework in which the predictive distribution is itself regarded as a random variable \( \Theta^{(i)} \). The second-order predictor \( Q^{(i)} \) specifies a distribution over predictive distributions \( \Theta^{(i)} \sim Q^{(i)} \),  with \( Q : (\gG,i) \mapsto \gP(\Delta_C) \). We use the standard information-theoretic decomposition of total predictive uncertainty into aleatoric and epistemic components~\citep{Depeweg-2019, Wimmer-2023}. Since the following quantities are defined node-wise, we omit the superscript \(^{(i)}\) and define the predictive mean by \( \bar{\theta} = \E_{Q}[\Theta] \). The total uncertainty (TU) is the Shannon entropy of the predictive mean,
\begin{equation}
    \TU(Q) = H(\E_Q[\Theta]) = - \sum_{c=1}^{C} \bar{\theta}_c \log \bar{\theta}_c.
\end{equation}
Aleatoric uncertainty (AU) is the irreducible uncertainty about the class label remaining after conditioning on the predictive distribution and is given by the conditional entropy \( H(Y| \Theta) \),
\begin{equation}
    \AU(Q) = H(Y | \Theta) = \E_{Q}[H(\theta)] = - \int_{\Delta_C} \sum_{c=1}^{C} \theta_c \log \theta_c \, dQ(\theta).
\end{equation}
Epistemic uncertainty (EU) quantifies uncertainty about which predictive distribution should be assigned to the node and is given by the mutual information (MI) \( I \) between \( Y \) and \( \Theta \),
\begin{equation}
\EU(Q) = I(Y;\Theta) = \TU(Q)-\AU(Q) = \E_{Q} \left[ \KL (\theta \| \bar{\theta}) \right],
\end{equation}
where  \( \KL (\cdot \| \cdot) \) denotes the Kullback-Leibler divergence. We next review graph EDL methods, which directly parameterize \( Q \) and are the main second-order baselines considered in this work.

\subsection{Graph EDL Models}
\label{sec:EDL-models}

Evidential node classifiers predict a second-order Dirichlet distribution \( Q_{\psi}^{(i)} = \Dir(\theta; \alpha^{(i)}) \), with \( \alpha^{(i)} = \mathbf{1} + e^{(i)} \) and class-wise evidence \( e^{(i)} \in \R^C_{>0} \). Its mean \( \bar{\theta}^{(i)} = \alpha^{(i)} / \alpha_0^{(i)} \) gives the predictive probabilities, while \( \alpha_0^{(i)} = \sum_{c=1}^{C} \alpha_c^{(i)} \) controls concentration. We organize graph EDL methods by their evidence parameterization and treatment of epistemic uncertainty: reduction through \emph{feature support} and \emph{structural support}, and reduction or irreducibility due to \emph{neighborhood conflict}.

\textbf{Direct parameterization.}
S-BGCN-K~\citep{Zhao-2020} directly predicts Dirichlet evidence with a Bayesian GNN and uses dropout-induced variability for uncertainty estimation. During training, its prediction is regularized toward a node-specific, graph-dependent GKDE prior assigning higher evidence to nodes closer to labeled nodes of the corresponding class, thereby encoding structural support; S-BGCN-T-K additionally introduces a teacher network to improve predictive performance. Assuming the Gaussian likelihood and denoting the one-hot ground-truth label \( e_y^{(i)} \), its objective is
\begin{equation}
\ell_{\mathrm{MSE}}(\psi;i) = \E_{\theta^{(i)} \sim Q_{\psi}^{(i)}} \!\left[ \| \theta^{(i)}-e_y^{(i)} \|^2 \right] + \lambda \KL(Q_{\psi}^{(i)} \| Q_{\mathrm{GKDE}}^{\mathrm{prior},(i)}).
\end{equation}

\textbf{Density parameterization.}
GPN~\citep{Stadler-2021} constructs evidence pseudo-counts from class-conditional densities in latent space, so that low-density regions receive little evidence, implementing feature support~\citep{Charpentier-2020,Charpentier-2022}. These pseudo-counts are propagated through Personalized PageRank, \( \beta_c^{\mathrm{agg},(i)} = \sum_{j \in \gV}\Pi^{\mathrm{PPR}}_{i,j}\beta_c^{\mathrm{ft},(j)} \), yielding \( \alpha^{(i)} = \mathbf{1}+\beta^{\mathrm{agg},(i)} \), where \( \mathbf{1} \) is a uniform Dirichlet prior. Propagation implements structural support by diffusing evidence across neighboring nodes, while accumulating conflicting class evidence within the same Dirichlet implements neighborhood conflict, attributing disagreement to irreducible aleatoric uncertainty. Assuming the categorical likelihood \( p(y^{(i)} | \theta^{(i)}) = \theta_{y^{(i)}} \), GPN minimizes
\begin{equation}
\ell_{\mathrm{UCE}}(\psi;i) = \E_{\theta^{(i)} \sim Q_{\psi}^{(i)}} [-e_y^{(i)\top} \log\theta^{(i)}] - \lambda H[Q_{\psi}^{(i)}].
\end{equation}
Distance-regularized GPN-GD~\citep{Hart-2023} augments this formulation with pairwise regularization of latent representations and evidence.

CUQ-GNN~\citep{Damke-2024b} retains the density-based parameterization and UCE objective but replaces the independent encoder and subsequent propagation step with a GNN encoder, jointly learning feature support, structural support, and the handling of neighborhood conflict. GPN-LOP~\citep{Damke-2024a} instead retains GPN's feature pseudo-counts but aggregates the resulting Dirichlet distributions, \( Q_\psi^{(i)} = \sum_{j \in \gV}\Pi^{\mathrm{PPR}}_{i,j} Q_\psi^{\mathrm{ft},(j)} \), allowing neighborhood disagreement to contribute to epistemic, rather than necessarily aleatoric, uncertainty. It is trained using a tractable upper bound on the UCE loss of the mixture,
\begin{equation}
\ell_{\mathrm{GPN-LOP}}(\psi;i) =
\sum_{j \in \gV} \Pi^{\mathrm{PPR}}_{i,j}
\left(
\E_{\theta^{(j)} \sim Q_{\psi}^{\mathrm{ft},(j)}}
[-e_y^{(i)\top} \log\theta^{(j)}]
- \lambda H[Q_{\psi}^{\mathrm{ft},(j)}]
\right).
\end{equation}

\subsection{Why Evaluate Epistemic Uncertainty under Information Growth?}
\label{sec:motivation}

Epistemic uncertainty is commonly considered useful a priori for downstream tasks such as OOD detection on graphs, as the lack of model knowledge about OOD inputs is expected to result in higher epistemic uncertainty~\citep{Bazhenov-2023, Wang-2024}, and performance on these benchmarks is therefore used to assess the quality of EDL uncertainty estimates. However, our results show that the most effective uncertainty type depends on the particular model and OOD protocol, with AU/TU typically performing best for leave-out-classes (LOC) and feature isolation (FI) with Bernoulli perturbations, and EU for Gaussian FI (Table~\ref{tab:pubmed-ood-auroc}; additional results in Appendix~\ref{app:sec-ood}). That non-epistemic components can outperform EU supports growing criticism that OOD discrimination and the reducibility of epistemic uncertainty under increasing information constitute distinct objectives whose optimal predictors need not coincide~\citep{Li-2025, Paplham-2026}. This motivates evaluating epistemic uncertainty under information growth.

\begin{table}[htbp]
\centering
\scriptsize
\caption{AUROC for OOD detection on PubMed under LOC and FI protocols. Best-performing uncertainty type(s) for each EDL model and OOD protocol are shown in bold.}
\label{tab:pubmed-ood-auroc}
\setlength{\tabcolsep}{4pt}
\begin{tabular}{l*{3}{c}*{3}{c}*{3}{c}}
\toprule
& \multicolumn{3}{c}{LOC}
& \multicolumn{3}{c}{FI (Bernoulli)}
& \multicolumn{3}{c}{FI (Gaussian)} \\
\cmidrule(lr){2-4}
\cmidrule(lr){5-7}
\cmidrule(lr){8-10}
Model
& TU & AU & EU
& TU & AU & EU
& TU & AU & EU \\
\midrule
CUQ-APPNP
& \textbf{0.693} & \textbf{0.693} & 0.643
& \textbf{0.652} & \textbf{0.652} & 0.616
& 0.975 & 0.942 & \textbf{0.994} \\

CUQ-GAT
& \textbf{0.709} & \textbf{0.709} & 0.618
& \textbf{0.586} & \textbf{0.586} & 0.540
& 0.660 & 0.659 & \textbf{0.684} \\

CUQ-GCN
& \textbf{0.717} & \textbf{0.717} & 0.591
& \textbf{0.613} & \textbf{0.613} & 0.554
& 0.936 & 0.858 & \textbf{0.945} \\

GPN
& 0.730 & 0.730 & \textbf{0.769}
& 0.593 & 0.593 & \textbf{0.638}
& 0.560 & 0.560 & \textbf{0.642} \\

GPN-GD\textsubscript{EVI}
& 0.800 & 0.771 & \textbf{0.857}
& 0.585 & 0.579 & \textbf{0.615}
& 0.578 & 0.569 & \textbf{0.621} \\

GPN-GD\textsubscript{LAT}
& 0.724 & 0.724 & \textbf{0.764}
& 0.596 & 0.596 & \textbf{0.640}
& 0.559 & 0.559 & \textbf{0.658} \\

GPN-LOP
& 0.748 & \textbf{0.781} & 0.651
& 0.578 & 0.446 & \textbf{0.689}
& 0.696 & 0.588 & \textbf{0.767} \\

S-BGCN-K
& 0.625 & \textbf{0.627} & 0.500
& 0.646 & \textbf{0.679} & 0.288
& 0.342 & 0.001 & \textbf{1.000} \\

S-BGCN-T-K
& 0.607 & \textbf{0.611} & 0.507
& 0.654 & \textbf{0.677} & 0.308
& 0.447 & 0.006 & \textbf{1.000} \\
\bottomrule
\end{tabular}
\end{table}

\section{A Statistical Framework for Graph Epistemic Uncertainty}
\label{sec:framework}

Addressing \textbf{(RQ1)}, we extend the i.i.d.\ characterization of epistemic predictors~\citep{Bengs-2022, Jurgens-2024} to node classification by defining a statistical framework with an information-growth experimental protocol and a consistency criterion for epistemic predictors.

\subsection{Information-Growth Experimental Protocol}
\label{sec:IG-experiment}

Constructing an information-growth experimental protocol for node classification requires three components: (i) the \emph{sampling unit}, i.e., what constitutes one observation from the DGP; (ii) the \emph{sampling scheme}, i.e., how increasingly informative observations are generated from the same DGP; and (iii) the \emph{learning task}, i.e., how the supervised prediction problem evolves along this sequence.

\textbf{Sampling unit.}
Unlike classical datasets, the nodes of a graph are not a collection of i.i.d.\ observations. We therefore take the observed graph \(\gG=(\gV,\gE)\), equivalently the subgraph induced by \(\gV\), as the sampling unit, following the graph-as-observation perspective of~\citet{Shalizi-2013}. Increasing information then corresponds to observing nested graphs \(\gG\subset\gG'\), with \(\gV\subset\gV'\) and \(\gG\) equal to the subgraph of \(\gG'\) induced by \(\gV\).

\textbf{Sampling scheme.}
Requiring two nested graphs to arise from the \emph{same} DGP is non-trivial. This consistency is implicit in the i.i.d.\ setting, where every finite sample is a restriction of the same infinite sequence, but need not hold across graphs of different sizes. We make it explicit through \emph{projectibility}~\citep{Shalizi-2013}. Let \( \gA \) be a collection of finite subsets of a countable node index set \( \I \), ordered by inclusion, and let \( \{ \Prob_{A,\phi} \}_{A \in \gA, \phi \in \Phi} \) denote a family of finite-dimensional DGPs, where \( G_{A} \sim \Prob_{A, \phi} \) denotes the attributed graph indexed by \( A \). For \(A \subset B\), let \( \pi_{B\to A} \) denote the natural projection that restricts an observation indexed by \( B \) to the indices in \( A \).

\begin{definition}[Projective DGP]
\label{def:projective}
The family \( \{ \Prob_{A,\phi} \}_{A \in \gA, \phi \in \Phi} \) is \emph{projective} if, for every \( A \subset B \),
\begin{equation}
\Prob_{A,\phi} = \Prob_{B,\phi} \circ \pi^{-1}_{B\to A} \qquad \forall \phi \in \Phi.
\end{equation}
\end{definition}

Thus, projectibility gives a coherent meaning to information growth by ensuring that every smaller observation has exactly the distribution obtained by marginalizing any larger observation under the stochastic process with the same parameter \( \phi \). Moreover, by the Kolmogorov extension theorem~\citep{Kallenberg-2002}, every such projective family uniquely defines an infinite-dimensional distribution \( \Prob_\phi \) whose finite-dimensional marginals are precisely \( \{ \Prob_{A, \phi} \}_{A \in \gA} \). Dense graphons are an important class of projective graph families~\citep{Veitch-2015} and will serve as the concrete instantiation of this framework in our experiments. Other prominent graph models, including many exponential random graph models and sparse random graph formulations, are not projective~\citep{Shalizi-2013,Schweinberger-2017}.

\textbf{Learning task.}
Finally, the supervised tasks induced by increasingly large graph observations must exploit the corresponding information growth. If \(\gT^{(\mathrm{tr})}\subset\gV\) and \(\gT'^{(\mathrm{tr})}\subset\gV'\) denote the training sets associated with \(\gG\subset\gG'\), respectively, we require \(\gT^{(\mathrm{tr})}\subset\gT'^{(\mathrm{tr})}\), so that each larger task extends the previous one with additional information about the same underlying node-classification problem. Validation and test sets are defined analogously.

\subsection{Consistent Epistemic Predictors}

We now characterize a consistent epistemic predictor under this information-growth experimental protocol. For an increasing sequence \( A_n \subset A_{n+1} \), with \( |A_n|=n \), let \( Q_{A_n} = \{ Q_{A_n}^{(i)} \}_{i \in A_n} \) \ denote the collection of node-wise second-order predictors. For an epistemic uncertainty measure \( U \), we define the uncertainty of this collection as the average node-wise uncertainty \( \bar{U}_{A_n} := |A_n|^{-1} \sum_{i \in A_n} U(Q_{A_n}^{(i)}) \). We define consistency over the full node set \( \gA_n \), even though empirically we evaluate uncertainty on the test nodes. Let \( p_\phi^\star \) denote the Bayes-optimal predictor under the infinite-dimensional law \( \Prob_\phi \) induced by projectibility, and \( p_{\phi}^\star|A_n \) its restriction to \( A_n \).

\begin{definition}[Consistent epistemic predictors]
\label{def:consistent-predictor}
Under the information-growth experimental protocol defined in Section~\ref{sec:IG-experiment}, a sequence of collections of second-order predictors \( \{Q_{A_n} \}_{n \geq 1} \) is \emph{consistent} if the following two conditions hold:
\begin{enumerate}
    \item \textbf{Monotonicity.} For every \(n\), \( \E_{G_{A_n}} [\bar{U}_{A_n} ] \geq \E_{G_{A_{n+1}}} [\bar{U}_{A_{n+1}}] \).
    Furthermore, there exist \( n_0 \) and \( k > 0 \) such that \( \E_{G_{A_{n_0}}} [\bar{U}_{A_{n_0}}] > \E_{G_{A_{n_0+k}}}[\bar{U}_{A_{n_0+k}}] \).
    \item \textbf{Consistency.} As \( |A_n| \to \infty \), \( Q_{A_n} \)  converges in probability to the Dirac measure at \( p_\phi^\star|A_n \) under the distance \(d(Q_{A_n}, \delta_{p^\star_{\phi|A_n}}) = |A_n|^{-1} \sum_{i \in A_n} d(Q_{A_n}^{(i)}, \delta_{p^\star_{\phi|A_n}}^{(i)} ) \), where the node-wise distance \( d \) metrizes weak convergence on \( \gP(\Delta_C) \).
\end{enumerate}
\end{definition}

These requirements characterize epistemic uncertainty as reducible with information: by condition (1), it decreases in expectation, excluding constant uncertainty sequences~\citep{Jurgens-2024}, and by (2), it vanishes through concentration on the limiting Bayes-optimal predictor~\citep{Bengs-2022}. For the entropy-based uncertainty considered here, this yields \( \lim_{n \to \infty} \E_{G_{A_n}}[\bar{U}_{A_n})] = 0 \), leaving only true aleatoric uncertainty. Such consistency requires strong assumptions~\citep{Jain-2026}; in practice, model misspecification and estimation bias may affect the predictor's estimates without being explicitly represented by its epistemic uncertainty~\citep{Jimenez-2026}.

\section{Graph EDL under Information Growth}
\label{sec:edl-limitations}
We examine graph EDL methods considered in Section~\ref{sec:EDL-models} when estimation and approximation errors are removed. By analyzing their population optima, we study whether they are consistent epistemic predictors per Definition~\ref{def:consistent-predictor}. Unlike their i.i.d.\ counterparts, these methods incorporate graph-specific uncertainty principles and losses, motivating an extension of the analogous population-level analysis by~\citet{Shen-2024}. To analyze them within a common framework, we express their objectives through the following unified loss.
\begin{definition}[Unified EDL loss]
\label{def:unified-edl-loss}
Let \( Q^{\mathrm{prior}} \) be a fixed Dirichlet prior, \( p(y | \theta) \) a likelihood, and \( \nu > 0 \) a tempering parameter. For each node \( i \), the unified EDL loss is 
\begin{equation}
\ell_{\mathrm U}(\psi;i) := \KL ( Q_\psi^{(i)} \| Q_\nu(\cdot | y^{(i)})),
\end{equation}
where \( Q_\nu(\theta | y) \propto
Q^{\mathrm{prior}} (\theta) p(y | \theta)^\nu \) denotes the tempered posterior.
\end{definition}
The next lemma establishes the equivalence between \( \ell_\mathrm{U} \) and the training objectives of each method.

\begin{lemma}
\label{lemma:U-EDL-equiv}
The unified EDL loss \( \ell_\mathrm{U} \) is equivalent, up to positive scaling and a \( \psi \)-independent additive constant, to the EDL objectives under the following choices:
\begin{enumerate}
    \item \( p(e_y^{(i)} | \theta^{(i)}) = \gN (e_y^{(i)}; \theta^{(i)}, \sigma^2 I_C) \), \( Q^{\mathrm{prior}} \) is a fixed Dirichlet prior, and \( \lambda = 2\sigma^2 / \nu \), yielding \( \ell_\mathrm{MSE} \).
    \item \( p(y^{(i)} | \theta^{(i)}) = \theta_{y^{(i)}}^{(i)} \), \( Q^{\mathrm{prior}} \) is a uniform Dirichlet prior, and \( \lambda=\nu^{-1} \), yielding \( \ell_\mathrm{UCE} \).
    \item With a uniform Dirichlet prior \( Q^{\mathrm{prior}} \) and \( \lambda = \nu^{-1} \), evaluating the UCE loss at the PPR-weighted soft target \( \tilde{e}_y^{(j)} \) yields the reformulated GPN-LOP node-wise objective \( \tilde{\ell}_{\mathrm{GPN-LOP}} (\psi, j) \propto \tilde{\ell}_{\mathrm{UCE}} (\psi, j) \); hence, it is equivalent to the unified objective \( \tilde{\ell}_\mathrm{U}(\psi; j) \).
\end{enumerate}
\end{lemma}

We now characterize the population optimum of this common loss.
\begin{theorem}
    \label{thm:U-loss-opt}
    The population minimizer of the unified EDL loss \( \ell_\mathrm{U} \) under \( \Prob_\phi \) satisfies
    \begin{equation}
        \argmin_{\psi}\, \E_{\Prob_\phi(\rmA, \rmX, Y)} [ \ell_\mathrm{U} (\psi; i) ] = \argmin_{\psi}\, \E_{\Prob_\phi(\rmA, \rmX)} [
        \KL( Q_\psi^{(i)} \| Q^{\star,(i)}) ],
    \end{equation}
        where
        \begin{equation}
            Q^{\star, (i)} (\theta) = \frac{Q^{\mathrm{prior}}(\theta) \exp (\nu \E_{\Prob_\phi(y^{(i)}| \rmX, \rmA)} [\log p(y ^{(i)}|\theta^{(i)})])}{\int_{\Delta_C} Q^{\mathrm{prior}}(\theta') \exp (\nu \E_{\Prob_\phi(y^{(i)}| \rmX, \rmA)} [\log p(y^{(i)}|\theta')]) \, d\theta'}.
        \end{equation}
\end{theorem}
The proofs of Lemma~\ref{lemma:U-EDL-equiv} and Theorem~\ref{thm:U-loss-opt} are provided in Appendix~\ref{app:sec-U-equiv}. For fixed \( Q^{\mathrm{prior}}\) and finite \( \nu \), GPN, GPN-LOP, and CUQ-GNN retain non-vanishing epistemic uncertainty at the population optimum, failing Definition~\ref{def:consistent-predictor}, while their aleatoric uncertainty depends on \( \nu \) and may differ from that of the Bayes-optimal predictor; see Appendix~\ref{app:sec-LOP-GPN} for more details on GPN-LOP. For S-BGCN-K, an unbounded increase in local evidence, such as under the dense graphon setting, implies convergence of the population-level GKDE prior and the corresponding population optimum to a Dirac distribution; see Appendix~\ref{app:sec-S-BGCN-K}. With Monte Carlo dropout, however, epistemic uncertainty need not vanish~\citep{Verdoja-2021}. The location of the limiting Dirac also depends on the GKDE construction and need not be Bayes-optimal. Moreover, because they are deterministic models, none of these methods account for epistemic uncertainty induced by data sampling~\citep{Shen-2024}, which is a recognized source of epistemic uncertainty~\citep{Jimenez-2026} and, in our setting, corresponds to randomness from observing a single finite graph. S-BGCN-T-K and GPN-GD additionally regularize class-probability estimation and pairwise relations between training nodes, respectively, without introducing further mechanisms for evidence accumulation under graph growth, and are therefore not analyzed separately.   

Thus, the asymptotic behavior of epistemic and aleatoric uncertainty predicted by S-BGCN-K, GPN, GPN-LOP, and CUQ-GNN is determined externally, either by the loss regularization or by the construction of the prior. Graph-specific uncertainty principles defined on a fixed graph therefore do not by themselves guarantee the required information-growth behavior.

\section{Graph Bootstrap Ensembles}
\label{sec:gbens}
The previous analysis motivates constructing a second-order predictor whose uncertainty is tied explicitly to the information available about the graph DGP. We pursue this idea through \emph{graph bootstrap ensembles} (GB-Ens), extending deep ensembles, an already understudied approach in node classification~\citep{Vieira-2026}, to the largely unexplored setting of graph bootstrapping. For each of \( B \) bootstrap graph realizations, we train \( S \) independently initialized GNNs, mirroring i.i.d.\ bootstrap ensembles that combine procedural and data variability through resampling~\citep{Jurgens-2024, Shen-2024, Jain-2026, Jimenez-2026}. The ability to represent both recognized sources of epistemic uncertainty~\citep{Jain-2026, Jimenez-2026} provides a key motivation for using bootstrap ensembles rather than alternative second-order predictors, including EDL models and deep ensembles, which may fail to capture one or both sources.

In the i.i.d.\ setting, bootstrap-based MI has also been shown, under suitable conditions, to converge asymptotically to its Bayesian counterpart~\citep{Jain-2026}, providing theoretical support for bootstrap-based epistemic uncertainty estimation with neural networks. For graphs, however, bootstrap guarantees remain restricted to specific combinations of graph DGPs, bootstrap schemes, and estimation targets~\citep{Schmid-2017, Levin-2021, Green-2022, Lin-2022, Shao-2026}, with no analogous result for MI. Nevertheless, projective synthetic graph DGPs whose sampling process can be accurately approximated by a bootstrap procedure provide a tractable setting for studying whether graph bootstrap ensembles empirically exhibit the consistent information-growth behavior of Definition~\ref{def:consistent-predictor}, analogous to that established in the i.i.d.\ setting. GB-Ens is also computationally feasible despite increased total cost, through parallel training and sequential inference; see Appendix~\ref{app:sec-comp-comp} for computational comparisons.

\textbf{Uncertainty decomposition.}
For node \(i\), let \(\theta_{b,s}^{(i)} \in \Delta_C\) be the predictive distribution from seed \(s\) on bootstrap realization \(b\). The resulting second-order predictor and predictive mean are
\begin{equation}
Q_{B,S}^{(i)} = \frac{1}{BS} \sum_{b=1}^{B} \sum_{s=1}^{S} \delta_{\theta_{b,s}^{(i)}}, \qquad \bar{\theta}^{(i)} = \frac{1}{BS} \sum_{b=1}^{B} \sum_{s=1}^{S} \theta_{b,s}^{(i)},
\end{equation}
where \( \delta_{\theta_{b,s}^{(i)}} \) denotes the Dirac measure concentrated at \( \theta_{b,s}^{(i)} \), with
\begin{equation}
\TU^{(i)} = H(\bar{\theta}^{(i)}), \qquad \AU^{(i)} = \frac{1}{BS} \sum_{b=1}^{B} \sum_{s=1}^{S} H(\theta_{b,s}^{(i)}), \qquad \EU^{(i)} = \TU^{(i)} - \AU^{(i)}.
\end{equation}
The crossed \( B \times S \) design further allows us to separate the two sources of epistemic uncertainty represented by GB-Ens. With the seed-averaged prediction \( \bar{\theta}_{b}^{(i)} = S^{-1} \sum_{s=1}^{S} \theta_{b,s}^{(i)} \),
\begin{equation}
    \EU_{\mathrm{boot}}^{(i)} = H(\bar{\theta}^{(i)}) - \frac{1}{B} \sum_{b=1}^{B} H(\bar{\theta}_{b}^{(i)}), \qquad \EU_{\mathrm{seed}}^{(i)} = \frac{1}{B} \sum_{b=1}^{B} \left[ H(\bar{\theta}_{b}^{(i)}) - \frac{1}{S} \sum_{s=1}^{S} H(\theta_{b,s}^{(i)}) \right].
\end{equation}
Here, \( \EU_{\mathrm{boot}} \) measures disagreement across bootstrap graph realizations, while \( \EU_{\mathrm{seed}} \) measures training-induced disagreement within each realization, and \( \EU^{(i)} = \EU_{\mathrm{boot}}^{(i)} + \EU_{\mathrm{seed}}^{(i)} \).

\textbf{Graph resampling.}
To instantiate the graph-resampling component of GB-Ens, we adopt a local non-parametric bootstrap adapted from~\citet{Jeong-2025}, developed under an exchangeable graphon model. It estimates asymptotically unbiased edge probabilities from the observed connectivity of similar nodes and samples graph realizations accordingly. We use an edge-only bootstrap that retains node identities, features, and labels across realizations, thereby preserving the node classification task while capturing sampling variability in the graph structure. The complete resampling procedure and implementation details are given in Appendix~\ref{app:bootstrap}.

\section{Experiments}
\label{sec:experiments}
We empirically address \textbf{(RQ2)}--\textbf{(RQ3)} through controlled experiments on synthetic projective graph DGPs. For \textbf{(RQ2)}, we validate our analysis of graph EDL uncertainty under information growth. For \textbf{(RQ3)}, we assess graph bootstrap ensembles as an alternative approach to epistemic prediction. Code is available at \href{[code]}{https://github.com/emmamenegh/Research-Code-Release}.

\textbf{Models.}
We evaluate S-BGCN-K and S-BGCN-T-K~\citep{Zhao-2020}, GPN~\citep{Stadler-2021}, GPN-GD\textsubscript{EVI} and GPN-GD\textsubscript{LAT}~\citep{Hart-2023}, GPN-LOP~\citep{Damke-2024a}, CUQ-APPNP, CUQ-GAT, and CUQ-GCN~\citep{Damke-2024b} as EDL baselines, and GB-Ens with APPNP, GAT, and GCN architectures. Each GB-Ens comprises \( 25 \times 4 \) models, corresponding to \( 25 \) bootstrap samples and \( 4 \) random seeds per sample. Model parameterizations are reported in Appendix~\ref{app:model-params}.

\textbf{Datasets.}
Our experiments comprise two synthetic projective graphon DGPs. The synthetic DGPs are an exponential-decay graphon and an assortative stochastic-block model (SBM), which define complementary continuous- and discrete-structure regimes for a balanced seven-class node-classification problem. For each DGP, we generate \(20\) independent replicates of a projective sequence, with six graph sizes logarithmically spaced between \(500\) and \(10{,}000\) nodes. At each graph size, we use \(5\%/15\%/80\%\) class-stratified train/validation/test splits constructed according to the protocol in Section~\ref{sec:IG-experiment} and fixed across bootstrap replicates for each observed graph. Full DGP specifications and graphon statistics are reported in Appendices~\ref{app:synthetic-graphons} and~\ref{app:bootstrap}. 

%%%%% RQ2 %%%%%
\textbf{RQ2: Graph EDL models are inconsistent epistemic predictors.}

We examine whether the uncertainty behavior of graph EDL models characterized at the nonparametric population optimum by Theorem~\ref{thm:U-loss-opt} persists with finite data and parametric models. We assess both EU concentration with increasing information and the dependence of EU and AU on model parameterization. Figure~\ref{fig:EU-bench-graphon1} shows mean EU across \( 20 \) projective-sequence replicates of the exponential-decay graphon, with \( 95\% \) bootstrap confidence intervals (CIs). S-BGCN-K and S-BGCN-T-K exhibit decreasing EU, as permitted by their information-dependent prior, but remain non-negligible over the observed sequence, whereas GPN, GPN-GD, GPN-LOP, and CUQ-GNN remain approximately constant, consistent with their population optima retaining finite evidence. EU magnitude also differs substantially across methods, reflecting its theoretical dependence on the parameterization, as illustrated by the similar magnitudes of GPN, GPN-GD, and CUQ-GNN under shared parameter values. At \( N=10{,}000 \), varying the EDL regularization parameter \( \lambda \) reveals substantial changes in AU on the same observed graph (Figure~\ref{fig:AU-bench-graphon1}), demonstrating its dependence on model parameterization rather than on the DGP alone. Analogous results for the SBM are reported in Appendix~\ref{app:sec-exp-proj}. 

\begin{figure}[t]
    \centering

    \begin{subfigure}[t]{0.49\linewidth}
        \centering
        \includegraphics[width=\linewidth]{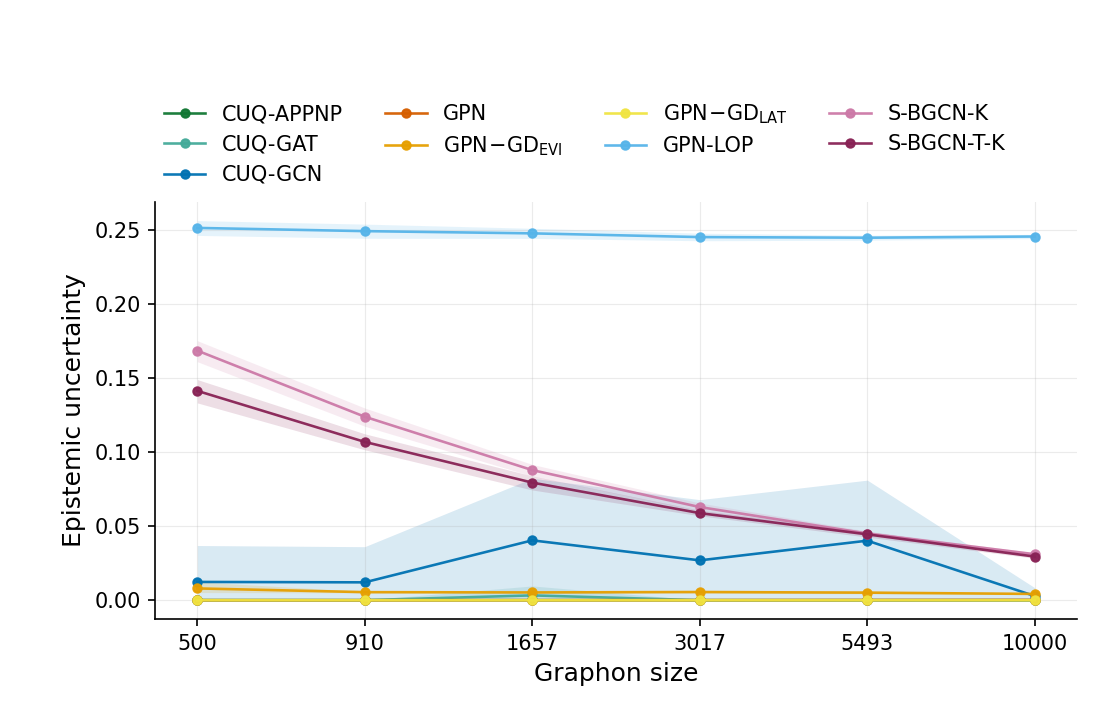}
        \caption{Epistemic uncertainty under information growth.}
        \label{fig:EU-bench-graphon1}
    \end{subfigure}
    \hfill
    \begin{subfigure}[t]{0.49\linewidth}
        \centering
        \includegraphics[width=\linewidth]{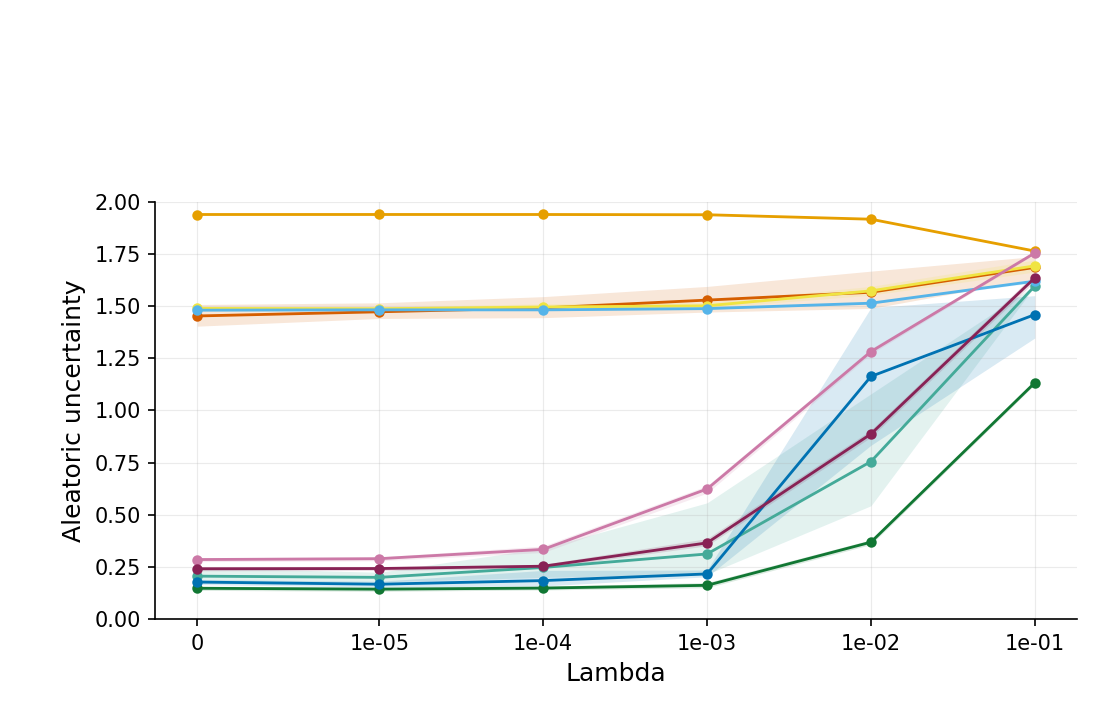}
        \caption{Aleatoric uncertainty across \( \lambda \) at \( N=10{,}000 \).}
        \label{fig:AU-bench-graphon1}
    \end{subfigure}
    
    \caption{Uncertainty estimates of EDL models on the exponential-decay graphon.}
    \label{fig:AU-graphons}
\end{figure}

%%%%% RQ3 %%%%%
\textbf{RQ3: GB-Ens is a family of candidate consistent epistemic predictors.}

Figure~\ref{fig:EU-GBEns-graphon1} shows mean GB-Ens epistemic uncertainty across \( 20 \) projective-sequence replicates of the exponential-decay graphon, with pointwise \( 95\% \) bootstrap CIs. From \( N=500 \) to \( N=10{,}000 \), EU decreases monotonically by \( 94\% \) for GCN, \( 91\% \) for APPNP, and \( 61\% \) for GAT. On the SBM (Appendix~\ref{app:sec-exp-proj}), EU similarly decreases substantially for APPNP and GCN, whereas GAT shows a weaker, non-monotonic reduction due to persistent randomized-training variability, plausibly reflecting its sensitivity to learned attention weights and additional attention dropout. The decomposition in Figure~\ref{fig:EU-deco-GBEns-graphon-1} further shows comparable randomized-seed and bootstrap contributions to EU that decrease with increasing information, with relative magnitude depending on the backbone and DGP. Finally, Figure~\ref{fig:AU-Ens-graphon-1} shows that GB-Ens AU decreases and eventually stabilizes, with all three backbones approaching similar plateaus, consistent with convergence toward a predictor-independent estimate of aleatoric uncertainty. The same behavior is observed for the SBM (Appendix~\ref{app:sec-exp-proj}).

To assess the effect of graph resampling, we compare GB-Ens with standard graph deep ensembles (G-Ens), trained with \( 10 \) different random seeds on the observed graph using the same backbones. Figure~\ref{fig:EU-GEns-graphon-1} shows that, on the exponential-decay graphon, G-Ens exhibit no consistent EU contraction, while Figure~\ref{fig:AU-Ens-graphon-1} shows strongly backbone-dependent AU trajectories, with GAT trending upward and GCN continuing to decrease. Similar behavior is observed on the SBM (Appendix~\ref{app:sec-exp-proj}). Moreover, the gap between G-Ens EU and the randomized-training component of GB-Ens widens with graph size, as the latter contracts under information growth while the former does not. Graph resampling therefore does not simply add bootstrap variability, but induces convergent behavior also in the seed component, enabling EU contraction that is not observed for standard deep ensembles.

As a simple diagnostic of GB-Ens's approximation to the full second-order distribution induced by true-DGP resampling, which captures predictive variability without DGP estimation error, we compare the two distributions using normalized node-wise Wasserstein-1 distances. On the exponential-decay graphon, distances remain small and stable with graph size for APPNP (\( \approx0.29 \)) and GCN (\( \approx0.21 \)), indicating good agreement, while increasing from \( 0.45 \) to \( 0.62 \) for GAT. Similar results on the SBM are reported in Appendix~\ref{app:sec-exp-proj}, suggesting that approximation quality varies across predictors.

Overall, these results support GB-Ens as a family of candidate consistent epistemic predictors. This behavior nevertheless remains predictor-specific, as EU contraction ultimately relies on regularity of the underlying learning procedure; persistent training variability can therefore produce weaker and partly non-monotonic reductions.

\begin{figure}[t]
    \centering

    \begin{subfigure}[t]{0.49\linewidth}
        \centering
        \includegraphics[width=\linewidth]{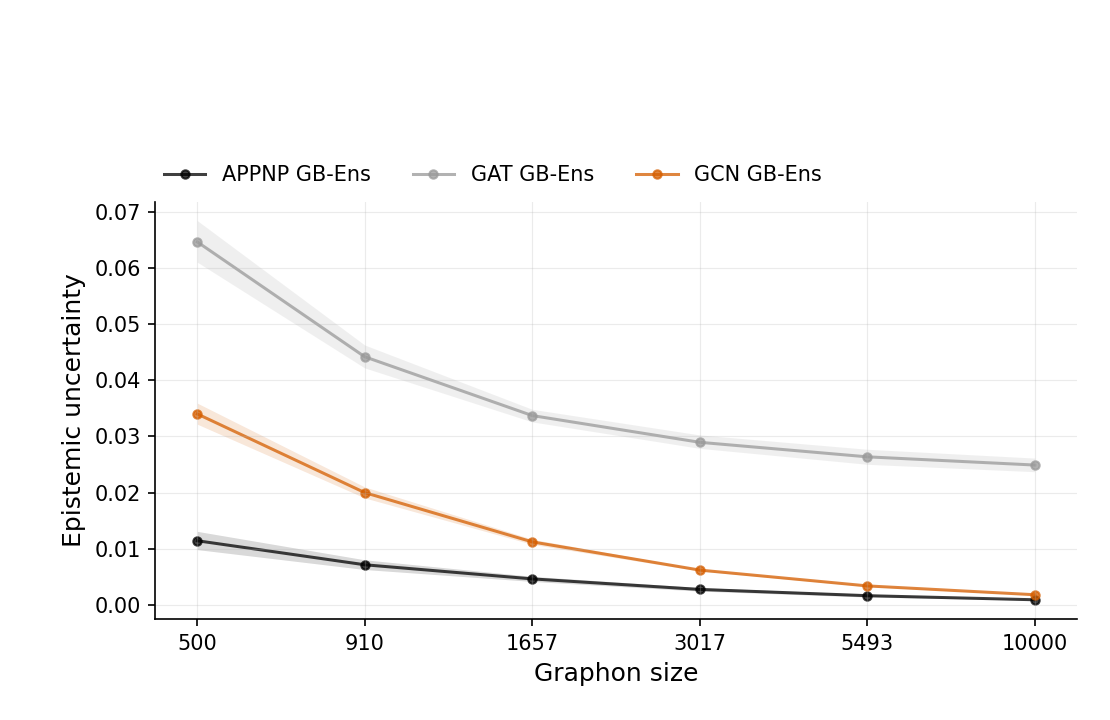}
        \caption{Total epistemic uncertainty.}
        \label{fig:EU-GBEns-graphon1}
    \end{subfigure}
    \hfill
    \begin{subfigure}[t]{0.49\linewidth}
        \centering
        \includegraphics[width=\linewidth]{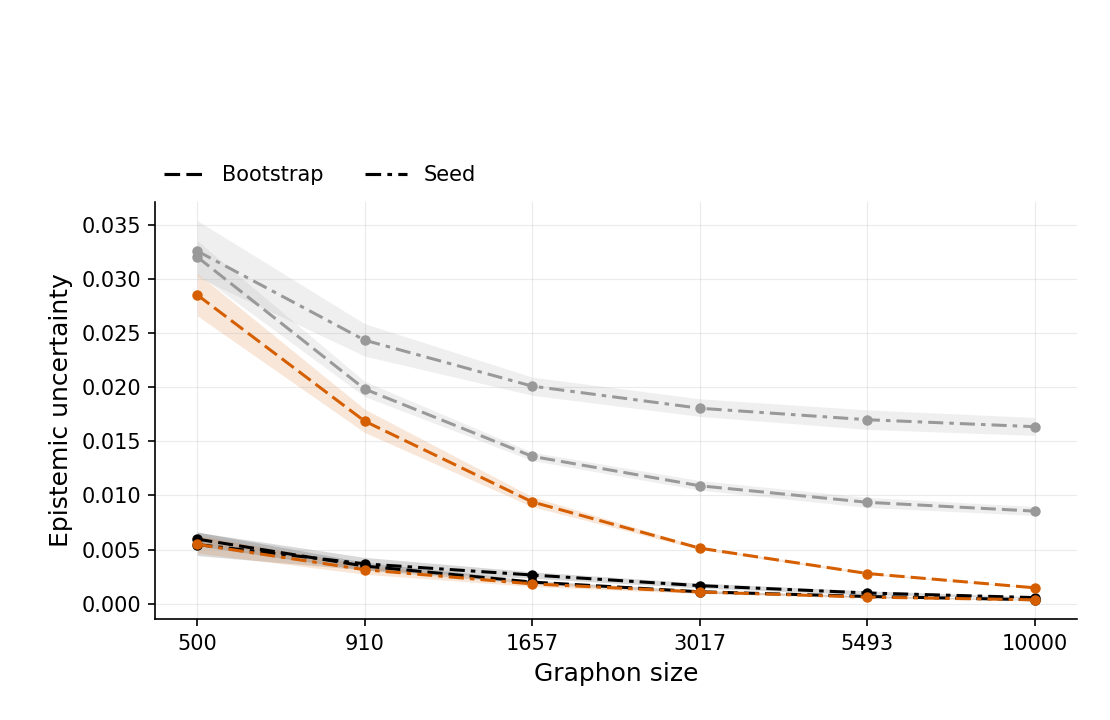}
        \caption{Bootstrap and training-seed decomposition.}
        \label{fig:EU-deco-GBEns-graphon-1}
    \end{subfigure}
    
    \caption{GB-Ens mean epistemic uncertainty on the exponential-decay graphon.}
    \label{fig:EU-GBEns-graphon1-general}
\end{figure}

\begin{figure}[t]
    \centering

    \begin{subfigure}[t]{0.49\linewidth}
        \centering
        \includegraphics[width=\linewidth]{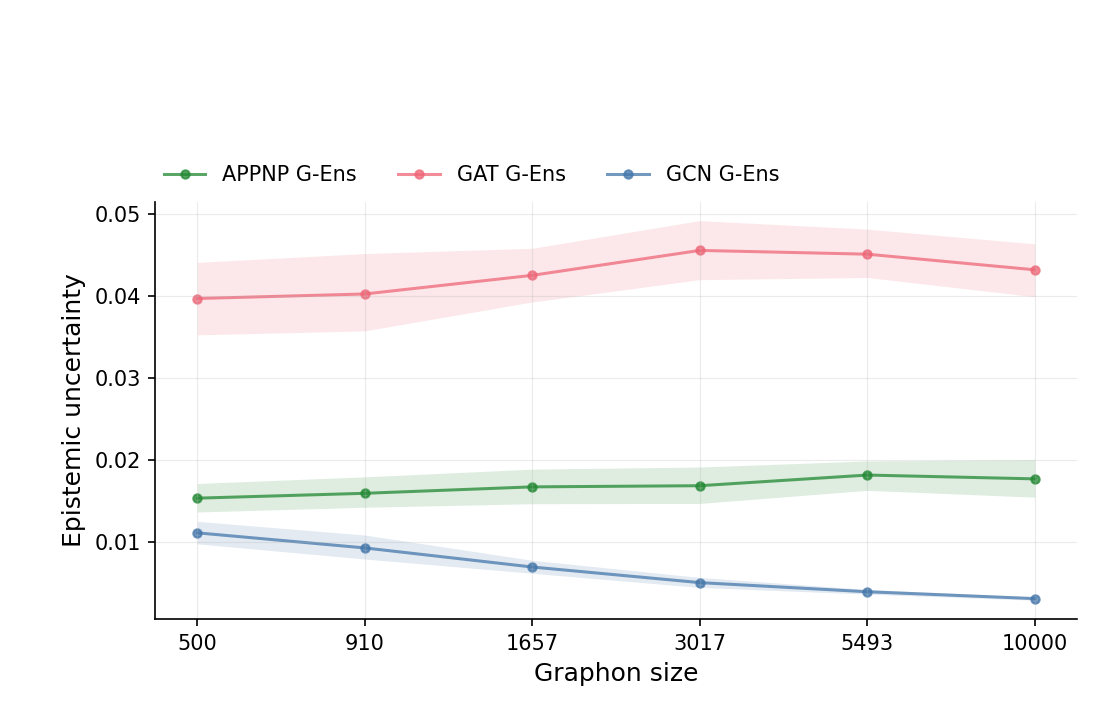}
        \caption{Epistemic uncertainty of G-Ens.}
        \label{fig:EU-GEns-graphon-1}
    \end{subfigure}
    \hfill
    \begin{subfigure}[t]{0.49\linewidth}
        \centering
        \includegraphics[width=\linewidth]{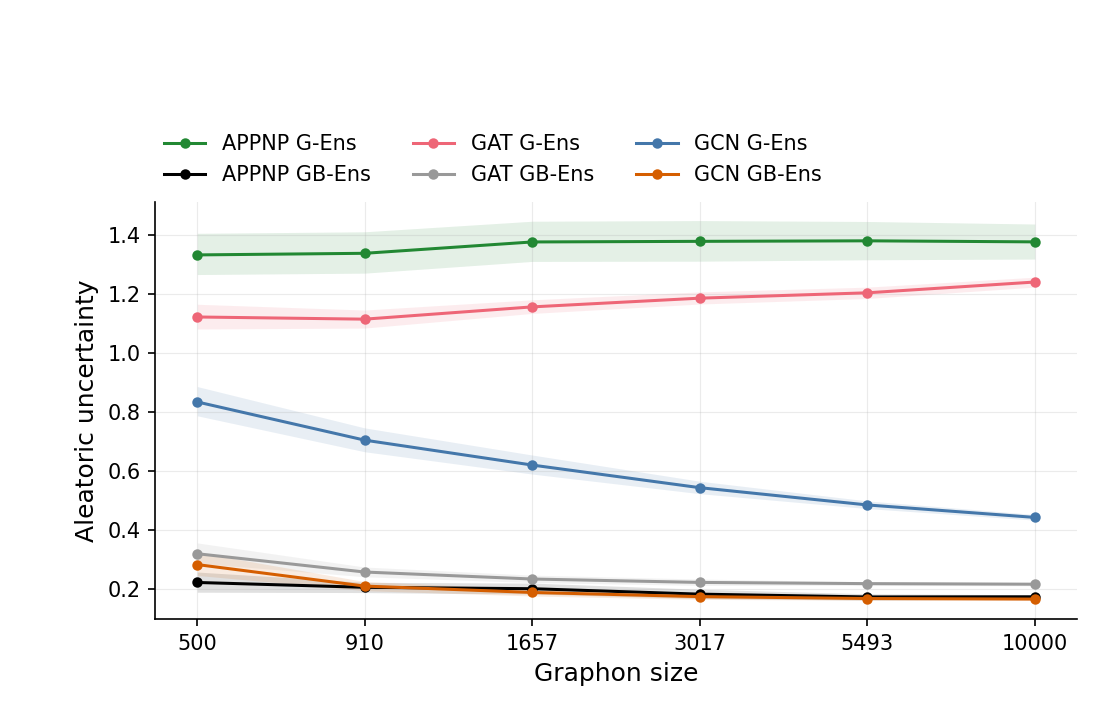}
        \caption{Aleatoric uncertainty of GB-Ens and G-Ens.}
        \label{fig:AU-Ens-graphon-1}
    \end{subfigure}
\caption{Ensemble mean uncertainties on the exponential-decay graphon.}
    \label{fig:U-Ens-graphon1}
\end{figure}

\section{Conclusions}
\textbf{Take-home message.} Epistemic uncertainty in node classification should be assessed through its behavior under increasing information, rather than solely by obeying principles or being useful for downstream tasks on a fixed graph. This perspective identifies consistent epistemic uncertainty as requiring estimates that reflect data uncertainty rather than remaining hyperparameter-dependent, a requirement generally not met by graph EDL models. GB-Ens provides a constructive approach by capturing data and procedural uncertainty, with controlled experiments supporting it as a family of candidate consistent predictors.

\textbf{Limitations \& future work.} The theoretical consistency of bootstrap-based graph epistemic predictors remains to be established. Moreover, information-growth behavior cannot be evaluated on standard real-world benchmarks, which lack the sampling process required to generate increasingly informative observations from the same DGP. Benchmarks leveraging real graph-growth mechanisms, where projectibility is often natural~\citep{Shalizi-2013}, could extend our analysis beyond synthetic random graphs, either retaining node classification and motivating improved bootstrap-based predictors or considering tasks more naturally aligned with graph growth.

\bibliography{iclr2027_conference}

@article{Abdar-2021,
  title={A review of uncertainty quantification in deep learning: Techniques, applications and challenges},
  author={Abdar, Moloud and Pourpanah, Farhad and Hussain, Sadiq and Rezazadegan, Dana and Liu, Li and Ghavamzadeh, Mohammad and Fieguth, Paul and Cao, Xiaochun and Khosravi, Abbas and Acharya, U Rajendra and others},
  journal={Information fusion},
  volume={76},
  pages={243--297},
  year={2021},
  publisher={Elsevier}
}

@misc{
Bazhenov-2023,
title={Revisiting Uncertainty Estimation for Node Classification: New Benchmark and Insights},
author={Gleb Bazhenov and Denis Kuznedelev and Andrey Malinin and Artem Babenko and Liudmila Prokhorenkova},
year={2023},
url={https://openreview.net/forum?id=DB3BH3arU2Y}
}

@article{Bengs-2022,
  title={Pitfalls of epistemic uncertainty quantification through loss minimisation},
  author={Bengs, Viktor and H{\"u}llermeier, Eyke and Waegeman, Willem},
  journal={Advances in Neural Information Processing Systems},
  volume={35},
  pages={29205--29216},
  year={2022}
}

@article{Charpentier-2020,
  title={Posterior network: Uncertainty estimation without ood samples via density-based pseudo-counts},
  author={Charpentier, Bertrand and Z{\"u}gner, Daniel and G{\"u}nnemann, Stephan},
  journal={Advances in neural information processing systems},
  volume={33},
  pages={1356--1367},
  year={2020}
}

@inproceedings{
Charpentier-2022,
title={Natural Posterior Network: Deep Bayesian Predictive Uncertainty for Exponential Family Distributions},
author={Bertrand Charpentier and Oliver Borchert and Daniel Z{\"u}gner and Simon Geisler and Stephan G{\"u}nnemann},
booktitle={International Conference on Learning Representations},
year={2022},
url={https://openreview.net/forum?id=tV3N0DWMxCg}
}

@article{Chen-2026,
  title={Uncertainty quantification on graph learning: a survey},
  author={Chen, Chao and Guo, Chenghua and Xu, Rui and Chen, Jiujiu and Liao, Xiangwen and Zhang, Xi and Xie, Sihong and Xiong, Hui and Yu, Philip},
  journal={arXiv preprint arXiv:2404.14642},
  year={2024}
}

@inproceedings{
Damke-2024a,
title={Linear Opinion Pooling for Uncertainty Quantification on Graphs},
author={Clemens Damke and Eyke H{\"u}llermeier},
booktitle={The 40th Conference on Uncertainty in Artificial Intelligence},
year={2024},
url={https://openreview.net/forum?id=qLGkfpXTSn}
}

@inproceedings{Damke-2024b,
  title={CUQ-GNN: Committee-Based Graph Uncertainty Quantification Using Posterior Networks},
  author={Damke, Clemens and H{\"u}llermeier, Eyke},
  booktitle={Joint European Conference on Machine Learning and Knowledge Discovery in Databases},
  pages={306--323},
  year={2024},
  organization={Springer}
}

@phdthesis{Depeweg-2019,
  title={Modeling epistemic and aleatoric uncertainty with bayesian neural networks and latent variables},
  author={Depeweg, Stefan},
  year={2019},
  school={Technische Universit{\"a}t M{\"u}nchen}
}

@article{Galikowski-2023,
  title={A survey of uncertainty in deep neural networks: J. Gawlikowski et al.},
  author={Gawlikowski, Jakob and Tassi, Cedrique Rovile Njieutcheu and Ali, Mohsin and Lee, Jongseok and Humt, Matthias and Feng, Jianxiang and Kruspe, Anna and Triebel, Rudolph and Jung, Peter and Roscher, Ribana and others},
  journal={Artificial intelligence review},
  volume={56},
  number={Suppl 1},
  pages={1513--1589},
  year={2023},
  publisher={Springer}
}

@inproceedings{Gasteiger-2018,
title={Combining Neural Networks with Personalized PageRank for Classification on Graphs},
author={Johannes Gasteiger and Aleksandar Bojchevski and Stephan Günnemann},
booktitle={International Conference on Learning Representations},
year={2019},
url={https://openreview.net/forum?id=H1gL-2A9Ym},
}

@inproceedings{Giles-1998,
  title={CiteSeer: An automatic citation indexing system},
  author={Giles, C Lee and Bollacker, Kurt D and Lawrence, Steve},
  booktitle={Proceedings of the third ACM conference on Digital libraries},
  pages={89--98},
  year={1998}
}

@article{Green-2022,
  title={Bootstrapping exchangeable random graphs},
  author={Green, Alden and Shalizi, Cosma Rohilla},
  journal={Electronic Journal of Statistics},
  volume={16},
  number={1},
  pages={1058--1095},
  year={2022},
  publisher={The Institute of Mathematical Statistics and the Bernoulli Society}
}

@article{Hart-2023,
  title={Improvements on uncertainty quantification for node classification via distance based regularization},
  author={Hart, Russell and Yu, Linlin and Lou, Yifei and Chen, Feng},
  journal={Advances in Neural Information Processing Systems},
  volume={36},
  pages={55454--55478},
  year={2023}
}

@article{Hu-2021,
  title={Open graph benchmark: Datasets for machine learning on graphs},
  author={Hu, Weihua and Fey, Matthias and Zitnik, Marinka and Dong, Yuxiao and Ren, Hongyu and Liu, Bowen and Catasta, Michele and Leskovec, Jure},
  journal={Advances in neural information processing systems},
  volume={33},
  pages={22118--22133},
  year={2020}
}

@article{Hullermeier-2021,
  title={Aleatoric and epistemic uncertainty in machine learning: An introduction to concepts and methods},
  author={H{\"u}llermeier, Eyke and Waegeman, Willem},
  journal={Machine learning},
  volume={110},
  number={3},
  pages={457--506},
  year={2021},
  publisher={Springer}
}

@article{Jain-2026,
    title={Deep Ensembles for Epistemic Uncertainty: A Frequentist Perspective},
    author={Jain, Anchit and Bates, Stephen},
    journal={arXiv preprint arXiv:2510.22063},
    year={2025}
}

@inproceedings{Jeong-2025,
    title={{LOBSTUR}: A Local Bootstrap Framework for Tuning Unsupervised Representations in Graph Neural Networks},
    author={Sowon Jeong and Claire Donnat},
    booktitle={Machine Learning on Graphs in the Era of Generative Artificial Intelligence},
    year={2025},
    url={https://openreview.net/forum?id=xRasvIyMcW}
}

@inproceedings{Jimenez-2026,
    title={Position: Epistemic Uncertainty Estimation Methods are Fundamentally Incomplete},
    author={Sebastian Jimenez and Mira Juergens and Willem Waegeman},
    booktitle={Forty-third International Conference on Machine Learning Position Paper Track},
    year={2026},
    url={https://openreview.net/forum?id=g598HZM6ib}
}

@inproceedings{Jurgens-2024,
    title={Is Epistemic Uncertainty Faithfully Represented by Evidential Deep Learning Methods?},
    author={Mira Juergens and Nis Meinert and Viktor Bengs and Eyke H{\"u}llermeier and Willem Waegeman},
    booktitle={Forty-first International Conference on Machine Learning},
    year={2024},
    url={https://openreview.net/forum?id=mxjB0LIgpT}
}

@book{Kallenberg-2002,
  title={Foundations of modern probability},
  author={Kallenberg, Olav},
  year={1997},
  publisher={Springer}
}

@inproceedings{Kipf-2017,
    title={Semi-Supervised Classification with Graph Convolutional Networks},
    author={Thomas N. Kipf and Max Welling},
    booktitle={International Conference on Learning Representations},
    year={2017},
    url={https://openreview.net/forum?id=SJU4ayYgl}
}

@article{Levin-2021,
  title={Bootstrapping networks with latent space structure},
  author={Levin, Keith and Levina, Elizaveta},
  journal={arXiv preprint arXiv:1907.10821},
  year={2019}
}

@inproceedings{Li-2025,
  title={Position: Supervised classifiers answer the wrong questions for OOD detection},
  author={Li, Yucen Lily and Lu, Daohan and Kirichenko, Polina and Qiu, Shikai and Rudner, Tim GJ and Bruss, C Bayan and Wilson, Andrew Gordon},
  booktitle={Forty-second International Conference on Machine Learning Position Paper Track},
  year={2025}
}

@article{Lin-2022,
  title={Trading off accuracy for speedup: Multiplier bootstraps for subgraph counts},
  author={Lin, Qiaohui and Lunde, Robert and Sarkar, Purnamrita},
  journal={arXiv preprint arXiv:2009.06170},
  year={2020}
}

@article{McCallum-2000,
  title={Automating the construction of internet portals with machine learning},
  author={McCallum, Andrew Kachites and Nigam, Kamal and Rennie, Jason and Seymore, Kristie},
  journal={Information Retrieval},
  volume={3},
  number={2},
  pages={127--163},
  year={2000},
  publisher={Springer}
}

@inproceedings{Namata-2012,
  title={Query-driven Active Surveying for Collective Classification},
  author={Galileo Namata and Ben London and Lise Getoor and Bert Huang},
  year={2012},
  url={https://api.semanticscholar.org/CorpusID:15384610}
}

@article{Paplham-2026,
  title={Evaluating Epistemic Uncertainty: Beyond OOD Detection and Active Learning},
  author={Paplh{\'a}m, Jakub and Waegeman, Willem and H{\"u}llermeier, Eyke and Franc, Vojt{\v{e}}ch},
  journal={arXiv preprint arXiv:2607.14817},
  year={2026}
}

@article{Sale-2026b,
  title={The Aleatoric-Epistemic Dichotomy of Uncertainty is Meaningful and Indispensable for Machine Learning},
  author={Sale, Yusuf and Kotelevskii, Nikita and Panov, Maxim and H{\"u}llermeier, Eyke},
  journal={Available at SSRN 6783338},
  year={2026}
}

@inproceedings{Schmid-2017,
  title={Exponential random graph models with big networks: Maximum pseudolikelihood estimation and the parametric bootstrap},
  author={Schmid, Christian S and Desmarais, Bruce A},
  booktitle={2017 IEEE international conference on big data (Big Data)},
  pages={116--121},
  year={2017},
  organization={IEEE}
}

@article{Shchur-2019,
  title={Pitfalls of graph neural network evaluation},
  author={Shchur, Oleksandr and Mumme, Maximilian and Bojchevski, Aleksandar and G{\"u}nnemann, Stephan},
  journal={arXiv preprint arXiv:1811.05868},
  year={2018}
}

@article{Schweinberger-2017,
  title={A note on the role of projectivity in likelihood-based inference for random graph models},
  author={Schweinberger, Michael and Krivitsky, Pavel N and Butts, Carter T},
  journal={arXiv preprint arXiv:1707.00211},
  year={2017}
}

@article{Shalizi-2013,
  title={Consistency under sampling of exponential random graph models},
  author={Shalizi, Cosma Rohilla and Rinaldo, Alessandro},
  journal={Annals of statistics},
  volume={41},
  number={2},
  pages={508},
  year={2013}
}

@article{Shao-2026,
  title={Parametric Bootstrap for Fixed Edge-Probability Network Models},
  author={Shao, Zhixuan and Le, Can M},
  journal={arXiv preprint arXiv:2402.01866},
  year={2024}
}

@article{Shen-2024,
  title={Are uncertainty quantification capabilities of evidential deep learning a mirage?},
  author={Shen, Maohao and Ryu, Jongha Jon and Ghosh, Soumya and Bu, Yuheng and Sattigeri, Prasanna and Das, Subhro and Wornell, Gregory W},
  journal={Advances in Neural Information Processing Systems},
  volume={37},
  pages={107830--107864},
  year={2024}
}

@inproceedings{Stadler-2021,
 author = {Stadler, Maximilian and Charpentier, Bertrand and Geisler, Simon and Z\"{u}gner, Daniel and G\"{u}nnemann, Stephan},
 booktitle = {Advances in Neural Information Processing Systems},
 editor = {M. Ranzato and A. Beygelzimer and Y. Dauphin and P.S. Liang and J. Wortman Vaughan},
 pages = {18033--18048},
 publisher = {Curran Associates, Inc.},
 title = {Graph Posterior Network: Bayesian Predictive Uncertainty for Node Classification},
 url = {https://proceedings.neurips.cc/paper_files/paper/2021/file/95b431e51fc53692913da5263c214162-Paper.pdf},
 volume = {34},
 year = {2021}
}

@article{Veitch-2015,
  title={The class of random graphs arising from exchangeable random measures},
  author={Veitch, Victor and Roy, Daniel M},
  journal={arXiv preprint arXiv:1512.03099},
  year={2015}
}

@inproceedings{
Velickovic-2018,
title={Graph Attention Networks},
author={Petar Veličković and Guillem Cucurull and Arantxa Casanova and Adriana Romero and Pietro Liò and Yoshua Bengio},
booktitle={International Conference on Learning Representations},
year={2018},
url={https://openreview.net/forum?id=rJXMpikCZ},
}

@article{Verdoja-2021,
  title={Notes on the behavior of mc dropout},
  author={Verdoja, Francesco and Kyrki, Ville},
  journal={arXiv preprint arXiv:2008.02627},
  year={2020}
}

@article{Vieira-2026,
  title={Do Deep Ensembles Actually Capture Uncertainty in Graph Neural Networks?},
  author={Vieira, Pedro C and Ribeiro, Pedro and Borovitskiy, Viacheslav},
  journal={arXiv preprint arXiv:2605.22593},
  year={2026}
}

@article{
Wang-2024,
title={Uncertainty in Graph Neural Networks: A Survey},
author={Fangxin Wang and Yuqing Liu and Kay Liu and Yibo Wang and Sourav Medya and Philip S. Yu},
journal={Transactions on Machine Learning Research},
issn={2835-8856},
year={2024},
url={https://openreview.net/forum?id=0e1Kn76HM1},
note={}
}

@article{Wang-2026,
  title={A review of uncertainty representation and quantification in neural networks},
  author={Wang, Kaizheng and Cuzzolin, Fabio and Shariatmadar, Keivan and Moens, David and Hallez, Hans},
  journal={IEEE transactions on pattern analysis and machine intelligence},
  year={2025},
  publisher={IEEE}
}

@inproceedings{Wimmer-2023,
  title={Quantifying aleatoric and epistemic uncertainty in machine learning: Are conditional entropy and mutual information appropriate measures?},
  author={Wimmer, Lisa and Sale, Yusuf and Hofman, Paul and Bischl, Bernd and H{\"u}llermeier, Eyke},
  booktitle={Uncertainty in artificial intelligence},
  pages={2282--2292},
  year={2023},
  organization={PMLR}
}

@inproceedings{Xu-2020,
  title={Learning graphons via structured gromov-wasserstein barycenters},
  author={Xu, Hongteng and Luo, Dixin and Carin, Lawrence and Zha, Hongyuan},
  booktitle={Proceedings of the AAAI Conference on Artificial Intelligence},
  volume={35},
  pages={10505--10513},
  year={2021}
}

@article{Zhao-2020,
  title={Uncertainty aware semi-supervised learning on graph data},
  author={Zhao, Xujiang and Chen, Feng and Hu, Shu and Cho, Jin-Hee},
  journal={Advances in neural information processing systems},
  volume={33},
  pages={12827--12836},
  year={2020}
}
\bibliographystyle{iclr2027_conference}

\appendix
\setcounter{equation}{0}
\section{Graph EDL Under Information Growth}
\label{app:sec-EDL}
\subsection{Unified EDL Loss}
\label{app:sec-U-equiv}
In this section, we extend the unified formulation of EDL losses from the i.i.d.\ setting~\citep{Shen-2024} to node classification. We begin by defining a unified EDL loss and establishing its equivalence to the MSE and UCE losses used by existing graph EDL methods.

Let \( Q^{\mathrm{prior}} \) be a fixed Dirichlet prior and \( p(y | \theta) \) a likelihood. For a tempering parameter \( \nu > 0 \), define the tempered posterior
\begin{equation}
    Q_{\nu}( \theta | y) := \frac{Q^{\mathrm{prior}}(\theta) p( y | \theta)^{\nu}}{\int_{\Delta_C} Q^{\mathrm{prior}}(\theta') p(y | \theta')^{\nu}  d\theta'}.
\end{equation}
For each node \( i \), the unified EDL loss is 
\begin{equation}
    \ell_{\mathrm{U}}(\psi;i) := \KL ( Q_\psi^{(i)} \| Q_\nu( \cdot | y^{(i)})).
\end{equation}
Expanding the KL divergence gives
\begin{equation}
    \label{eq-A:U-loss}
    \ell_{\mathrm{U}}(\psi;i) = \KL ( Q_\psi^{(i)} \| Q^{\mathrm{prior}} ) - \nu \E_{\theta^{(i)} \sim Q_\psi^{(i)}} [
        \log p(y^{(i)} | \theta^{(i)}) ] + C,
\end{equation}
where \(C\) is independent of \(\psi\).

\begin{lemma}
\label{app:lemma-U-SBGCN}
For the Gaussian likelihood
\( p(e_y^{(i)} | \theta^{(i)}) = \gN(e_y^{(i)};\theta^{(i)}, \sigma^2I_C) \) and a fixed Dirichlet prior \( Q^{\mathrm{prior}} \),
setting \( \lambda = 2 \sigma^2 / \nu \) yields
\begin{equation}
    \ell_{\mathrm{U}}(\psi;i) = \lambda^{-1} \ell_{\mathrm{MSE}}(\psi;i) + C,
\end{equation}
where
\begin{equation}
    \ell_{\mathrm{MSE}}(\psi;i) = \E_{\theta^{(i)} \sim Q_\psi^{(i)}} [ \|\theta^{(i)} -e_y^{(i)}\|^2] + \lambda \KL ( Q_\psi^{(i)} \| Q^{\mathrm{prior}} ),
\end{equation}
and \( C \) is independent of \( \psi \). Thus, \( \ell_\mathrm{U} \) is equivalent, up to scaling and additive constants, to \( \ell_\mathrm{MSE} \).
\end{lemma}
\begin{proof}
Since
\begin{equation}
    -\log p(e_y^{(i)} | \theta^{(i)}) =
    \frac{1}{2 \sigma^2}
    \| \theta^{(i)} -e_y^{(i)} \|^2 + C',
\end{equation}
substitution into~\eqref{eq-A:U-loss} gives the result.
\end{proof}

\begin{lemma}
\label{lemma-A:U-UCE-equiv}
For the categorical likelihood \( p(y^{(i)} | \theta^{(i)}) = \theta_{y^{(i)}}^{(i)}\) and a uniform Dirichlet prior \( Q^{\mathrm{prior}} \),
setting \(\lambda = \nu^{-1}\) yields
\begin{equation}
    \ell_{\mathrm{U}}(\psi;i) = \nu \ell_{\mathrm{UCE}}(\psi;i) + C,
\end{equation}
where
\begin{equation}
    \ell_{\mathrm{UCE}}(\psi;i) =
    \E_{\theta^{(i)} \sim Q_\psi^{(i)}} [ -e_y^{(i) \top} \log \theta^{(i)} ] -
    \lambda H(Q_\psi^{(i)}),
\end{equation}
and \( C \) is independent of \( \psi \). Thus, \( \ell_\mathrm{U} \) is equivalent, up to scaling and additive constants, to \( \ell_\mathrm{UCE} \).
\end{lemma}
\begin{proof}
The result follows directly from~\eqref{eq-A:U-loss} and
\( \KL(Q_\psi^{(i)} \| Q^{\mathrm{prior}}) = -H(Q_\psi^{(i)}) + C' \) for a uniform Dirichlet prior.
\end{proof}

\begin{lemma}
For each node \( j \in \gV \) with
\begin{equation}
    w_j := \sum_{i \in \gT^{(\mathrm{tr})}} \Pi_{i,j}^{\mathrm{PPR}} >0,
\end{equation}
define the soft target
\begin{equation}
    \tilde{e}_y^{(j)} :=
    \frac{1}{w_j} \sum_{i \in \gT^{(\mathrm{tr})}} \Pi_{i,j}^{\mathrm{PPR}} e_y^{(i)}.
\end{equation}
The reformulated GPN-LOP node-wise loss satisfies
\begin{equation}
    \tilde{\ell}_{\mathrm{GPN-LOP}}(\psi;j) = w_j \tilde{\ell}_{\mathrm{UCE}}(\psi;j),
\end{equation}
where
\begin{equation}
    \tilde{\ell}_{\mathrm{UCE}}(\psi;j) := \E_{\theta^{(j)} \sim Q_\psi^{\mathrm{ft},(j)}} [ -\tilde{e}_y^{(j) \top} \log \theta^{(j)}] - \lambda H (Q_\psi^{\mathrm{ft},(j)}).
\end{equation}
Let
\begin{equation}
    \tilde{\ell}_{\mathrm{U}}(\psi;j) := \KL (Q_\psi^{\mathrm{ft},(j)} \| \tilde Q_\nu^{(j)}),
\end{equation}
where \( \tilde Q_\nu^{(j)} \) denotes the tempered posterior under the soft-target categorical likelihood, \( p(\tilde e_y^{(j)} | \theta^{(j)}) \propto \prod_{c=1}^C (\theta_c^{(j)})^{\tilde e_{y,c}^{(j)}} \).
Hence, under this soft-target likelihood, a uniform Dirichlet prior \( Q^{\mathrm{prior}} \), and \( \lambda = \nu^{-1} \),
\begin{equation}
    \tilde{\ell}_{\mathrm{U}}(\psi;j) = \frac{\nu}{w_j} \tilde{\ell}_{\mathrm{GPN-LOP}}(\psi;j) + C,
\end{equation}
where \( C \) is independent of \( \psi \). Thus, \( \tilde{\ell}_{\mathrm{GPN-LOP}} \) is equivalent, up to scaling and an additive constant, to the unified loss \( \tilde{\ell}_\mathrm{U} \) associated with the soft target \( \tilde e_y^{(j)} \).
\end{lemma}
\begin{proof}
Starting from the training objective,
\begin{equation}
\gL(\psi) = \sum_{i \in \gT^{(\mathrm{tr})}} \sum_{j \in \gV} \Pi_{i,j}^{\mathrm{PPR}} \!\left( \E_{\theta^{(j)} \sim Q_\psi^{\mathrm{ft}, (j)}} [-e_y^{(i) \top} \log \theta^{(j)}] - \lambda H [Q_\psi^{\mathrm{ft}, (j)} ] \right),
\end{equation}
interchanging the finite sums gives
\begin{equation}
    \gL(\psi) = \sum_{j \in \gV} \tilde{\ell}_{\mathrm{GPN-LOP}}(\psi;j).
\end{equation}
For fixed \( j \), linearity of expectation yields
\begin{align}
    \tilde{\ell}_{\mathrm{GPN-LOP}}(\psi;j) &= \E_{\theta^{(j)} \sim Q_\psi^{\mathrm{ft},(j)}}
    \!\left[ - ( \sum_{i \in \gT^{(\mathrm{tr})}} \Pi_{i,j}^{\mathrm{PPR}} e_y^{(i)} )^\top \log \theta^{(j)} \right] \\ & \quad - \lambda \sum_{i \in \gT^{(\mathrm{tr})}} \Pi_{i,j}^{\mathrm{PPR}} H (Q_\psi^{\mathrm{ft},(j)}) \\ &=
    w_j \!\left( \E_{\theta^{(j)} \sim Q_\psi^{\mathrm{ft},(j)}} [ -\tilde e_y^{(j) \top} \log \theta^{(j)} ] - \lambda H (Q_\psi^{\mathrm{ft},(j)} )\right) \\ &= w_j \tilde \ell_{\mathrm{UCE}}(\psi;j).
\end{align}
Applying Lemma~\ref{lemma-A:U-UCE-equiv} to this soft-target extension,
\begin{equation}
    \tilde{\ell}_{\mathrm{U}}(\psi;j) = \nu \tilde{\ell}_{\mathrm{UCE}}(\psi;j) + C,
\end{equation}
and substituting \( \tilde{\ell}_{\mathrm{UCE}}(\psi;j) = w_j^{-1} \tilde{\ell}_{\mathrm{GPN-LOP}}(\psi;j) \) gives the result.
\end{proof}

We next characterize the unconstrained minimizer of the unified objective.
\begin{theorem}
    \label{app:thm-U-loss-opt}
    The population minimizer of the unified EDL loss \( \ell_\mathrm{U} \) under \( \Prob_\phi \) satisfies
    \begin{equation}
        \argmin_{\psi}\, \E_{\Prob_\phi(\rmX, \rmA, Y)} [ \KL ( Q_\psi^{(i)} \| Q_\nu(\cdot | y^{(i)})) ] = \argmin_{\psi}\, \E_{\Prob_\phi(\rmA,\rmX)} [
        \KL(Q_\psi^{(i)} \| Q^{\star,(i)}) ],
        \end{equation}
        where
        \begin{equation}
            Q^{\star, (i)} (\theta) = \frac{Q^{\mathrm{prior}}(\theta) \exp (\nu \E_{\Prob_\phi(y^{(i)}|\rmX, \rmA)} [\log p(y ^{(i)}|\theta^{(i)})])}{\int_{\Delta_C} Q^{\mathrm{prior}}(\theta') \exp (\nu \E_{\Prob_\phi(y^{(i)}|\rmX, \rmA)} [\log p(y^{(i)}|\theta')]) \, d\theta'}.
        \end{equation}
\end{theorem}
\begin{proof}
The proof follows analogously to the i.i.d.\ case in Theorem 5.1 of~\citet{Shen-2024}, with \( (\rmA, \rmX) \) replacing the i.i.d.\ covariates.
\end{proof}

\subsection{Model-Specific Population Optima}
\subsubsection{GPN-LOP}
\label{app:sec-LOP-GPN}
For GPN-LOP, the population-optimal feature-based Dirichlet distribution at node \( j \) is determined by the PPR-weighted conditional class probabilities \( \tilde{\eta}^{(j)} = w_j^{-1} \sum_{k \in \gT^{(\mathrm{tr})}} \Pi_{k,j}^{\mathrm{PPR}} \eta^{(k)} \), where \( w_j = \sum_{k \in \gT^{(\mathrm{tr})}} \Pi_{k,j}^{\mathrm{PPR}} \) and \( \eta^{(k)} = \E_{\Prob_\phi(y^{(k)} | \rmA,\rmX)}[e_y^{(k)}] \). Under the categorical likelihood and uniform Dirichlet prior,
Theorem~\ref{app:thm-U-loss-opt} therefore gives \( Q^{\mathrm{ft} \star,(j)} = \Dir\!\left( \theta; \alpha^{\mathrm{prior}} + \nu\tilde{\eta}^{(j)} \right) \). Since \( \sum_{c=1}^C \tilde{\eta}_c^{(j)} = 1 \), each component has total concentration \( \alpha_0^{\mathrm{prior}}+\nu \), which remains finite for fixed \(\nu\). GPN-LOP aggregates these feature-based distributions through PPR, \( Q^{\star,(i)} = \sum_{j\in\gV} \Pi_{i,j}^{\mathrm{PPR}} Q^{\mathrm{ft} \star,(j)} \). Its epistemic uncertainty decomposes as
\begin{equation}
    \EU(Q^{(i)}) = \sum_{j \in \gV} \Pi^{\mathrm{PPR}}_{i,j} \EU(Q^{\mathrm{ft}, (j)}) + \sum_{j \in \gV} \Pi^{\mathrm{PPR}}_{i,j} \KL(\bar{\theta}^{(j)} \| \bar{\theta}^{(i)}),
\end{equation}
where \(\bar{\theta}^{(i)} = \sum_{j \in \gV} \Pi^{\mathrm{PPR}}_{i,j} \bar{\theta}^{(j)} \). 

The first term is strictly positive for finite \( \nu \), since each Dirichlet component has finite concentration, while the second term is non-negative. Hence, the GPN-LOP population optimum retains strictly positive epistemic uncertainty even when the component means agree.

\subsubsection{S-BGCN-K}
\label{app:sec-S-BGCN-K}
Theorem~\ref{app:thm-U-loss-opt} characterizes the population optimum for a fixed prior and therefore does not directly apply to S-BGCN-K, whose GKDE prior depends on the observed graph and training labels. Marginalizing this data-dependent prior in the population objective yields an effective prior \( Q_n^{\mathrm{prior},(i)} \), defined, for the data-dependent GKDE prior \( Q_{G_{A_n}}^{\mathrm{prior},(i)} \), by
\begin{equation}
Q_n^{\mathrm{prior}, (i)}(\theta) \propto \exp( \E_{G_{A_n} \sim \Prob_{A_n, \phi}} [ \log Q_{G_{A_n}}^{\mathrm{prior}, (i)}(\theta) ] ).
\end{equation}
This effective prior changes with graph size \( n \). We first show that convergence of these priors implies convergence of the corresponding population optima, and then establish conditions under which the effective GKDE prior concentrates under graph growth.

\begin{lemma}[Population optima under convergent priors]
\label{lem:seq-graph-prior}
Let \( \{Q^{\mathrm{prior}, (i)}_{{n}} \}_{n\geq1}\) be a sequence of priors such that \( Q^{\mathrm{prior}, (i)}_n \Rightarrow Q^{\mathrm{prior}, (i)}_\infty \). For fixed finite \( \nu \), suppose that \( g^{(i)} \) is continuous and strictly positive on \( \Delta_C \). Then, the corresponding population optima
\begin{equation}
    Q_n^{\star, (i)}(d\theta) = \frac{g^{(i)}(\theta) Q^{\mathrm{prior}, (i)}_n (d\theta)}{\int_{\Delta_C} g^{(i)}(\vartheta) Q^{\mathrm{prior}, (i)}_n (d\vartheta)}
\end{equation}
satisfy \( Q_n^{\star,(i)} \Rightarrow Q_\infty^{\star,(i)} \), where \( Q_\infty^{\star,(i)} \) is the population optimum of Theorem~\ref{app:thm-U-loss-opt} associated with the limiting prior \( Q_\infty^{\mathrm{prior},(i)} \), namely,
\begin{equation}
    Q_\infty^{\star,(i)}(d\theta) = \frac{ g^{(i)}(\theta) Q^{\mathrm{prior}, (i)}_\infty (d\theta) }{ \int_{\Delta_C} g^{(i)}(\vartheta) Q^{\mathrm{prior}, (i)}_\infty(d\vartheta)}.
\end{equation}
Moreover, \( Q_\infty^{\star,(i)} \) is a Dirac distribution if and only if \( Q^{\mathrm{prior}, (i)}_{\infty} \) is a Dirac distribution.
\end{lemma}
\begin{proof}
For any bounded continuous \( f: \Delta_C \to \R \),
\begin{equation}
    \int_{\Delta_C} f(\theta) Q_n^{\star, (i)}(d\theta) = \frac{\int_{\Delta_C} f(\theta) g^{(i)}(\theta) Q^{\mathrm{prior},(i)}_n (d\theta) }{\int_{\Delta_C} g^{(i)}(\theta) Q^{\mathrm{prior}, (i)}_n (d\theta)}.
\end{equation}
Since \( \Delta_C \) is compact, continuity of \( g^{(i)} \) implies that \( g^{(i)} \) and \( f g^{(i)} \) are bounded and continuous. The assumed weak convergence \( Q^{\mathrm{prior}, (i)}_n \Rightarrow Q^{\mathrm{prior}, (i)}_\infty \) therefore gives
\begin{equation}
    \int f g^{(i)} dQ^{\mathrm{prior},(i)}_n \to \int f g^{(i)} dQ^{\mathrm{prior}, (i)}_\infty \quad \text{and} \quad \int g^{(i)} dQ^{\mathrm{prior},(i)}_n \to \int g^{(i)} dQ^{\mathrm{prior}, (i)}_\infty.
\end{equation}
Since \(g^{(i)}\) is strictly positive on \(\Delta_C\), the latter limit
is strictly positive, and hence \( Q_n^{\star,(i)} \Rightarrow Q_\infty^{\star,(i)} \). Moreover, strict positivity implies that \( Q_\infty^{\star,(i)} \) and \( Q^{\mathrm{prior}, (i)}_\infty \) have the same null sets, and hence one is concentrated at a single point if and only if the other is.
\end{proof}

For the Gaussian likelihood underlying \( \ell_{\mathrm{MSE}} \), we define the likelihood-derived weighting function \(  g^{(i)}(\theta) = \exp \{ \nu\, \E_{\Prob_\phi(y^{(i)} | \rmX, \rmA)} [\log p(y^{(i)} | \theta)] \} \), which is continuous and strictly positive on the compact simplex. Hence, Lemma~\ref{lem:seq-graph-prior} applies to any convergent sequence of priors. We next give sufficient conditions for a sequence of Dirichlet priors to concentrate.
\begin{lemma}[Concentration of Dirichlet priors]
\label{lem:dirichlet-concentration}
Let
\begin{equation}
    Q^{\mathrm{prior}, (i)}_n = \Dir(\theta; \beta_n^{(i)}), \qquad \beta_{0,n}^{(i)} = \sum_{c=1}^C \beta_{n,c}^{(i)}, \qquad m_n^{(i)} = \frac{\beta_n^{(i)}}{\beta_{0,n}^{(i)}}.
\end{equation}
If \( \beta_{0,n}^{(i)} \to \infty \) and \( m_n^{(i)} \to m^{(i)} \in \Delta_C \),  then \( Q^{\mathrm{prior},(i)}_n \Rightarrow \delta_{m^{(i)}} \).
\end{lemma}
\begin{proof}
Let \( \Theta_n^{(i)} \sim Q^{\mathrm{prior},(i)}_n \). Then, \( \E[\Theta_n^{(i)}] = m_n^{(i)} \). Moreover, since
\( m_{n,c}^{(i)} \in [0,1] \), we have
\( m_{n,c}^{(i)}(1-m_{n,c}^{(i)}) \leq 1/4 \), and thus, for each class \( c \),
\begin{equation}
    \Var(\Theta_{n,c}^{(i)}) = \frac{ m_{n,c}^{(i)} (1-m_{n,c}^{(i)})}{ \beta_{0,n}^{(i)}+1} \leq \frac{1}{4(\beta_{0,n}^{(i)}+1)} \to 0,
\end{equation}
where the convergence follows from \( \beta_{0,n}^{(i)} \to \infty \).
Therefore, \( \Theta_n^{(i)} - m_n^{(i)}
\to 0 \) in probability. Since \(m_n^{(i)}\to m^{(i)}\), it follows that \( \Theta_n^{(i)}
\to m^{(i)} \) in probability, and hence \( Q^{\mathrm{prior}, (i)}_n \Rightarrow \delta_{m^{(i)}} \).
\end{proof}

Thus, if \( \beta_{0,n}^{(i)} \to \infty \) and \(m_n^{(i)} \to m^{(i)}\), Lemmas~\ref{lem:dirichlet-concentration} and~\ref{lem:seq-graph-prior} give \( Q_n^{\mathrm{prior},(i)} \Rightarrow \delta_{m^{(i)}} \) and \( Q_n^{\star,(i)} \Rightarrow \delta_{m^{(i)}} \). Consequently,
\begin{equation}
\EU(Q_n^{\star,(i)}) \to 0 \quad \text{and} \quad \AU(Q_n^{\star,(i)}) \to H(m^{(i)}).
\end{equation}

S-BGCN-K provides a graph-dependent Dirichlet prior through GKDE, with parameters
\begin{equation}
\beta_{n,c}^{(i)} = 1 + \sum_{j \in \gT_n^{(\mathrm{tr})}} \mathbf{1}\{y^{(j)}=c\} g_\sigma(d_{ij}),
\end{equation}
where \( d_{ij} \) denotes the graph distance between nodes \( i \) and \( j \), and \( g_\sigma(d_{ij}) \) is the Gaussian kernel applied to this distance, with bandwidth \( \sigma \). Since the expectation of a Dirichlet log-density differs from the log-density of a Dirichlet distribution with expected parameters only by a \( \theta \)-independent constant, the effective population prior is
\begin{equation}
Q_n^{\mathrm{prior},(i)} =
\Dir(\theta;\beta_n^{(i)}), \qquad \beta_{n,c}^{(i)} = \E_{G_{A_n} \sim \Prob_{A_n, \phi}} [\beta_{n,c}^{(i)} ].
\end{equation}
Summing over classes gives the total concentration of the effective prior,
\begin{equation}
    \beta_{0,n}^{(i)} = C + \E_{G_{A_n} \sim \Prob_{A_n, \phi}}\!\Bigl[ \sum_{j \in \gT_n^{(\mathrm{tr})}} g_\sigma(d_{ij}) \Bigr].
\end{equation}
Thus, concentration grows with the expected amount of kernel-weighted labeled evidence available around node \( i \). In particular, for any fixed
finite \( D \),
\begin{equation}
    \beta_{0,n}^{(i)} \geq C + g_\sigma(D) \E_{G_{A_n} \sim \Prob_{A_n, \phi}}[ | \{ j \in \gT_n^{(\mathrm{tr})} : d_{ij} \leq D \}| ],
\end{equation}
where we use that the Gaussian kernel is non-increasing in distance. Hence, a sufficient condition for \( \beta_{0,n}^{(i)} \to \infty \) is the existence of a fixed \( D \) such that \( \E_{G_{A_n} \sim \Prob_{A_n, \phi}}[ | \{ j \in \gT_n^{(\mathrm{tr})} : d_{ij} \leq D \} | ] \to \infty \). Under the dense projective graph DGPs considered in our experiments, positive population normalized degree for node \( i \), together with a training-label sampling scheme that retains a positive fraction of its neighbors, yields a diverging expected number of labeled one-hop neighbors, so the condition holds with \( D=1 \). If, in addition, the normalized GKDE evidence converges, \( m_n^{(i)} \to m^{(i)} \), Lemma~\ref{lem:dirichlet-concentration} gives \( Q^{\mathrm{prior}, (i)}_n \Rightarrow \delta_{m^{(i)}} \) and Lemma~\ref{lem:seq-graph-prior} then yields \( Q_n^{\star,(i)} \Rightarrow \delta_{m^{(i)}} \). 

However, \( m^{(i)} \) is the limiting vector of normalized population-averaged
kernel-weighted class evidence induced by the GKDE construction and is not guaranteed to coincide with the Bayes-optimal predictor \( p_\phi^{\star,(i)} \), as it depends on choices such as the GKDE bandwidth and the graph-distance weighting function. Hence, despite concentration of the evidential distribution, Definition~\ref{def:consistent-predictor} need not be satisfied, and \(\AU(Q_n^{\star,(i)}) \to H(m^{(i)})\) need not equal the true aleatoric uncertainty \( H(p_\phi^{\star,(i)}) \).

The analysis above considers S-BGCN-K with Monte Carlo (MC) dropout disabled, thereby isolating the uncertainty induced by the evidential objective. When MC dropout with a fixed Bernoulli dropout rate is enabled, as in the original formulation, concentration of the evidential distribution is no longer sufficient for concentration of the resulting predictive distribution: the predictions must additionally become invariant to the dropout mask as information grows. Fixed-rate MC dropout does not in general ensure such concentration, and has been shown to retain predictive variability independently of the amount of observed data in simple settings \citep{Verdoja-2021}.

\section{Synthetic Graphons}
\label{app:synthetic-graphons}
We consider two complementary homophilic graph DGPs: a smooth exponential-decay graphon, representing continuous latent similarity, and an assortative stochastic block model (SBM) graphon, representing discrete community structure, both commonly considered in the literature~\citep{Xu-2020}. Besides providing qualitatively different graph structure, the SBM connects our construction to projective dyadic-independence models studied within the exponential random-graph model framework of~\citet{Shalizi-2013}. Both follow the graphon construction
\begin{align}
    U_i &\overset{\mathrm{iid}}{\sim} \mathrm{Unif}(0,1), \\
    A_{ij} | U_i,U_j
        &\sim \mathrm{Bernoulli} (W(U_i,U_j)),
\end{align}
where \( W:[0,1]^2 \rightarrow [0,1] \) maps pairs of latent positions to edge probabilities. For the exponential-decay graphon, this mapping depends directly on the distance between latent positions,
\begin{equation}
    W_{\mathrm{exp}}(U_i,U_j)
    = \exp (-\beta |U_i-U_j|),
\end{equation}
such that connectivity decreases smoothly with latent distance. For the SBM graphon, the latent space is instead partitioned into $K$ blocks through a mapping \( b:[0,1] \rightarrow \{1,\ldots,K\} \). Edge probabilities are then determined by the corresponding entries of an assortative block probability matrix \( P \), yielding the step-function graphon
\begin{equation}
    W_{\mathrm{SBM}}(U_i,U_j)
    = P_{b(U_i),b(U_j)}.
\end{equation}

Both DGPs define balanced seven-class node-classification problems in which labels and features depend on the latent position, with
\begin{align}
    Y_i \mid U_i &\sim P_Y(\,\cdot \mid U_i), \\
    X_i \mid U_i &\sim P_X(\,\cdot \mid U_i).
\end{align}
For the exponential-decay graphon, class probabilities vary smoothly over the latent space, with greater label ambiguity near transitions between neighboring classes. For the SBM graphon, class probabilities follow the block structure but vary smoothly within blocks, with greater ambiguity near block transitions.

Each node additionally receives 12 noisy continuous and binary features that are informative of the latent class structure without deterministically identifying the observed label. Feature informativeness is likewise allowed to vary with latent position, such that graph structure and node features provide complementary and imperfect information about the class.

Complete generator specifications and parameter settings are provided in Table~\ref{tab:graphon-parameters}. Tables~\ref{tab:graphon1-stats}--\ref{tab:graphon2-stats} report the realized structural properties of the generated graph sequences across \( 20 \) independent realizations. 

Each projective graphon sequence comprises six graph sizes, logarithmically spaced between \( 500 \) and \( 10{,}000 \) nodes. We first generate the largest graph and construct each smaller graph as the induced subgraph on its first \( n \) nodes, thereby yielding a nested sequence from the same underlying realization.

\begin{table}[t]
\centering
\scriptsize
\caption{Parameters used for the synthetic graph DGPs.}
\label{tab:graphon-parameters}
\begin{tabular}{lll}
\toprule
Component & Parameter & Value \\
\midrule

Generation
    & Graph sizes $n$
    & $500, 910, 1657, 3017, 5493, 10000$ \\
    & Independent sequences
    & $20$ \\

\midrule
Exponential graphon
    & Decay $\beta$
    & $15$ \\

\midrule
SBM graphon
    & Blocks $K$
    & $7$ \\
    & Between-block probability $p_{\mathrm{out}}$
    & $0.02$ \\
    & Within-block increment $\delta$
    & $0.83$ \\
    & Block-probability noise $\sigma_P$
    & $0.02$ \\

\midrule
Features
    & Dimension $d$
    & $12$ \\
    & Continuous-feature noise $\sigma_X$
    & $0.04$ \\
    & Binary-feature flip probability
    & $0.20$ \\

\midrule
Exponential labels
    & Gaussian width $\sigma_Y$
    & $0.03$ \\

\midrule
SBM labels
    & Center / boundary class probability
    & $0.90 / 0.65$ \\
    & Boundary width
    & $0.15$ \\

\midrule
SBM features
    & Center / boundary class probability
    & $0.80 / 0.55$ \\
    & Boundary width
    & $0.15$ \\

\bottomrule
\end{tabular}
\end{table}
\FloatBarrier

\section{Local Bootstrap}
\label{app:bootstrap}
We construct bootstrap graph replicates using a non-parametric procedure adapted from the local probability estimator of~\citet{Jeong-2025}, developed under an exchangeable graphon model and exploiting node similarity. The procedure preserves the observed node set, features and labels across replicates, while resampling the graph structure from smoothed edge probabilities.

For each node \( i \in \gV \), we first construct a set \( \gN_k(i) \) of its \( k \) most similar nodes according to a chosen node similarity metric. In the synthetic graphon experiments, where the latent coordinates \( U_i \) are known, we use Euclidean distance in latent space,
\begin{equation}
    d(i,j) = \|U_i-U_j\|_2, \qquad i \neq j.
\end{equation}
For the OOD benchmark datasets, where latent coordinates are unavailable, we instead use
the open-neighborhood Jaccard similarity suggested by~\citet{Jeong-2025},
\begin{equation}
    J(i,j) = \frac{|\gN(i) \cap \gN(j)|}
    {|\gN(i) \cup \gN(j)|}, \qquad i \neq j,
\end{equation}
where \( \gN(i) \) denotes the open graph neighborhood of node \( i \). This criterion favors nodes with similar observed connectivity profiles, so that the selected neighbors provide local substitutes for node \( i \) when estimating its edge probabilities.

The directed edge probability between nodes \( i \) and \( j \) is then estimated by averaging the observed adjacency of \( j \) over \( \gN_k(i) \):
\begin{equation}
    \hat{p}_{ij} = \frac{1}{|\gN_k(i)|}
    \sum_{u \in \gN_k(i)} \rmA_{uj}.
\end{equation}
Thus, \( \hat{p}_{ij} \) is the empirical fraction of nodes locally similar to \( i \) that are adjacent to \( j \) in the observed graph. Because this estimate is generally asymmetric, we obtain an undirected edge probability by
\begin{equation}
    \hat{p}^{\mathrm{sym}}_{ij} = \frac{1}{2} ( \hat{p}_{ij} + \hat{p}_{ji}),
\end{equation}
which preserves the asymptotic unbiasedness of the local probability estimator under the conditions of~\citet{Jeong-2025}.

Conditional on the resulting probability matrix, bootstrap graphs are generated independently by sampling each unordered node pair according to
\begin{equation}
    A^{\mathrm{boot}}_{ij,b} | \hat{p}_{ij}^\mathrm{sym} \overset{\mathrm{iid}}{\sim} \mathrm{Bernoulli} (\hat{p}^{\mathrm{sym}}_{ij}), \qquad i < j, \quad b \in \{1, \ldots, B\},
\end{equation}
and setting \( A_{ii,b}^{\mathrm{boot}} = 0 \). The complete graph resampling procedure is given in Algorithm~\ref{alg:graph-bootstrap}. Following the finite-sample specification of~\citet{Jeong-2025}, we use \( k=20 \) for both the synthetic graphons and the node classification benchmark datasets.

\begin{algorithm}[t]
\caption{Conditional local graph bootstrap}
\label{alg:graph-bootstrap}
\small
\begin{algorithmic}[1]
\Require Graph \( \gG=(A, X) \), neighborhood size \( k \), pairwise node similarity metric \( d \), number of replicates \(B\)
\Ensure Bootstrap graphs \( \{\gG^{\mathrm{boot}}_b =(A^{\mathrm{boot}}_b,X)\}_{b=1}^B \)

\For{\( i \in \gV \)}
    \State Construct \( \gN_k(i) \) as the \( k \) nearest nodes similar to \( i \) according to \( d \)
\EndFor

\For{\(i,j\in\gV\), \(i\neq j\)}
    \State \(\displaystyle \hat p_{ij} \gets \frac{1}{|\gN_k(i)|} \sum_{u \in \gN_k(i)} A_{uj}\)
\EndFor

\For{\(i<j\)}
    \State \(\displaystyle \hat p^{\mathrm{sym}}_{ij} \gets \frac{\hat p_{ij} + \hat p_{ji}}{2}\)
\EndFor

\For{\(b=1,\ldots,B\)}
    \For{\(i<j\)}
        \State \(\displaystyle A^{\mathrm{boot}}_{ij,b} \sim \mathrm{Bernoulli} \left(\hat p^{\mathrm{sym}}_{ij}\right)\) \State \(A^{\mathrm{boot}}_{ji,b} \gets A^{\mathrm{boot}}_{ij,b}\)
    \EndFor
    \State \(A^{\mathrm{boot}}_{ii,b}\gets0\)
    for all \(i \in \gV\)
    \State \(\gG^{\mathrm{boot}}_b
    \gets (A^{\mathrm{boot}}_b,X)\)
\EndFor
\State \Return \( \{ \gG^{\mathrm{boot}}_b=(A^{\mathrm{boot}}_b, X) \}_{b=1}^B \).
\end{algorithmic}
\end{algorithm}

The statistics of \( 25 \) bootstrap graphs are compared with those of the original graphs for the synthetic graphons and node classification benchmark datasets in Tables~\ref{tab:graphon1-stats}--~\ref{tab:graphon2-stats} and~\ref{tab:MLdata-stats}, respectively.

% Requires:
% \usepackage{booktabs}
% \usepackage{longtable}
% \usepackage{array}
%
% The caption remains at the document's normal caption size.
% Only the table cells are set in \tiny.

% ============================================================
% Exponential-decay graphon
% ============================================================

% ============================================================
% Exponential-decay graphon
% ============================================================

{
\setlength{\LTleft}{\fill}
\setlength{\LTright}{\fill}
\setlength{\tabcolsep}{3pt}
\renewcommand{\arraystretch}{0.95}

\begin{longtable}{
    >{\scriptsize}l
    >{\scriptsize}c
    >{\scriptsize}c
}
\caption{Original and bootstrap graph statistics for the exponential-decay graphon at increasing sizes. Values are mean \( \pm \) standard deviation across \(20\) graphon realizations, with bootstrap statistics averaged across \(25\) replicates per realization.}
\label{tab:graphon1-stats}\\
\endfirsthead
\endhead
\endfoot
\endlastfoot

% n = 500
\toprule
& \multicolumn{2}{c}{\tiny $n=500$} \\
\cmidrule(lr){2-3}
Metric & Original & Bootstrap \\
\midrule
Features & $12$ & $12$ \\
Classes & $7$ & $7$ \\
Nodes & $500$ & $500$ \\
Edges
Edges
& $(1.547\times10^{4}) \pm 212.020$
& $(1.530\times10^{4}) \pm 220.126$ \\
Average degree
& $61.882 \pm 0.848$
& $61.206 \pm 0.881$ \\
Degree std.
& $10.118 \pm 1.385$
& $10.327 \pm 1.311$ \\
Density
& $(1.240\times10^{-1}) \pm (1.700\times10^{-3})$
& $(1.227\times10^{-1}) \pm (1.765\times10^{-3})$ \\
Local clustering
& $0.404 \pm 0.005$
& $0.394 \pm 0.004$ \\
Components & $1$ & $1$ \\
Giant component & $500$ & $500$ \\
Assortativity
& $0.416 \pm 0.096$
& $0.371 \pm 0.093$ \\
Node homophily
& $0.631 \pm 0.011$
& $0.623 \pm 0.012$ \\
Edge homophily
& $0.620 \pm 0.013$
& $0.612 \pm 0.013$ \\
Transitivity
& $0.396 \pm 0.005$
& $0.386 \pm 0.005$ \\
Triangle density
& $(6.169\times10^{-3}) \pm (1.826\times10^{-4})$
& $(5.889\times10^{-3}) \pm (1.837\times10^{-4})$ \\
\bottomrule

\addlinespace[0.8em]

% n = 910
\toprule
& \multicolumn{2}{c}{\tiny $n=910$} \\
\cmidrule(lr){2-3}
Metric & Original & Bootstrap \\
\midrule
Features & $12$ & $12$ \\
Classes & $7$ & $7$ \\
Nodes & $910$ & $910$ \\
Edges
& $(5.131\times10^{4}) \pm 361.446$
& $(5.095\times10^{4}) \pm 362.437$ \\
Average degree
& $112.768 \pm 0.794$
& $111.969 \pm 0.797$ \\
Degree std.
& $16.285 \pm 1.714$
& $16.549 \pm 1.712$ \\
Density
& $(1.241\times10^{-1}) \pm (8.739\times10^{-4})$
& $(1.232\times10^{-1}) \pm (8.763\times10^{-4})$ \\
Local clustering
& $0.404 \pm 0.003$
& $0.399 \pm 0.003$ \\
Components & $1$ & $1$ \\
Giant component & $910$ & $910$ \\
Assortativity
& $0.496 \pm 0.082$
& $0.464 \pm 0.081$ \\
Node homophily
& $0.632 \pm 0.005$
& $0.628 \pm 0.005$ \\
Edge homophily
& $0.621 \pm 0.005$
& $0.617 \pm 0.005$ \\
Transitivity
& $0.396 \pm 0.003$
& $0.391 \pm 0.003$ \\
Triangle density
& $(6.170\times10^{-3}) \pm (9.183\times10^{-5})$
& $(6.013\times10^{-3}) \pm (9.348\times10^{-5})$ \\
\bottomrule

\addlinespace[0.8em]

% n = 1657
\toprule
& \multicolumn{2}{c}{\tiny $n=1657$} \\
\cmidrule(lr){2-3}
Metric & Original & Bootstrap \\
\midrule
Features & $12$ & $12$ \\
Classes & $7$ & $7$ \\
Nodes & $1657$ & $1657$ \\
Edges
Edges
& $(1.705\times10^{5}) \pm (1.041\times10^{3})$
& $(1.697\times10^{5}) \pm (1.034\times10^{3})$ \\
Average degree
& $205.798 \pm 1.256$
& $204.886 \pm 1.249$ \\
Degree std.
& $27.837 \pm 2.599$
& $28.184 \pm 2.567$ \\
Density
& $(1.243\times10^{-1}) \pm (7.585\times10^{-4})$
& $(1.237\times10^{-1}) \pm (7.540\times10^{-4})$ \\
Local clustering
& $0.404 \pm 0.002$
& $0.401 \pm 0.002$ \\
Components & $1$ & $1$ \\
Giant component & $1657$ & $1657$ \\
Assortativity
& $0.572 \pm 0.062$
& $0.548 \pm 0.062$ \\
Node homophily
& $0.631 \pm 0.004$
& $0.629 \pm 0.004$ \\
Edge homophily
& $0.620 \pm 0.004$
& $0.618 \pm 0.004$ \\
Transitivity
& $0.396 \pm 0.003$
& $0.393 \pm 0.002$ \\
Triangle density
& $(6.196\times10^{-3}) \pm (8.944\times10^{-5})$
& $(6.109\times10^{-3}) \pm (8.044\times10^{-5})$ \\
\bottomrule

\addlinespace[0.8em]

% n = 3017
\toprule
& \multicolumn{2}{c}{\tiny $n=3017$} \\
\cmidrule(lr){2-3}
Metric & Original & Bootstrap \\
\midrule
Features & $12$ & $12$ \\
Classes & $7$ & $7$ \\
Nodes & $3017$ & $3017$ \\
Edges
& $(5.661\times10^{5}) \pm (2.464\times10^{3})$
& $(5.646\times10^{5}) \pm (2.505\times10^{3})$ \\
Average degree
& $375.244 \pm 1.634$
& $374.306 \pm 1.661$ \\
Degree std.
& $48.265 \pm 2.831$
& $48.657 \pm 2.799$ \\
Density
& $(1.244\times10^{-1}) \pm (5.417\times10^{-4})$
& $(1.241\times10^{-1}) \pm (5.506\times10^{-4})$ \\
Local clustering
& $0.403 \pm 0.002$
& $0.402 \pm 0.002$ \\
Components & $1$ & $1$ \\
Giant component & $3017$ & $3017$ \\
Assortativity
& $0.609 \pm 0.043$
& $0.593 \pm 0.044$ \\
Node homophily
& $0.631 \pm 0.003$
& $0.630 \pm 0.003$ \\
Edge homophily
& $0.620 \pm 0.003$
& $0.619 \pm 0.003$ \\
Transitivity
& $0.395 \pm 0.002$
& $0.395 \pm 0.001$ \\
Triangle density
& $(6.206\times10^{-3}) \pm (5.918\times10^{-5})$
& $(6.170\times10^{-3}) \pm (6.114\times10^{-5})$ \\
\bottomrule

\addlinespace[0.8em]

% n = 5493
\toprule
& \multicolumn{2}{c}{\tiny $n=5493$} \\
\cmidrule(lr){2-3}
Metric & Original & Bootstrap \\
\midrule
Features & $12$ & $12$ \\
Classes & $7$ & $7$ \\
Nodes & $5493$ & $5493$ \\
Edges
& $(1.878\times10^{6}) \pm (6.496\times10^{3})$
& $(1.875\times10^{6}) \pm (6.503\times10^{3})$ \\
Average degree
& $683.611 \pm 2.365$
& $682.675 \pm 2.368$ \\
Degree std.
& $85.005 \pm 4.412$
& $85.439 \pm 4.400$ \\
Density
& $(1.245\times10^{-1}) \pm (4.306\times10^{-4})$
& $(1.243\times10^{-1}) \pm (4.311\times10^{-4})$ \\
Local clustering
& $0.404 \pm 0.002$
& $0.403 \pm 0.001$ \\
Components & $1$ & $1$ \\
Giant component & $5493$ & $5493$ \\
Assortativity
& $0.632 \pm 0.035$
& $0.622 \pm 0.035$ \\
Node homophily
& $0.631 \pm 0.003$
& $0.631 \pm 0.003$ \\
Edge homophily
& $0.620 \pm 0.003$
& $0.619 \pm 0.003$ \\
Transitivity
& $0.397 \pm 0.002$
& $0.395 \pm 0.001$ \\
Triangle density
& $(6.236\times10^{-3}) \pm (5.362\times10^{-5})$
& $(6.199\times10^{-3}) \pm (4.463\times10^{-5})$ \\
\bottomrule

\addlinespace[0.8em]

% n = 10000
\toprule
& \multicolumn{2}{c}{\tiny $n=10000$} \\
\cmidrule(lr){2-3}
Metric & Original & Bootstrap \\
\midrule
Features & $12$ & $12$ \\
Classes & $7$ & $7$ \\
Nodes & $10000$ & $10000$ \\
Edges
& $(6.224\times10^{6}) \pm (1.645\times10^{4})$
& $(6.220\times10^{6}) \pm (1.655\times10^{4})$ \\
Average degree
& $1244.879 \pm 3.290$
& $1243.939 \pm 3.310$ \\
Degree std.
& $152.286 \pm 6.828$
& $152.737 \pm 6.834$ \\
Density
& $(1.245\times10^{-1}) \pm (3.291\times10^{-4})$
& $(1.244\times10^{-1}) \pm (3.310\times10^{-4})$ \\
Local clustering
& $0.404 \pm 0.002$
& $0.403 \pm 0.002$ \\
Components & $1$ & $1$ \\
Giant component & $10000$ & $10000$ \\
Assortativity
& $0.652 \pm 0.031$
& $0.646 \pm 0.031$ \\
Node homophily
& $0.631 \pm 0.002$
& $0.631 \pm 0.002$ \\
Edge homophily
& $0.620 \pm 0.002$
& $0.620 \pm 0.002$ \\
Transitivity
& $0.396 \pm 0.002$
& $0.396 \pm 0.001$ \\
Triangle density
& $(6.222\times10^{-3}) \pm (4.641\times10^{-5})$
& $(6.213\times10^{-3}) \pm (3.253\times10^{-5})$ \\
\bottomrule

\end{longtable}
}

% ============================================================
% Block-assortative graphon
% ============================================================

% ============================================================
% SBM graphon
% ============================================================

{
\setlength{\LTleft}{\fill}
\setlength{\LTright}{\fill}
\setlength{\tabcolsep}{3pt}
\renewcommand{\arraystretch}{0.95}

\begin{longtable}{
    >{\scriptsize}l
    >{\scriptsize}c
    >{\scriptsize}c
}
\caption{Original and bootstrap graph statistics for the SBM graphon at increasing sizes. Values are mean \( \pm \) standard deviation across \(20\) graphon realizations, with bootstrap statistics averaged across \(25\) replicates per realization.}
\label{tab:graphon2-stats}\\
\endfirsthead
\endhead
\endfoot
\endlastfoot

% n = 500
\toprule
& \multicolumn{2}{c}{\tiny $n=500$} \\
\cmidrule(lr){2-3}
Metric & Original & Bootstrap \\
\midrule
Features & $12$ & $12$ \\
Classes & $7$ & $7$ \\
Nodes & $500$ & $500$ \\
Edges
& $(1.755\times10^{4}) \pm 161.030$
& $(1.738\times10^{4}) \pm 169.203$ \\
Average degree
& $70.203 \pm 0.644$
& $69.502 \pm 0.677$ \\
Degree std.
& $8.206 \pm 1.813$
& $8.799 \pm 1.683$ \\
Density
& $(1.407\times10^{-1}) \pm (1.291\times10^{-3})$
& $(1.393\times10^{-1}) \pm (1.356\times10^{-3})$ \\
Local clustering
& $0.643 \pm 0.005$
& $0.550 \pm 0.008$ \\
Components & $1$ & $1$ \\
Giant component & $500$ & $500$ \\
Assortativity
& $0.570 \pm 0.127$
& $0.458 \pm 0.125$ \\
Node homophily
& $0.620 \pm 0.016$
& $0.601 \pm 0.016$ \\
Edge homophily
& $0.619 \pm 0.016$
& $0.600 \pm 0.017$ \\
Transitivity
& $0.641 \pm 0.007$
& $0.549 \pm 0.010$ \\
Triangle density
& $(1.270\times10^{-2}) \pm (2.959\times10^{-4})$
& $(1.070\times10^{-2}) \pm (3.086\times10^{-4})$ \\
\bottomrule

\addlinespace[0.8em]

% n = 910
\toprule
& \multicolumn{2}{c}{\tiny $n=910$} \\
\cmidrule(lr){2-3}
Metric & Original & Bootstrap \\
\midrule
Features & $12$ & $12$ \\
Classes & $7$ & $7$ \\
Nodes & $910$ & $910$ \\
Edges
& $(5.804\times10^{4}) \pm 237.804$
& $(5.767\times10^{4}) \pm 240.446$ \\
Average degree
& $127.554 \pm 0.523$
& $126.742 \pm 0.528$ \\
Degree std.
& $11.021 \pm 2.066$
& $11.708 \pm 1.985$ \\
Density
& $(1.403\times10^{-1}) \pm (5.750\times10^{-4})$
& $(1.394\times10^{-1}) \pm (5.814\times10^{-4})$ \\
Local clustering
& $0.644 \pm 0.002$
& $0.589 \pm 0.006$ \\
Components & $1$ & $1$ \\
Giant component & $910$ & $910$ \\
Assortativity
& $0.576 \pm 0.116$
& $0.501 \pm 0.115$ \\
Node homophily
& $0.612 \pm 0.011$
& $0.603 \pm 0.011$ \\
Edge homophily
& $0.612 \pm 0.010$
& $0.603 \pm 0.011$ \\
Transitivity
& $0.641 \pm 0.004$
& $0.586 \pm 0.007$ \\
Triangle density
& $(1.262\times10^{-2}) \pm (1.232\times10^{-4})$
& $(1.142\times10^{-2}) \pm (1.594\times10^{-4})$ \\
\bottomrule

\addlinespace[0.8em]

% n = 1657
\toprule
& \multicolumn{2}{c}{\tiny $n=1657$} \\
\cmidrule(lr){2-3}
Metric & Original & Bootstrap \\
\midrule
Features & $12$ & $12$ \\
Classes & $7$ & $7$ \\
Nodes & $1657$ & $1657$ \\
Edges
& $(1.926\times10^{5}) \pm 555.421$
& $(1.919\times10^{5}) \pm 563.833$ \\
Average degree
& $232.496 \pm 0.670$
& $231.675 \pm 0.681$ \\
Degree std.
& $15.500 \pm 3.951$
& $16.232 \pm 3.744$ \\
Density
& $(1.404\times10^{-1}) \pm (4.048\times10^{-4})$
& $(1.399\times10^{-1}) \pm (4.110\times10^{-4})$ \\
Local clustering
& $0.644 \pm 0.002$
& $0.613 \pm 0.004$ \\
Components & $1$ & $1$ \\
Giant component & $1657$ & $1657$ \\
Assortativity
& $0.575 \pm 0.138$
& $0.520 \pm 0.137$ \\
Node homophily
& $0.607 \pm 0.008$
& $0.603 \pm 0.007$ \\
Edge homophily
& $0.607 \pm 0.008$
& $0.602 \pm 0.007$ \\
Transitivity
& $0.642 \pm 0.003$
& $0.612 \pm 0.004$ \\
Triangle density
& $(1.266\times10^{-2}) \pm (1.082\times10^{-4})$
& $(1.200\times10^{-2}) \pm (1.083\times10^{-4})$ \\
\bottomrule

\addlinespace[0.8em]

% n = 3017
\toprule
& \multicolumn{2}{c}{\tiny $n=3017$} \\
\cmidrule(lr){2-3}
Metric & Original & Bootstrap \\
\midrule
Features & $12$ & $12$ \\
Classes & $7$ & $7$ \\
Nodes & $3017$ & $3017$ \\
Edges
& $(6.388\times10^{5}) \pm 915.115$
& $(6.375\times10^{5}) \pm 954.026$ \\
Average degree
& $423.445 \pm 0.607$
& $422.636 \pm 0.632$ \\
Degree std.
& $23.506 \pm 4.586$
& $24.242 \pm 4.521$ \\
Density
& $(1.404\times10^{-1}) \pm (2.011\times10^{-4})$
& $(1.401\times10^{-1}) \pm (2.097\times10^{-4})$ \\
Local clustering
& $0.644 \pm 0.003$
& $0.627 \pm 0.003$ \\
Components & $1$ & $1$ \\
Giant component & $3017$ & $3017$ \\
Assortativity
& $0.638 \pm 0.079$
& $0.594 \pm 0.086$ \\
Node homophily
& $0.612 \pm 0.006$
& $0.610 \pm 0.006$ \\
Edge homophily
& $0.611 \pm 0.006$
& $0.609 \pm 0.006$ \\
Transitivity
& $0.642 \pm 0.003$
& $0.626 \pm 0.003$ \\
Triangle density
& $(1.266\times10^{-2}) \pm (7.382\times10^{-5})$
& $(1.231\times10^{-2}) \pm (7.212\times10^{-5})$ \\
\bottomrule

\addlinespace[0.8em]

% n = 5493
\toprule
& \multicolumn{2}{c}{\tiny $n=5493$} \\
\cmidrule(lr){2-3}
Metric & Original & Bootstrap \\
\midrule
Features & $12$ & $12$ \\
Classes & $7$ & $7$ \\
Nodes & $5493$ & $5493$ \\
Edges
& $(2.118\times10^{6}) \pm (2.087\times10^{3})$
& $(2.116\times10^{6}) \pm (2.135\times10^{3})$ \\
Average degree
& $771.060 \pm 0.760$
& $770.270 \pm 0.777$ \\
Degree std.
& $37.061 \pm 6.232$
& $37.786 \pm 6.088$ \\
Density
& $(1.404\times10^{-1}) \pm (1.384\times10^{-4})$
& $(1.403\times10^{-1}) \pm (1.415\times10^{-4})$ \\
Local clustering
& $0.644 \pm 0.002$
& $0.634 \pm 0.002$ \\
Components & $1$ & $1$ \\
Giant component & $5493$ & $5493$ \\
Assortativity
& $0.690 \pm 0.052$
& $0.658 \pm 0.059$ \\
Node homophily
& $0.616 \pm 0.004$
& $0.615 \pm 0.004$ \\
Edge homophily
& $0.616 \pm 0.004$
& $0.615 \pm 0.004$ \\
Transitivity
& $0.642 \pm 0.003$
& $0.633 \pm 0.002$ \\
Triangle density
& $(1.268\times10^{-2}) \pm (5.112\times10^{-5})$
& $(1.247\times10^{-2}) \pm (3.925\times10^{-5})$ \\
\bottomrule

\addlinespace[0.8em]

% n = 10000
\toprule
& \multicolumn{2}{c}{\tiny $n=10000$} \\
\cmidrule(lr){2-3}
Metric & Original & Bootstrap \\
\midrule
Features & $12$ & $12$ \\
Classes & $7$ & $7$ \\
Nodes & $10000$ & $10000$ \\
Edges
& $(7.019\times10^{6}) \pm (3.921\times10^{3})$
& $(7.015\times10^{6}) \pm (3.931\times10^{3})$ \\
Average degree
& $1403.808 \pm 0.784$
& $1403.021 \pm 0.786$ \\
Degree std.
& $59.618 \pm 7.229$
& $60.364 \pm 7.116$ \\
Density
& $(1.404\times10^{-1}) \pm (7.843\times10^{-5})$
& $(1.403\times10^{-1}) \pm (7.863\times10^{-5})$ \\
Local clustering
& $0.644 \pm 0.003$
& $0.638 \pm 0.001$ \\
Components & $1$ & $1$ \\
Giant component & $10000$ & $10000$ \\
Assortativity
& $0.725 \pm 0.027$
& $0.703 \pm 0.029$ \\
Node homophily
& $0.618 \pm 0.004$
& $0.617 \pm 0.004$ \\
Edge homophily
& $0.617 \pm 0.004$
& $0.617 \pm 0.004$ \\
Transitivity
& $0.641 \pm 0.002$
& $0.637 \pm 0.001$ \\
Triangle density
& $(1.266\times10^{-2}) \pm (3.880\times10^{-5})$
& $(1.257\times10^{-2}) \pm (2.195\times10^{-5})$ \\
\bottomrule

\end{longtable}
}
{
\setlength{\LTleft}{\fill}
\setlength{\LTright}{\fill}
\setlength{\tabcolsep}{2pt}
\renewcommand{\arraystretch}{0.95}

\begin{longtable}{
    >{\scriptsize}l
    >{\scriptsize}c
    >{\scriptsize}c
    >{\scriptsize}c
    >{\scriptsize}c
}
\caption{Original and bootstrap graph statistics for OOD datasets. Values are bootstrap means \(\pm\) standard deviations across \(25\) bootstrap replicates.}
\label{tab:MLdata-stats}\\
\endfirsthead
\endhead
\endfoot
\endlastfoot

% ============================================================
% CoraML / PubMed
% ============================================================

\toprule
&
\multicolumn{2}{c}{\scriptsize CoraML}
&
\multicolumn{2}{c}{\scriptsize PubMed} \\
\cmidrule(lr){2-3}\cmidrule(lr){4-5}
Metric & Original & Bootstrap & Original & Bootstrap \\
\midrule
Features
& $2879$ & $2879$
& $500$ & $500$ \\
Classes
& $7$ & $7$
& $3$ & $3$ \\
Nodes
& $2995$ & $2995$
& $19717$ & $19717$ \\
Edges
& $8158$ & $9802.200 \pm 82.619$
& $44324$ & $60935.400 \pm 243.491$ \\
Average degree
& $5.450$ & $6.546 \pm 0.055$
& $4.500$ & $6.181 \pm 0.025$ \\
Degree std.
& $8.240$ & $7.458 \pm 0.132$
& $7.430$ & $8.509 \pm 0.049$ \\
Density
& $1.820\times10^{-3}$
& $(2.186\times10^{-3}) \pm (1.843\times10^{-5})$
& $2.280\times10^{-4}$
& $(3.135\times10^{-4}) \pm (1.253\times10^{-6})$ \\
Local clustering
& $0.285$ & $0.062 \pm 0.004$
& $0.061$ & $0.027 \pm 0.003$ \\
Components
& $61$ & $66.360 \pm 6.557$
& $1$ & $1278.800 \pm 37.791$ \\
Giant component
& $2810$ & $2928.680 \pm 6.638$
& $19717$ & $18415.600 \pm 38.176$ \\
Assortativity
& $-0.072$ & $-0.062 \pm 0.004$
& $-0.044$ & $-0.038 \pm 0.005$ \\
Node homophily
& $0.810$ & $0.560 \pm 0.005$
& $0.792$ & $0.626 \pm 0.002$ \\
Edge homophily
& $0.789$ & $0.595 \pm 0.005$
& $0.802$ & $0.696 \pm 0.002$ \\
Transitivity
& $0.116$ & $0.046 \pm 0.002$
& $0.054$ & $0.025 \pm 0.001$ \\
Triangle density
& $1.190\times10^{-6}$
& $(4.712\times10^{-7}) \pm (1.986\times10^{-8})$
& $9.840\times10^{-9}$
& $(6.590\times10^{-9}) \pm (2.318\times10^{-10})$ \\
\bottomrule

\addlinespace[0.8em]

% ============================================================
% CiteSeer / AmazonComputers
% ============================================================

\toprule
&
\multicolumn{2}{c}{\scriptsize CiteSeer}
&
\multicolumn{2}{c}{\scriptsize AmazonComputers} \\
\cmidrule(lr){2-3}\cmidrule(lr){4-5}
Metric & Original & Bootstrap & Original & Bootstrap \\
\midrule
Features
& $3703$ & $3703$
& $767$ & $767$ \\
Classes
& $6$ & $6$
& $10$ & $10$ \\
Nodes
& $3327$ & $3327$
& $13752$ & $13752$ \\
Edges
& $4552$ & $5887.080 \pm 97.053$
& $245861$ & $268828.640 \pm 434.954$ \\
Average degree
& $2.740$ & $3.539 \pm 0.058$
& $35.800$ & $39.097 \pm 0.063$ \\
Degree std.
& $3.380$ & $4.133 \pm 0.092$
& $70.300$ & $66.104 \pm 0.153$ \\
Density
& $8.230\times10^{-4}$
& $(1.064\times10^{-3}) \pm (1.754\times10^{-5})$
& $2.600\times10^{-3}$
& $(2.843\times10^{-3}) \pm (4.600\times10^{-6})$ \\
Local clustering
& $0.150$ & $0.019 \pm 0.002$
& $0.362$ & $0.190 \pm 0.003$ \\
Components
& $438$ & $362.360 \pm 24.088$
& $314$ & $9.640 \pm 3.187$ \\
Giant component
& $2120$ & $2931.480 \pm 28.849$
& $13381$ & $13743.360 \pm 3.187$ \\
Assortativity
& $0.048$ & $0.041 \pm 0.013$
& $-0.057$ & $-0.052 \pm 0.001$ \\
Node homophily
& $0.706$ & $0.333 \pm 0.005$
& $0.785$ & $0.733 \pm 0.001$ \\
Edge homophily
& $0.736$ & $0.447 \pm 0.005$
& $0.777$ & $0.746 \pm 0.001$ \\
Transitivity
& $0.131$ & $0.036 \pm 0.002$
& $0.108$ & $0.090 \pm 0.001$ \\
Triangle density
& $1.910\times10^{-7}$
& $(8.483\times10^{-8}) \pm (7.211\times10^{-9})$
& $3.550\times10^{-6}$
& $(2.790\times10^{-6}) \pm (3.004\times10^{-8})$ \\
\bottomrule

\addlinespace[0.8em]

% ============================================================
% AmazonPhotos / CoauthorCS
% ============================================================

\toprule
&
\multicolumn{2}{c}{\scriptsize AmazonPhotos}
&
\multicolumn{2}{c}{\scriptsize CoauthorCS} \\
\cmidrule(lr){2-3}\cmidrule(lr){4-5}
Metric & Original & Bootstrap & Original & Bootstrap \\
\midrule
Features
& $745$ & $745$
& $6805$ & $6805$ \\
Classes
& $8$ & $8$
& $15$ & $15$ \\
Nodes
& $7650$ & $7650$
& $18333$ & $18333$ \\
Edges
& $119081$ & $133971.160 \pm 265.568$
& $81894$ & $91787.760 \pm 266.237$ \\
Average degree
& $31.100$ & $35.025 \pm 0.069$
& $8.930$ & $10.013 \pm 0.029$ \\
Degree std.
& $47.300$ & $45.657 \pm 0.125$
& $9.110$ & $8.149 \pm 0.043$ \\
Density
& $4.070\times10^{-3}$
& $(4.579\times10^{-3}) \pm (9.077\times10^{-6})$
& $4.870\times10^{-4}$
& $(5.462\times10^{-4}) \pm (1.584\times10^{-6})$ \\
Local clustering
& $0.406$ & $0.254 \pm 0.003$
& $0.356$ & $0.056 \pm 0.002$ \\
Components
& $136$ & $6.400 \pm 2.273$
& $1$ & $55.480 \pm 7.434$ \\
Giant component
& $7487$ & $7644.600 \pm 2.273$
& $18333$ & $18278.240 \pm 7.683$ \\
Assortativity
& $-0.045$ & $-0.031 \pm 0.001$
& $0.113$ & $0.088 \pm 0.004$ \\
Node homophily
& $0.836$ & $0.782 \pm 0.001$
& $0.832$ & $0.638 \pm 0.002$ \\
Edge homophily
& $0.827$ & $0.794 \pm 0.001$
& $0.808$ & $0.665 \pm 0.001$ \\
Transitivity
& $0.178$ & $0.156 \pm 0.001$
& $0.185$ & $0.062 \pm 0.001$ \\
Triangle density
& $9.650\times10^{-6}$
& $(8.917\times10^{-6}) \pm (8.917\times10^{-8})$
& $8.470\times10^{-8}$
& $(2.901\times10^{-8}) \pm (5.275\times10^{-10})$ \\
\bottomrule

\addlinespace[0.8em]

% ============================================================
% CoauthorPhysics / OGBN-ArXiv
% ============================================================

\toprule
&
\multicolumn{2}{c}{\scriptsize CoauthorPhysics}
&
\multicolumn{2}{c}{\scriptsize OGBN-ArXiv} \\
\cmidrule(lr){2-3}\cmidrule(lr){4-5}
Metric & Original & Bootstrap & Original & Bootstrap \\
\midrule
Features
& $8415$ & $8415$
& $128$ & $128$ \\
Classes
& $5$ & $5$
& $40$ & $40$ \\
Nodes
& $34493$ & $34493$
& $169343$ & $169343$ \\
Edges
& $247962$ & $266575.000 \pm 462.087$
& $1.158\times10^{6}$ & $(1.902\times10^{6}) \pm (1.061\times10^{3})$ \\
Average degree
& $14.400$ & $15.457 \pm 0.027$
& $13.700$ & $22.469 \pm 0.013$ \\
Degree std.
& $15.600$ & $13.170 \pm 0.042$
& $68.600$ & $78.516 \pm 0.117$ \\
Density
& $4.170\times10^{-4}$
& $(4.481\times10^{-4}) \pm (7.768\times10^{-7})$
& $8.070\times10^{-5}$
& $(1.327\times10^{-4}) \pm (7.403\times10^{-8})$ \\
Local clustering
& $0.373$ & $0.076 \pm 0.001$
& $0.235$ & $0.060 \pm 0.002$ \\
Components
& $1$ & $28.000 \pm 4.933$
& $1$ & $1123.560 \pm 28.833$ \\
Giant component
& $34493$ & $34465.960 \pm 4.945$
& $169343$ & $168214.880 \pm 29.359$ \\
Assortativity
& $0.201$ & $0.213 \pm 0.003$
& $-0.043$ & $-0.056 \pm 0.000$ \\
Node homophily
& $0.915$ & $0.811 \pm 0.001$
& $0.635$ & $0.429 \pm 0.000$ \\
Edge homophily
& $0.931$ & $0.862 \pm 0.001$
& $0.654$ & $0.435 \pm 0.000$ \\
Transitivity
& $0.189$ & $0.073 \pm 0.001$
& $0.016$ & $0.008 \pm 0.000$ \\
Triangle density
& $6.900\times10^{-8}$
& $(2.421\times10^{-8}) \pm (3.094\times10^{-10})$
& $2.740\times10^{-9}$
& $(1.893\times10^{-9}) \pm (5.433\times10^{-11})$ \\
\bottomrule

\end{longtable}
}
\FloatBarrier

\section{Model Architectures and Parameters}
\label{app:model-params}

\subsection{Graph EDL Models}
For the graph EDL baselines, we consider S-BGCN-K and S-BGCN-T-K~\citep{Zhao-2020}, GPN~\citep{Stadler-2021}, GPN-GD\textsubscript{EVI} and GPN-GD\textsubscript{LAT}~\citep{Hart-2023}, GPN-LOP~\citep{Damke-2024a}, and CUQ-GCN, CUQ-GAT, and CUQ-APPNP~\citep{Damke-2024b}. Unless otherwise specified, we use the same hyperparameters across the graphon information-growth experiments, while for the OOD benchmarks we use the dataset-specific configurations suggested in the original implementations.

\textbf{GPN, GPN-GD, GPN-LOP, CUQ-GNN.} GPN, all GPN-GD variants, GPN-LOP, and all CUQ-GNN variants use a learning rate of \(0.01\) across all datasets and \(10\) radial-flow layers for density estimation. For GPN, all GPN-GD variants, GPN-LOP, and CUQ-APPNP, APPNP propagation uses \(K=10\). CUQ-GCN, CUQ-GAT, and CUQ-APPNP differ only in the graph convolution used before density estimation; CUQ-GAT uses one attention head per layer, negative slope \(0.2\), and no attention-coefficient dropout. GPN-GD uses the same hyperparameters as GPN, and additionally tunes the graph distance regularization weight and activation function, using the same hyperparameters for both variants. Graphon and OOD dataset-specific hyperparameters are reported in Tables~\ref{tab:edl-gpn-hparams} and~\ref{tab:edl-gpn-gd-hparams}.

\textbf{S-BGCN.} We evaluate S-BGCN-K and S-BGCN-T-K. We use a learning rate of \(0.01\) across datasets. S-BGCN-K uses \(\lambda_1\) to weight the KL regularization term, while S-BGCN-T-K additionally uses \(\lambda_2\) to weight the teacher regularization term. The graphon and OOD dataset-specific hyperparameters are reported in Table~\ref{tab:sbgcn-hparams}. Uncertainty estimates are computed using \(100\) MC dropout samples.

\begin{table}[t]
\centering
\scriptsize
\caption{Hyperparameters for GPN, GPN-LOP, and CUQ-GNN.}
\label{tab:edl-gpn-hparams}
\setlength{\tabcolsep}{4pt}
\begin{tabular}{lccccccc}
\toprule
Dataset & Hidden & Latent & Layers & Dropout & APPNP $\alpha$ & Entropy reg. & Weight decay \\
\midrule
Graphons          & 64  & 16 & 2 & 0.50 & 0.10 & $10^{-3}$ & $10^{-3}$ \\
\midrule
CoraML            & 64  & 16 & 2 & 0.50 & 0.10 & $10^{-3}$ & $10^{-3}$ \\
CiteSeer          & 64  & 16 & 2 & 0.50 & 0.10 & $10^{-3}$ & $10^{-3}$ \\
PubMed            & 64  & 16 & 2 & 0.50 & 0.10 & $10^{-3}$ & $10^{-3}$ \\
AmazonComputers   & 64  & 10 & 2 & 0.50 & 0.20 & $10^{-5}$ & $5\times10^{-4}$ \\
AmazonPhotos      & 64  & 10 & 2 & 0.50 & 0.20 & $10^{-5}$ & $5\times10^{-4}$ \\
CoauthorCS        & 64  & 16 & 2 & 0.50 & 0.10 & $10^{-3}$ & $10^{-3}$ \\
CoauthorPhysics   & 64  & 16 & 2 & 0.50 & 0.10 & $10^{-3}$ & $10^{-3}$ \\
OGBN-ArXiv        & 256 & 16 & 3 & 0.25 & 0.20 & $10^{-3}$ & $0$ \\
\bottomrule
\end{tabular}
\end{table}

\begin{table}[t]
\centering
\scriptsize
\caption{Hyperparameters for GPN-GD.}
\label{tab:edl-gpn-gd-hparams}
\begin{tabular}{lccc}
\toprule
Dataset
& Entropy reg.
& Distance reg.
& Activation \\
\midrule
Graphons          & $0$       & $10^{-4}$   & GELU \\
\midrule
CoraML          & $0$       & $10^{-4}$   & GELU \\
CiteSeer        & $10^{-4}$ & $10^{-9.5}$ & LogSigmoid \\
PubMed          & $10^{-5}$ & $10^{-4}$   & ReLU \\
AmazonComputers & $10^{-5}$ & $10^{-4}$   & ReLU \\
AmazonPhotos    & $10^{-5}$ & $10^{-11}$  & ReLU \\
CoauthorCS      & $0$       & $10^{-6}$   & ReLU \\
CoauthorPhysics & $10^{-4}$ & $10^{-4.5}$ & LogSigmoid \\
OGBN-ArXiv      & $10^{-5}$ & $10^{-8}$   & ReLU \\
\bottomrule
\end{tabular}
\end{table}

\begin{table}[t]
\centering
\scriptsize
\caption{Hyperparameters for S-BGCN-K and S-BGCN-T-K.}
\label{tab:sbgcn-hparams}
\setlength{\tabcolsep}{4pt}
\begin{tabular}{lccccc}
\toprule
Dataset & Hidden & Dropout & $L_2$ reg. & $\lambda_1$ & $\lambda_2$ \\
\midrule
Graphons         & 16 & 0.50 & $5\times10^{-4}$ & $10^{-3}$ & $1$ \\
\midrule
CoraML           & 16 & 0.50 & $5\times10^{-4}$ & $10^{-3}$ & $1$ \\
CiteSeer         & 16 & 0.50 & $5\times10^{-4}$ & $10^{-3}$ & $1$ \\
PubMed           & 16 & 0.50 & $5\times10^{-4}$ & $10^{-3}$ & $1$ \\
AmazonComputers  & 64 & 0.20 & $10^{-4}$        & $10^{-3}$ & $1$ \\
AmazonPhotos     & 64 & 0.20 & $10^{-4}$        & $10^{-3}$ & $1$ \\
CoauthorCS       & 64 & 0.10 & $10^{-3}$        & $10^{-3}$ & $1$ \\
CoauthorPhysics  & 64 & 0.10 & $10^{-3}$        & $10^{-3}$ & $1$ \\
OGBN-ArXiv       & 64 & 0.50 & $10^{-4}$        & $10^{-4}$ & $1$ \\
\bottomrule
\end{tabular}
\end{table}

\subsection{G-Ens and GB-Ens}
Unless otherwise specified, G-Ens and GB-Ens use the same hyperparameters for the graphons, while for the OOD benchmarks we perform the same dataset-specific hyperparameter search, selecting model and training hyperparameters by validation accuracy; GB-Ens additionally applies graph bootstrapping.

\textbf{APPNP backbone.}
For the graphon experiments, we use an APPNP architecture~\citep{Gasteiger-2018} consisting of a two-layer MLP with hidden dimension \(64\), dropout \(0.3\), followed by APPNP propagation with \(K=10\) iterations and teleport probability \(\alpha=0.1\). We use a learning rate of \(0.01\) and weight decay \(5\times10^{-4}\). For the OOD datasets, we fix the MLP to two layers and \(K=10\), and search over hidden dimensions in \(\{64,128,256\}\); dropout rates in \(\{0.3,0.5,0.6\}\); teleport probabilities in \(\{0.05,0.1,0.2\}\); learning rates in \(\{0.005,0.01,0.02\}\); and weight decay in \(\{10^{-5},5\times10^{-4},5\times10^{-3}\}\).

\textbf{GAT backbone.}
For the graphon experiments, we use a two-layer GAT~\citep{Velickovic-2018} with total hidden dimension \( 64 \), \( 8 \) attention heads, feature and attention-coefficient dropout rates of \( 0.3 \) and \( 0.2 \), respectively, attention negative slope of \( 0.2 \), learning rate of \( 0.005 \), and weight decay of \( 5\times10^{-4} \). For the OOD datasets, we search over \( 2 \) or \( 3 \) GAT layers; total hidden dimensions in \( \{ 64,128,256 \} \); \( 4 \) or \( 8 \) attention heads; feature dropout rates in \( \{ 0.2,0.3,0.5,0.6 \} \); attention-coefficient dropout rates in \( \{ 0.3,0.5,0.6 \} \); learning rates in \( \{ 0.005,0.01 \} \); and weight decay in \( \{10^{-4},10^{-3}\} \). The attention negative slope is fixed to \( 0.2 \).

\textbf{GCN backbone.}
For the graphon experiments, we use a two-layer GCN~\citep{Kipf-2017} with hidden dimension \(64\), dropout of \(0.3\), learning rate of \(0.01\), and weight decay of \(5 \times 10^{-4}\). For the OOD datasets, we search over \( 2 \) or \( 3 \) GCN layers; hidden dimensions in \( \{64,128,256 \} \); dropout rates in \( \{0.2, 0.4 ,0.5 \} \); learning rates in \( \{0.01,0.02,0.03 \} \); and weight decay in \(\{5 \times10^{-5},5\times10^{-4}, 5\times 10^{-3}\}\).

\section{Additional Experimental Results}
\label{app:sec-exp-proj}

\textbf{RQ2: Graph EDL models are inconsistent epistemic predictors.}

The SBM results in Figures~\ref{fig:EU-bench-graphon2} and~\ref{fig:AU-bench-graphon2} reproduce the patterns observed for the exponential-decay graphon: EDL methods fail to exhibit consistent EU concentration, while AU remains sensitive to the regularization parameter \(\lambda\). Similar differences in EU magnitude across methods persist on real-world datasets, as shown in Table~\ref{tab:EU-EDL-bench}. For completeness, Table~\ref{tab:EDL-acc-graphon} reports predictive accuracy averaged over the two graphon sequences; GPN variants share a common parameterization, as in their original formulations, resulting in lower accuracy for some variants, particularly GPN-GD\textsubscript{EVI}.

\begin{figure}[t]
    \centering

    \begin{subfigure}[t]{0.49\linewidth}
        \centering
        \includegraphics[width=\linewidth]{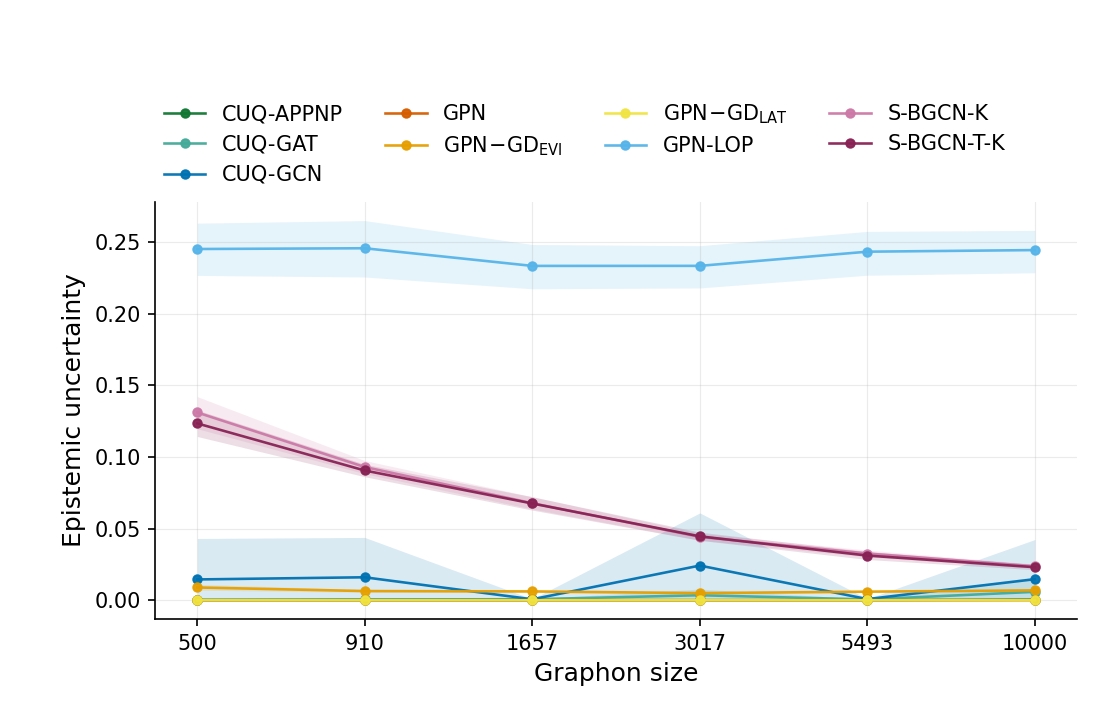}
        \caption{Epistemic uncertainty under information growth.}
        \label{fig:EU-bench-graphon2}
    \end{subfigure}
    \hfill
    \begin{subfigure}[t]{0.49\linewidth}
        \centering
        \includegraphics[width=\linewidth]{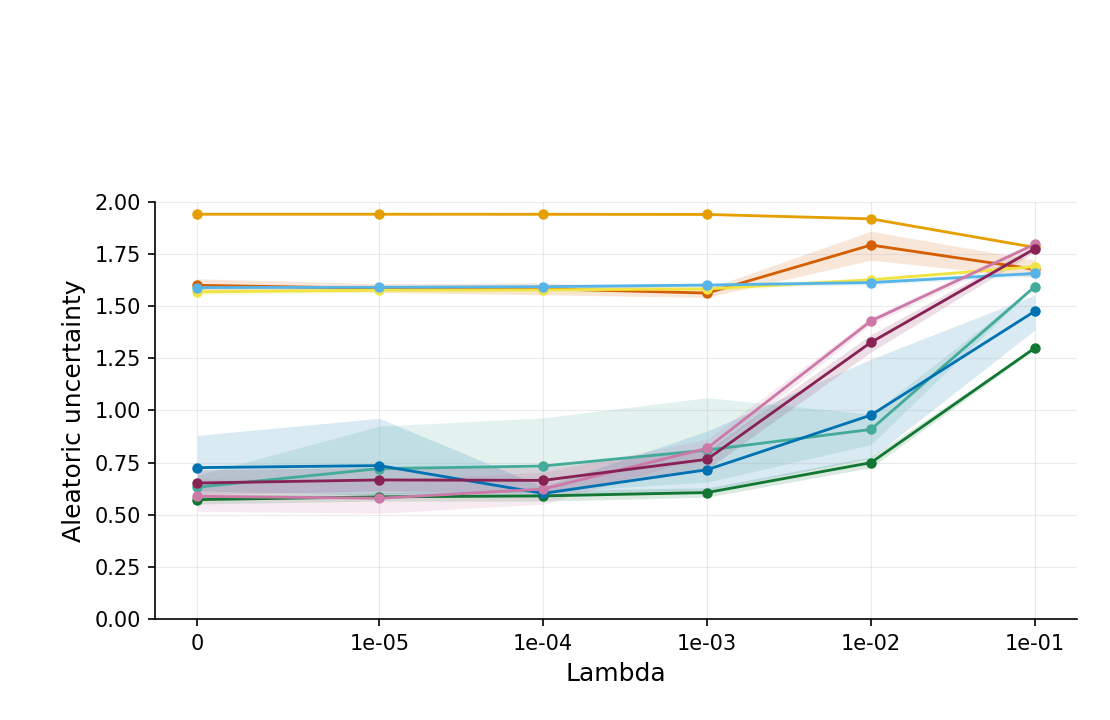}
        \caption{Aleatoric uncertainty across \(\lambda\) at \(N=10{,}000\).}
        \label{fig:AU-bench-graphon2}
    \end{subfigure}

    \caption{Uncertainty estimates of EDL models on the SBM.}
    \label{fig:U-bench-graphons}
\end{figure}

\begin{table}[t]
\centering
\caption{Mean epistemic uncertainty estimated by EDL methods on OOD datasets.}
\label{tab:EU-EDL-bench}

\scriptsize
\setlength{\tabcolsep}{3pt}
\renewcommand{\arraystretch}{0.95}

% ------------------------------------------------
% Block 1
% ------------------------------------------------
\begin{tabular}{lrrrrr}
\toprule
Method
& CoraML
& CiteSeer
& PubMed
& AmazonPhotos
& AmazonComputers \\
\midrule

CUQ-APPNP
& \(4.42 \times 10^{-4}\)
& \(6.64 \times 10^{-4}\)
& \(4.80 \times 10^{-5}\)
& \(8.00 \times 10^{-6}\)
& \(4.00 \times 10^{-6}\) \\

CUQ-GAT
& \(1.87 \times 10^{-4}\)
& \(4.07 \times 10^{-4}\)
& \(2.60 \times 10^{-5}\)
& \(4.20 \times 10^{-5}\)
& \(1.66 \times 10^{-4}\) \\

CUQ-GCN
& \(7.97 \times 10^{-4}\)
& \(7.90 \times 10^{-4}\)
& \(1.79 \times 10^{-4}\)
& \(2.20 \times 10^{-4}\)
& \(5.65 \times 10^{-4}\) \\

GPN
& \(3.65 \times 10^{-4}\)
& \(6.21 \times 10^{-4}\)
& \(5.00 \times 10^{-6}\)
& \(1.00 \times 10^{-5}\)
& \(7.00 \times 10^{-6}\) \\

GPN-GD\textsubscript{EVI}
& \(8.89 \times 10^{-2}\)
& \(5.01 \times 10^{-3}\)
& \(4.44 \times 10^{-2}\)
& \(1.76 \times 10^{-5}\)
& \(1.38 \times 10^{-1}\) \\

GPN-GD\textsubscript{LAT}
& \(2.21 \times 10^{-4}\)
& \(4.77 \times 10^{-3}\)
& \(1.93 \times 10^{-6}\)
& \(8.05 \times 10^{-6}\)
& \(3.32 \times 10^{-6}\) \\

GPN-LOP
& \(3.15 \times 10^{-1}\)
& \(2.09 \times 10^{-1}\)
& \(1.81 \times 10^{-1}\)
& \(2.46 \times 10^{-1}\)
& \(3.23 \times 10^{-1}\) \\

S-BGCN-K
& \(2.69 \times 10^{-2}\)
& \(7.88 \times 10^{-2}\)
& \(4.47 \times 10^{-2}\)
& \(2.66 \times 10^{-3}\)
& \(1.71 \times 10^{-3}\) \\

S-BGCN-T-K
& \(4.30 \times 10^{-2}\)
& \(6.32 \times 10^{-2}\)
& \(4.38 \times 10^{-2}\)
& \(4.36 \times 10^{-3}\)
& \(2.85 \times 10^{-3}\) \\

\bottomrule
\end{tabular}

\vspace{1em}

% ------------------------------------------------
% Block 2
% ------------------------------------------------
\begin{tabular}{lrrr}
\toprule
Method
& CoauthorCS
& CoauthorPhysics
& OGBN-ArXiv \\
\midrule

CUQ-APPNP
& \(1.09 \times 10^{-4}\)
& \(1.44 \times 10^{-4}\)
& \(1.20 \times 10^{-2}\) \\

CUQ-GAT
& \(5.79 \times 10^{-2}\)
& \(2.10 \times 10^{-5}\)
& \(4.10 \times 10^{-1}\) \\

CUQ-GCN
& \(3.48 \times 10^{-2}\)
& \(2.08 \times 10^{-4}\)
& \(2.40 \times 10^{-1}\) \\

GPN
& \(4.00 \times 10^{-6}\)
& \(1.10 \times 10^{-5}\)
& \(5.44 \times 10^{-3}\) \\

GPN-GD\textsubscript{EVI}
& \(1.02 \times 10^{-2}\)
& \(5.31 \times 10^{-2}\)
& \(2.05 \times 10^{-3}\) \\

GPN-GD\textsubscript{LAT}
& \(1.02 \times 10^{-6}\)
& \(2.72 \times 10^{-5}\)
& \(1.05 \times 10^{-5}\) \\

GPN-LOP
& \(5.33 \times 10^{-1}\)
& \(2.65 \times 10^{-1}\)
& OOM \\

S-BGCN-K
& \(1.62 \times 10^{-3}\)
& \(4.35 \times 10^{-3}\)
& \(5.33 \times 10^{-2}\) \\

S-BGCN-T-K
& \(5.62 \times 10^{-3}\)
& \(7.06 \times 10^{-3}\)
& \( 5.68 \times 10^{-2} \) \\

\bottomrule
\end{tabular}

\end{table}
\begin{table}[t]
\centering
\caption{EDL test accuracy averaged across each graphon sequence.}
\label{tab:EDL-acc-graphon}
\scriptsize
\setlength{\tabcolsep}{3.2pt}
\begin{tabular}{lccccccccc}
\toprule
DGP
& CUQ-APPNP
& CUQ-GAT
& CUQ-GCN
& GPN
& GPN-GD$_{\mathrm{EVI}}$
& GPN-GD$_{\mathrm{LAT}}$
& GPN-LOP
& S-BGCN-K
& S-BGCN-T-K \\
\midrule
Exp.-decay
& 0.940
& 0.910
& 0.848
& 0.555
& 0.158
& 0.761
& 0.793
& 0.904
& 0.906 \\
SBM
& 0.822
& 0.832
& 0.800
& 0.786
& 0.159
& 0.821
& 0.834
& 0.823
& 0.823 \\
\bottomrule
\end{tabular}
\end{table}

\textbf{RQ3: GB-Ens is a family of candidate consistent epistemic predictors.}

Complementing the results reported in the main text, Figure~\ref{fig:EU-GBEns-graphon2} shows that on the SBM, GB-Ens EU decreases from \( N=500 \) to \( N=10,000 \) by \( 92\% \) for APPNP and \( 94\% \) for GCN, while GAT exhibits a weaker, non-monotonic EU reduction of \( 34\% \), reaching a maximum of \( 42\% \) at an intermediate graph size. Figure~\ref{fig:EU-deco-GBEns-graphon-2} further decomposes EU into seed and bootstrap contributions of similar magnitude, both decreasing with increasing information. For GAT, the non-monotonicity is driven primarily by the seed component, whereas the bootstrap component follows a more consistently decreasing trajectory. Regarding aleatoric uncertainty, Figure~\ref{fig:AU-Ens-graphon-2} shows that GB-Ens AU decreases and eventually stabilizes, with all three backbones approaching similar plateaus.

Figure~\ref{fig:EU-GEns-graphon-2} shows that G-Ens similarly exhibit no consistent EU contraction on the SBM, while Figure~\ref{fig:AU-Ens-graphon-2} shows strongly backbone-dependent AU trajectories. For completeness, Table~\ref{tab:Ens-acc-graphon} reports predictive accuracy averaged over the two graphon information-growth sequences for G-Ens and GB-Ens. For the true-DGP resampling diagnostic, we compute normalized node-wise Wasserstein-1 distances on a random \( 10\% \) subset of test nodes and average them across nodes and \( 20 \) projective-sequence replicates, using the Euclidean diameter \( \sqrt{2} \) of the probability simplex for normalization. Table~\ref{tab:wasserstein} shows that, on the SBM, distances remain relatively low for GCN (\( 0.12 \)--\( 0.14 \)) and APPNP (\( 0.20 \)--\( 0.34 \)), whereas GAT exhibits larger distances ranging from \( 0.24 \) to \( 0.45 \).

\begin{figure}[t]
    \centering

    \begin{subfigure}[t]{0.49\linewidth}
        \centering
        \includegraphics[width=\linewidth]{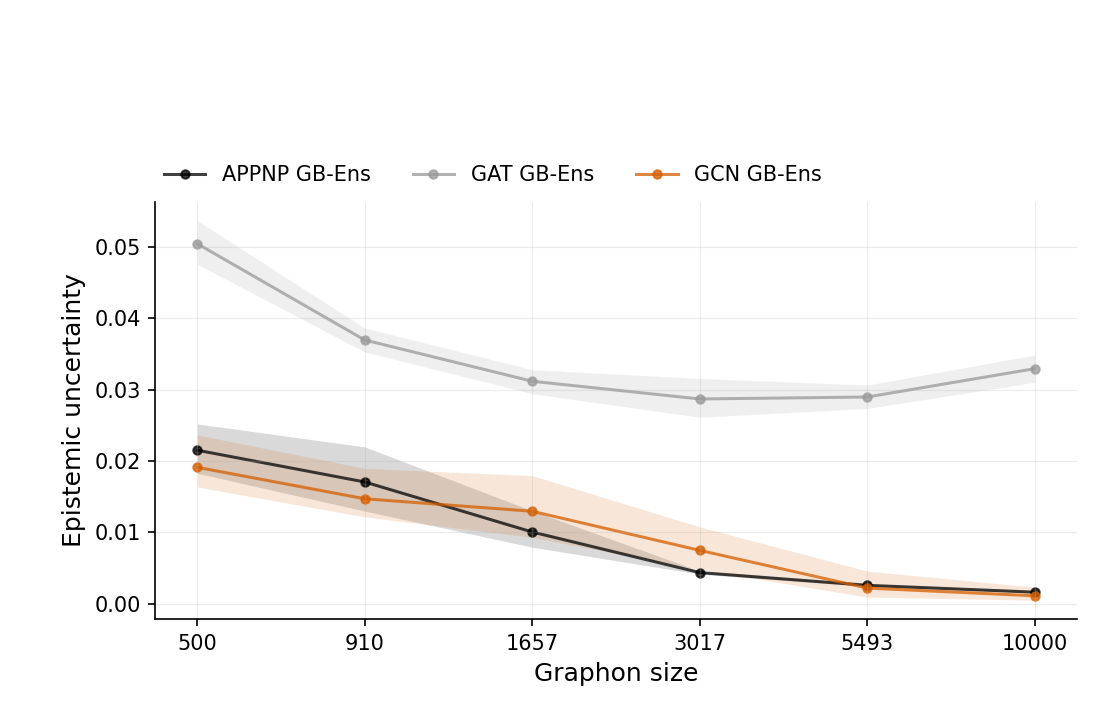}
        \caption{Total epistemic uncertainty.}
        \label{fig:EU-GBEns-graphon2}
    \end{subfigure}
    \hfill
    \begin{subfigure}[t]{0.49\linewidth}
        \centering
        \includegraphics[width=\linewidth]{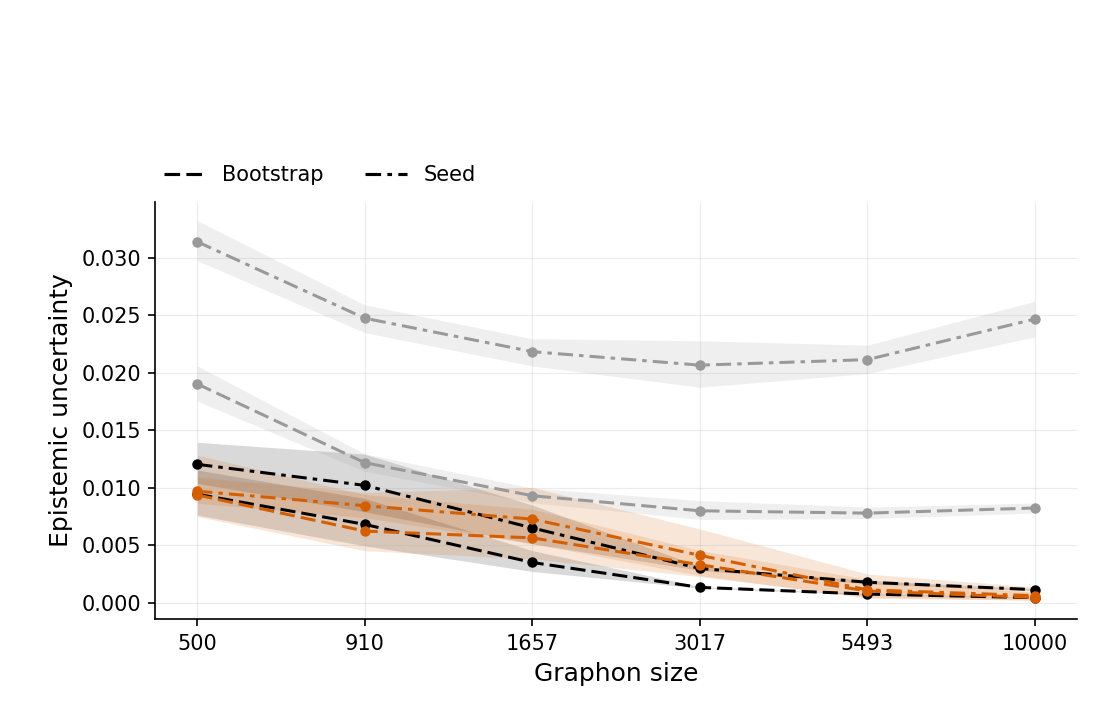}
        \caption{Bootstrap and training-seed decomposition.}
        \label{fig:EU-deco-GBEns-graphon-2}
    \end{subfigure}
    
    \caption{GB-Ens mean epistemic uncertainty on the SBM.}
    \label{fig:EU-GBEns-graphon-2-general}
\end{figure}

\begin{figure}[t]
    \centering
    \begin{subfigure}[t]{0.49\linewidth}
        \centering
        \includegraphics[width=\linewidth]{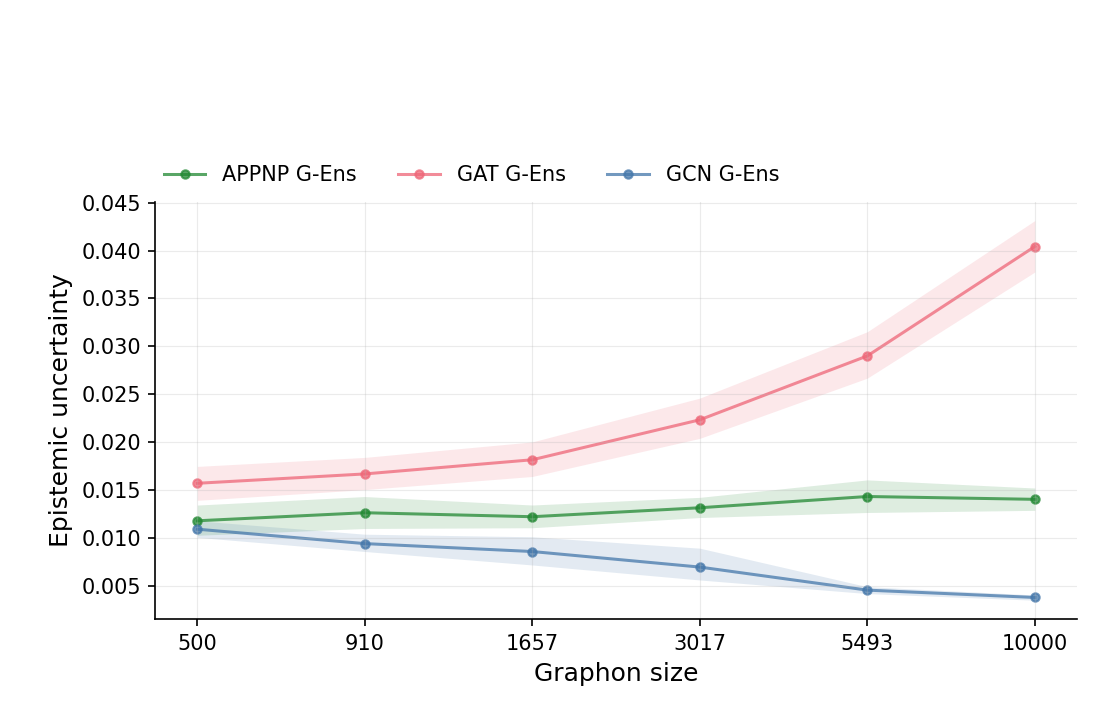}
        \caption{Epistemic uncertainty of G-Ens.}
        \label{fig:EU-GEns-graphon-2}
    \end{subfigure}
    \hfill
    \begin{subfigure}[t]{0.49\linewidth}
        \centering
        \includegraphics[width=\linewidth]{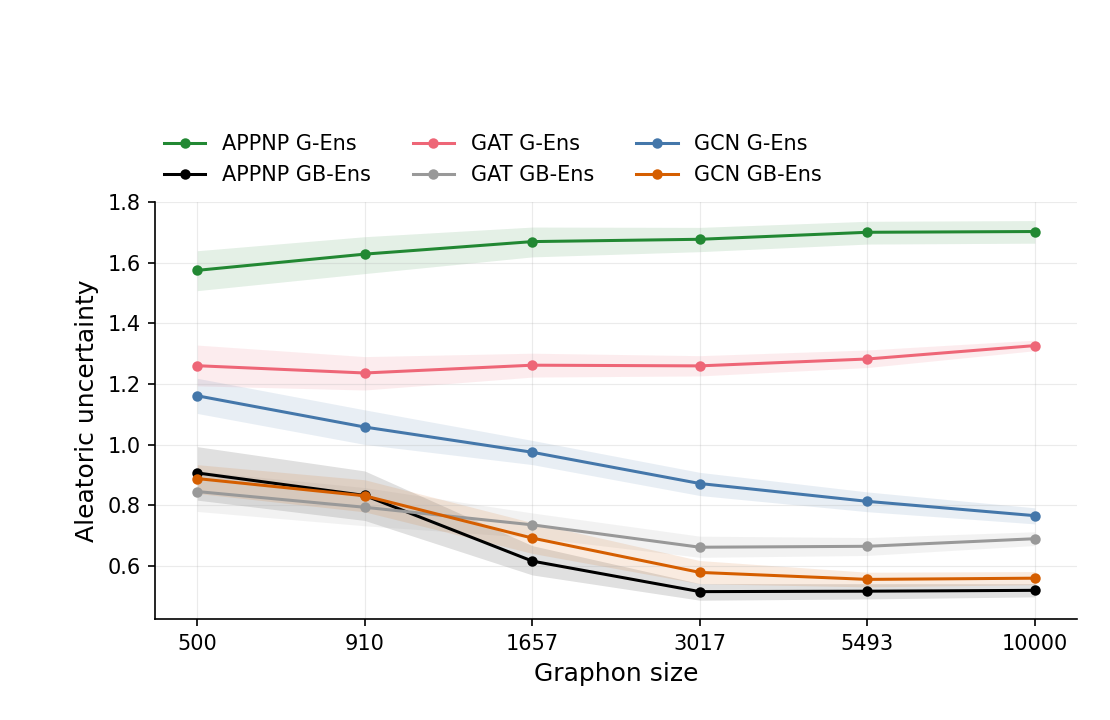}
        \caption{Aleatoric uncertainty of GB-Ens and G-Ens.}
        \label{fig:AU-Ens-graphon-2}
    \end{subfigure}
    
    \caption{Ensemble mean uncertainties on the SBM.}
    \label{fig:wooo}
\end{figure}

\begin{table}[t]
\scriptsize
\centering
\caption{Mean normalized Wasserstein distances between GB-Ens- and true-DGP-induced second-order distributions.}
\label{tab:wasserstein}
\begin{tabular}{llcccccc}
\toprule
& & \multicolumn{6}{c}{Graph size} \\
\cmidrule(lr){3-8}
Graphon & Backbone
& 500 & 910 & 1657 & 3017 & 5493 & 10000 \\
\midrule

\multirow{3}{*}{Exp.-decay}
& APPNP & 0.2991 & 0.2931 & 0.2870 & 0.2886 & 0.2841 & 0.2826 \\
& GAT  & 0.4482 & 0.5102 & 0.5505 & 0.5790 & 0.5946 & 0.6156 \\
& GCN   & 0.2053 & 0.2113 & 0.2119 & 0.2071 & 0.2059 & 0.2050 \\
\midrule

\multirow{3}{*}{SBM}
& APPNP & 0.2049 & 0.2185 & 0.3146 & 0.3483 & 0.3471 & 0.3431 \\
& GAT  & 0.2440 & 0.2998 & 0.3684 & 0.4130 & 0.4365 & 0.4506 \\
& GCN   & 0.1277 & 0.1194 & 0.1437 & 0.1527 & 0.1524 & 0.1480 \\
\bottomrule
\end{tabular}
\end{table}

\begin{table}[t]
\centering
\caption{G-Ens and GB-Ens test accuracy averaged across each graphon sequence.}
\label{tab:Ens-acc-graphon}
\scriptsize
\setlength{\tabcolsep}{4pt}
\begin{tabular}{lcccccc}
\toprule
DGP
& APPNP G-Ens
& GAT G-Ens
& GCN G-Ens
& APPNP GB-Ens
& GAT GB-Ens
& GCN GB-Ens \\
\midrule
Exp.-decay
& 0.808
& 0.731
& 0.902
& 0.946
& 0.931
& 0.932 \\
SBM
& 0.704
& 0.772
& 0.814
& 0.816
& 0.830
& 0.830 \\
\bottomrule
\end{tabular}
\end{table}

%\input{tables/EU_deco_app_bench}

% \section{Misclassification Detection and Accuracy--Rejection}

% \textbf{Misclassification detection.} 
% We treat incorrectly classified ID test nodes as positives and use uncertainty as the detection score. We report AUROC and the area under the precision--recall curve (AUPR), with random-ranking baselines of \(0.5\) and the misclassification prevalence, respectively; higher values are better.

% \textbf{Accuracy--Rejection.}
% Predictions are progressively rejected from highest to lowest uncertainty. We report E-AURC, the excess area under the risk--coverage curve (AURC) over the optimal ranking for a classifier with the same error rate, which accounts for differences in predictive accuracy across models. Lower values are better, with \(0\) indicating optimal ranking.

% Tables~\ref{tab:bench-md-auroc-1}--\ref{tab:bench-md-ap-3} and~\ref{tab:bench-ar-1}--\ref{tab:bench-ar-3} show a consistent pattern across both tasks: total and aleatoric uncertainty generally provide the strongest misclassification detection and rejection rankings, whereas epistemic uncertainty is less effective. This is consistent with previous graph UQ results~\citep{Zhao-2020,Stadler-2021}. No model consistently dominates across datasets, but \textbf{GB-Ens} and its backbone variants remain competitive with the considered baselines across AUROC, AP, and E-AURC, supporting the practical utility of their uncertainty estimates for identifying and rejecting unreliable predictions.

% \input{tables/misclassification_app}
% \input{tables/acc_rej_app}

\section{Out-Of-Distribution Detection}
\label{app:sec-ood}
We report the full OOD detection results for the EDL baselines and, for completeness, extend the evaluation to G-Ens and GB-Ens in Tables~\ref{tab:app-auroc-coraml} to~\ref{tab:app-auroc-ogbn-arxiv}, for which we similarly find that the most effective uncertainty type depends on the model and OOD protocol. Bold highlights the best-performing uncertainty type(s) for each model and shift. GB-Ens bootstrap and original dataset statistics are reported in Table~\ref{tab:MLdata-stats}. The observed discrepancies in graph statistics indicate that the local bootstrap does not preserve all topological properties of the observed graphs, limiting its fidelity as an approximation of graph observation variability; nevertheless, in these downstream OOD experiments uncertainty serves only as a detection score, and this structural degradation does not automatically translate into degraded OOD detection performance. We follow the leave-out-classes (LOC) and feature isolation (FI) protocols of~\citet{Stadler-2021, Damke-2024a, Damke-2024b}.

On OGBN-ArXiv, CUQ-GAT and CUQ-GCN achieve very low accuracy (\(\approx 0.051\)--\(0.073\)) using the hyperparameters recommended in the original paper~\citep{Damke-2024b}, while GPN-LOP exceeds the available 80\,GB of NVIDIA A100 GPU memory; we therefore do not report results for these models.

\textbf{Datasets.} We use CoraML~\citep{McCallum-2000}, CiteSeer~\citep{Giles-1998}, PubMed~\citep{Namata-2012}, CoauthorCS, CoauthorPhysics, AmazonComputers, AmazonPhotos~\citep{Shchur-2019}, and OGBN-ArXiv~\citep{Hu-2021}.

\textbf{Leave-out-classes (LOC).}
Following \citet{Stadler-2021}, we partition the label space \( \gY \) into retained classes \( \gY_{\mathrm{ID}} \) and the same left-out classes \( \gY_{\mathrm{OOD}} \) used in their LOC setup. Nodes from \( \gY_{\mathrm{OOD}} \) are excluded from training and validation but remain in \( \gG \), preserving the transductive setting. For GB-Ens, the same class partition is fixed across bootstrap replicates, with each ensemble member evaluated on its corresponding bootstrap graph.

\textbf{Feature isolation (FI).}
We sample \(n_{\mathrm{iso}}=100\) test nodes and, for each node \(i\), construct a copy of \(\gG\) in which only \(X_i\) is replaced by coordinate-wise Gaussian \(\gN(0,\sigma^2 I_D)\) or Bernoulli \(\mathrm{Bernoulli}(0.5)\) noise. For GB-Ens, the same nodes and feature perturbations are used across bootstrap replicates. Each perturbed node is evaluated in its corresponding graph as OOD; for GB-Ens, this evaluation is performed separately on each member's bootstrap graph, while the unperturbed test set \(\gT^{(\mathrm{te})}\) forms the ID set.

\textbf{Metrics.}
Using total, aleatoric, and epistemic uncertainty as OOD scores, we report AUROC, together with ID accuracy and FI OOD accuracy.

% Tables~\ref{tab:app-auroc-loc-1}--~\ref{tab:app-aupr-loc-3} report LOC detection results, Tables~\ref{tab:app-auroc-fi-normal-1}--~\ref{tab:app-aupr-fi-normal-3} the FI Gaussian results, and Tables~\ref{tab:app-auroc-fi-bernoulli-1}--~\ref{tab:app-aupr-fi-bernoulli-3} the FI Bernoulli results. Under LOC, AU generally provides the strongest OOD discrimination, followed by TU, with GPN as the main exception, for which EU is typically strongest. FI further reveals a strong dependence on the perturbation. Under Bernoulli noise, detection is predominantly driven by TU and AU, with GPN showing the clearest tendency toward stronger EU. Under Gaussian noise, by contrast, EU is the strongest detector for most methods, with LOP-GPN as the main exception. For the baselines, the best-performing uncertainty measure also consistently achieves AUPR above the random-prevalence baseline, whereas this occurs less consistently for \textbf{GB-Ens}, indicating that favorable AUROC does not always concentrate OOD nodes at the high-uncertainty end of the ranking.

\begin{table*}[t]
\centering
\scriptsize
\setlength{\tabcolsep}{3.5pt}
\renewcommand{\arraystretch}{1.08}
\caption{OOD detection on CoraML measured by AUROC.}
\label{tab:app-auroc-coraml}
\begin{tabular}{l c ccc c ccc c ccc}
\toprule
Model & ID ACC & \multicolumn{3}{c}{LOC}
& \multicolumn{4}{c}{FI Gaussian}
& \multicolumn{4}{c}{FI Bernoulli} \\
\cmidrule(lr){3-5}\cmidrule(lr){6-9}\cmidrule(lr){10-13}
& & TU & AU & EU
& OOD ACC & TU & AU & EU
& OOD ACC & TU & AU & EU \\
\midrule
CUQ-APPNP & 0.790 & \textbf{0.761} & \textbf{0.761} & 0.665 & 0.080 & \textbf{1.000} & 0.992 & \textbf{1.000} & 0.760 & \textbf{0.603} & \textbf{0.603} & 0.596 \\
CUQ-GAT & 0.764 & \textbf{0.697} & \textbf{0.697} & 0.562 & 0.070 & 0.996 & 0.996 & \textbf{1.000} & 0.740 & \textbf{0.544} & \textbf{0.544} & 0.537 \\
CUQ-GCN & 0.712 & \textbf{0.715} & \textbf{0.715} & 0.616 & 0.110 & \textbf{1.000} & 0.998 & \textbf{1.000} & 0.700 & \textbf{0.574} & \textbf{0.574} & 0.531 \\
GPN & 0.802 & 0.836 & 0.836 & \textbf{0.837} & 0.790 & 0.559 & 0.559 & \textbf{0.583} & 0.800 & \textbf{0.578} & \textbf{0.578} & 0.577 \\
GPN-GD\textsubscript{EVI} & 0.782 & \textbf{0.892} & 0.862 & 0.879 & 0.760 & \textbf{0.554} & 0.551 & 0.545 & 0.760 & \textbf{0.556} & 0.554 & 0.544 \\
GPN-GD\textsubscript{LAT} & 0.790 & \textbf{0.832} & \textbf{0.832} & 0.824 & 0.790 & 0.555 & 0.555 & \textbf{0.595} & 0.790 & 0.579 & 0.579 & \textbf{0.585} \\
LOP-GPN & 0.795 & 0.837 & \textbf{0.865} & 0.507 & 0.840 & 0.726 & 0.641 & \textbf{0.836} & 0.810 & 0.668 & 0.644 & \textbf{0.673} \\
S-BGCN-K & 0.513 & \textbf{0.865} & 0.864 & 0.462 & 0.260 & 0.028 & 0.002 & \textbf{1.000} & 0.500 & 0.571 & \textbf{0.579} & 0.423 \\
S-BGCN-T-K & 0.786 & 0.885 & \textbf{0.890} & 0.368 & 0.240 & 0.169 & 0.009 & \textbf{1.000} & 0.780 & 0.633 & \textbf{0.653} & 0.362 \\
\midrule
G-Ens\textsubscript{APPNP} & 0.814 & \textbf{0.835} & 0.831 & 0.789 & 0.240 & 0.272 & 0.072 & \textbf{0.668} & 0.725 & 0.667 & \textbf{0.671} & 0.532 \\
G-Ens\textsubscript{GAT} & 0.814 & 0.850 & \textbf{0.857} & 0.621 & 0.200 & 0.217 & 0.002 & \textbf{0.863} & 0.735 & 0.633 & \textbf{0.644} & 0.469 \\
G-Ens\textsubscript{GCN} & 0.816 & \textbf{0.848} & 0.847 & 0.810 & 0.200 & 0.226 & 0.080 & \textbf{0.540} & 0.810 & 0.670 & \textbf{0.671} & 0.614 \\
\midrule
GB-Ens\textsubscript{APPNP} & 0.794 & \textbf{0.807} & 0.799 & 0.726 & 0.240 & 0.304 & 0.025 & \textbf{0.825} & 0.650 & 0.728 & \textbf{0.734} & 0.499 \\
GB-Ens\textsubscript{GAT} & 0.770 & 0.797 & \textbf{0.812} & 0.625 & 0.210 & 0.389 & 0.002 & \textbf{0.989} & 0.680 & 0.658 & \textbf{0.679} & 0.520 \\
GB-Ens\textsubscript{GCN} & 0.687 & \textbf{0.766} & 0.758 & 0.683 & 0.280 & 0.533 & 0.070 & \textbf{0.983} & 0.730 & 0.603 & \textbf{0.626} & 0.479 \\
\bottomrule
\end{tabular}
\end{table*}

\begin{table*}[t]
\centering
\scriptsize
\setlength{\tabcolsep}{3.5pt}
\renewcommand{\arraystretch}{1.08}
\caption{OOD detection on CiteSeer measured by AUROC.}
\label{tab:app-auroc-citeseer}
\begin{tabular}{l c ccc c ccc c ccc}
\toprule
Model & ID ACC & \multicolumn{3}{c}{LOC}
& \multicolumn{4}{c}{FI Gaussian}
& \multicolumn{4}{c}{FI Bernoulli} \\
\cmidrule(lr){3-5}\cmidrule(lr){6-9}\cmidrule(lr){10-13}
& & TU & AU & EU
& OOD ACC & TU & AU & EU
& OOD ACC & TU & AU & EU \\
\midrule
CUQ-APPNP & 0.647 & \textbf{0.751} & \textbf{0.751} & 0.616 & 0.080 & \textbf{1.000} & 0.711 & \textbf{1.000} & 0.570 & \textbf{0.640} & \textbf{0.640} & 0.591 \\
CUQ-GAT & 0.613 & \textbf{0.768} & \textbf{0.768} & 0.679 & 0.100 & 0.941 & 0.701 & \textbf{0.989} & 0.560 & \textbf{0.597} & \textbf{0.597} & \textbf{0.597} \\
CUQ-GCN & 0.557 & 0.593 & \textbf{0.594} & 0.571 & 0.110 & 0.999 & 0.535 & \textbf{1.000} & 0.480 & \textbf{0.576} & \textbf{0.576} & 0.479 \\
GPN & 0.678 & \textbf{0.741} & \textbf{0.741} & 0.615 & 0.690 & 0.489 & 0.487 & \textbf{0.702} & 0.630 & \textbf{0.599} & \textbf{0.599} & 0.569 \\
GPN-GD\textsubscript{EVI} & 0.688 & \textbf{0.807} & \textbf{0.807} & 0.774 & 0.660 & 0.486 & 0.482 & \textbf{0.599} & 0.670 & \textbf{0.568} & \textbf{0.568} & 0.558 \\
GPN-GD\textsubscript{LAT} & 0.687 & \textbf{0.805} & \textbf{0.805} & 0.773 & 0.670 & 0.486 & 0.482 & \textbf{0.598} & 0.680 & \textbf{0.568} & \textbf{0.568} & 0.557 \\
LOP-GPN & 0.680 & \textbf{0.748} & 0.724 & 0.596 & 0.640 & 0.746 & 0.596 & \textbf{0.874} & 0.620 & 0.700 & 0.620 & \textbf{0.713} \\
S-BGCN-K & 0.699 & \textbf{0.739} & 0.737 & 0.413 & 0.280 & 0.159 & 0.002 & \textbf{1.000} & 0.720 & 0.683 & \textbf{0.708} & 0.415 \\
S-BGCN-T-K & 0.723 & 0.861 & \textbf{0.876} & 0.255 & 0.260 & 0.230 & 0.005 & \textbf{1.000} & 0.740 & 0.704 & \textbf{0.731} & 0.393 \\
\midrule
G-Ens\textsubscript{APPNP} & 0.716 & \textbf{0.805} & 0.804 & 0.545 & 0.185 & 0.099 & 0.002 & \textbf{0.767} & 0.540 & 0.832 & \textbf{0.839} & 0.238 \\
G-Ens\textsubscript{GAT} & 0.715 & \textbf{0.799} & 0.793 & 0.565 & 0.180 & 0.121 & 0.005 & \textbf{0.862} & 0.535 & 0.684 & \textbf{0.694} & 0.458 \\
G-Ens\textsubscript{GCN} & 0.706 & \textbf{0.829} & \textbf{0.829} & 0.674 & 0.210 & 0.117 & 0.016 & \textbf{0.581} & 0.720 & 0.701 & \textbf{0.702} & 0.461 \\
\midrule
GB-Ens\textsubscript{APPNP} & 0.676 & \textbf{0.771} & \textbf{0.771} & 0.399 & 0.200 & 0.037 & 0.000 & \textbf{0.777} & 0.370 & 0.892 & \textbf{0.902} & 0.125 \\
GB-Ens\textsubscript{GAT} & 0.666 & \textbf{0.802} & 0.788 & 0.548 & 0.210 & 0.075 & 0.000 & \textbf{1.000} & 0.360 & 0.764 & \textbf{0.781} & 0.418 \\
GB-Ens\textsubscript{GCN} & 0.667 & 0.789 & \textbf{0.794} & 0.355 & 0.160 & 0.036 & 0.000 & \textbf{0.970} & 0.410 & 0.771 & \textbf{0.812} & 0.154 \\
\bottomrule
\end{tabular}
\end{table*}

\begin{table*}[t]
\centering
\scriptsize
\setlength{\tabcolsep}{3.5pt}
\renewcommand{\arraystretch}{1.08}
\caption{OOD detection on PubMed measured by AUROC.}
\label{tab:app-auroc-pubmed}
\begin{tabular}{l c ccc c ccc c ccc}
\toprule
Model & ID ACC & \multicolumn{3}{c}{LOC} & \multicolumn{4}{c}{FI Gaussian} & \multicolumn{4}{c}{FI Bernoulli} \\
\cmidrule(lr){3-5}\cmidrule(lr){6-9}\cmidrule(lr){10-13}
& & TU & AU & EU & OOD ACC & TU & AU & EU & OOD ACC & TU & AU & EU \\
\midrule
CUQ-APPNP & 0.852 & \textbf{0.693} & \textbf{0.693} & 0.643 & 0.300 & 0.975 & 0.942 & \textbf{0.994} & 0.770 & \textbf{0.652} & \textbf{0.652} & 0.616 \\
CUQ-GAT & 0.835 & \textbf{0.709} & \textbf{0.709} & 0.618 & 0.670 & 0.660 & 0.659 & \textbf{0.684} & 0.780 & \textbf{0.586} & \textbf{0.586} & 0.540 \\
CUQ-GCN & 0.845 & \textbf{0.717} & \textbf{0.717} & 0.591 & 0.340 & 0.936 & 0.858 & \textbf{0.945} & 0.770 & \textbf{0.613} & \textbf{0.613} & 0.554 \\
GPN & 0.837 & 0.730 & 0.730 & \textbf{0.769} & 0.820 & 0.560 & 0.560 & \textbf{0.642} & 0.820 & 0.593 & 0.593 & \textbf{0.638} \\
GPN-GD\textsubscript{EVI} & 0.827 & 0.800 & 0.771 & \textbf{0.857} & 0.820 & 0.578 & 0.569 & \textbf{0.621} & 0.820 & 0.585 & 0.579 & \textbf{0.615} \\
GPN-GD\textsubscript{LAT} & 0.842 & 0.724 & 0.724 & \textbf{0.764} & 0.810 & 0.559 & 0.559 & \textbf{0.658} & 0.800 & 0.596 & 0.596 & \textbf{0.640} \\
LOP-GPN & 0.834 & 0.748 & \textbf{0.781} & 0.651 & 0.780 & 0.696 & 0.588 & \textbf{0.767} & 0.770 & 0.578 & 0.446 & \textbf{0.689} \\
S-BGCN-K & 0.848 & 0.625 & \textbf{0.627} & 0.500 & 0.400 & 0.342 & 0.001 & \textbf{1.000} & 0.770 & 0.646 & \textbf{0.679} & 0.288 \\
S-BGCN-T-K & 0.850 & 0.607 & \textbf{0.611} & 0.507 & 0.400 & 0.447 & 0.006 & \textbf{1.000} & 0.790 & 0.654 & \textbf{0.677} & 0.308 \\
\midrule
G-Ens\textsubscript{APPNP} & 0.867 & 0.663 & \textbf{0.664} & 0.641 & 0.385 & 0.138 & 0.028 & \textbf{0.463} & 0.730 & \textbf{0.781} & 0.780 & 0.540 \\
G-Ens\textsubscript{GAT} & 0.845 & \textbf{0.698} & 0.697 & 0.657 & 0.365 & 0.397 & 0.034 & \textbf{0.679} & 0.765 & 0.648 & \textbf{0.652} & 0.560 \\
G-Ens\textsubscript{GCN} & 0.854 & 0.728 & 0.728 & \textbf{0.733} & 0.400 & 0.192 & 0.083 & \textbf{0.375} & 0.820 & \textbf{0.696} & \textbf{0.696} & 0.604 \\
\midrule
GB-Ens\textsubscript{APPNP} & 0.857 & 0.647 & \textbf{0.648} & 0.600 & 0.380 & 0.175 & 0.025 & \textbf{0.536} & 0.640 & \textbf{0.775} & 0.773 & 0.551 \\
GB-Ens\textsubscript{GAT} & 0.845 & \textbf{0.676} & 0.675 & 0.657 & 0.320 & 0.573 & 0.029 & \textbf{0.962} & 0.730 & 0.661 & \textbf{0.668} & 0.576 \\
GB-Ens\textsubscript{GCN} & 0.829 & 0.618 & \textbf{0.626} & 0.568 & 0.400 & 0.261 & 0.060 & \textbf{0.780} & 0.720 & 0.641 & \textbf{0.649} & 0.514 \\
\bottomrule
\end{tabular}
\end{table*}

\begin{table*}[t]
\centering
\scriptsize
\setlength{\tabcolsep}{3.5pt}
\renewcommand{\arraystretch}{1.08}
\caption{OOD detection on Amazon Computers measured by AUROC.}
\label{tab:app-auroc-amazon-computers}
\begin{tabular}{l c ccc c ccc c ccc}
\toprule
Model & ID ACC & \multicolumn{3}{c}{LOC} & \multicolumn{4}{c}{FI Gaussian} & \multicolumn{4}{c}{FI Bernoulli} \\
\cmidrule(lr){3-5}\cmidrule(lr){6-9}\cmidrule(lr){10-13}
& & TU & AU & EU & OOD ACC & TU & AU & EU & OOD ACC & TU & AU & EU \\
\midrule
CUQ-APPNP & 0.838 & \textbf{0.717} & \textbf{0.717} & 0.589 & 0.010 & \textbf{1.000} & \textbf{1.000} & \textbf{1.000} & 0.720 & \textbf{0.698} & \textbf{0.698} & 0.543 \\
CUQ-GAT & 0.875 & \textbf{0.773} & \textbf{0.773} & 0.719 & 0.170 & 0.908 & 0.907 & \textbf{0.909} & 0.820 & \textbf{0.503} & \textbf{0.503} & 0.484 \\
CUQ-GCN & 0.831 & \textbf{0.632} & \textbf{0.632} & 0.538 & 0.210 & 0.896 & 0.895 & \textbf{0.942} & 0.790 & \textbf{0.526} & \textbf{0.526} & 0.494 \\
GPN & 0.810 & 0.749 & 0.749 & \textbf{0.762} & 0.720 & 0.551 & 0.551 & \textbf{0.562} & 0.730 & 0.554 & 0.554 & \textbf{0.560} \\
GPN-GD\textsubscript{EVI} & 0.381 & 0.408 & 0.382 & \textbf{0.511} & 0.910 & 0.515 & 0.515 & \textbf{0.582} & 0.430 & 0.497 & \textbf{0.505} & 0.482 \\
GPN-GD\textsubscript{LAT} & 0.832 & \textbf{0.768} & \textbf{0.768} & 0.607 & 0.900 & 0.515 & 0.515 & \textbf{0.580} & 0.800 & 0.567 & 0.567 & \textbf{0.583} \\
LOP-GPN & 0.792 & 0.825 & \textbf{0.876} & 0.507 & 0.720 & \textbf{0.891} & 0.818 & 0.871 & 0.740 & 0.713 & 0.637 & \textbf{0.717} \\
S-BGCN-K & 0.396 & 0.752 & \textbf{0.753} & 0.593 & 0.400 & 0.235 & 0.164 & \textbf{0.960} & 0.440 & 0.491 & 0.491 & \textbf{0.509} \\
S-BGCN-T-K & 0.626 & 0.785 & \textbf{0.787} & 0.581 & 0.320 & 0.140 & 0.078 & \textbf{0.980} & 0.660 & \textbf{0.507} & 0.506 & 0.502 \\
\midrule
G-Ens\textsubscript{APPNP} & 0.875 & 0.773 & \textbf{0.774} & 0.751 & 0.140 & 0.359 & 0.064 & \textbf{0.580} & 0.785 & \textbf{0.737} & \textbf{0.737} & 0.697 \\
G-Ens\textsubscript{GAT} & 0.852 & 0.810 & \textbf{0.817} & 0.760 & 0.170 & 0.866 & 0.120 & \textbf{0.980} & 0.830 & 0.537 & \textbf{0.538} & 0.511 \\
G-Ens\textsubscript{GCN} & 0.854 & 0.727 & 0.722 & \textbf{0.781} & 0.370 & 0.501 & 0.375 & \textbf{0.759} & 0.820 & 0.507 & \textbf{0.508} & 0.476 \\
\midrule
GB-Ens\textsubscript{APPNP} & 0.871 & 0.762 & \textbf{0.764} & 0.717 & 0.130 & 0.447 & 0.059 & \textbf{0.739} & 0.790 & \textbf{0.747} & \textbf{0.747} & 0.698 \\
GB-Ens\textsubscript{GAT} & 0.857 & 0.817 & \textbf{0.827} & 0.763 & 0.150 & 0.937 & 0.107 & \textbf{1.000} & 0.850 & \textbf{0.538} & 0.537 & 0.510 \\
GB-Ens\textsubscript{GCN} & 0.835 & 0.727 & \textbf{0.730} & 0.656 & 0.540 & 0.537 & 0.439 & \textbf{0.819} & 0.770 & 0.503 & \textbf{0.504} & 0.481 \\
\bottomrule
\end{tabular}
\end{table*}

\begin{table*}[t]
\centering
\scriptsize
\setlength{\tabcolsep}{3.5pt}
\renewcommand{\arraystretch}{1.08}
\caption{OOD detection on Amazon Photos measured by AUROC.}
\label{tab:app-auroc-amazon-photos}
\begin{tabular}{l c ccc c ccc c ccc}
\toprule
Model & ID ACC & \multicolumn{3}{c}{LOC} & \multicolumn{4}{c}{FI Gaussian} & \multicolumn{4}{c}{FI Bernoulli} \\
\cmidrule(lr){3-5}\cmidrule(lr){6-9}\cmidrule(lr){10-13}
& & TU & AU & EU & OOD ACC & TU & AU & EU & OOD ACC & TU & AU & EU \\
\midrule
CUQ-APPNP & 0.901 & 0.720 & \textbf{0.721} & 0.694 & 0.100 & \textbf{1.000} & 0.999 & \textbf{1.000} & 0.800 & \textbf{0.693} & \textbf{0.693} & 0.612 \\
CUQ-GAT & 0.915 & \textbf{0.842} & \textbf{0.842} & 0.823 & 0.220 & 0.927 & 0.927 & \textbf{0.936} & 0.870 & \textbf{0.510} & \textbf{0.510} & 0.497 \\
CUQ-GCN & 0.912 & \textbf{0.802} & \textbf{0.802} & 0.738 & 0.130 & 0.948 & 0.947 & \textbf{0.953} & 0.900 & \textbf{0.503} & \textbf{0.503} & 0.474 \\
GPN & 0.900 & 0.835 & 0.835 & \textbf{0.892} & 0.910 & 0.522 & 0.522 & \textbf{0.568} & 0.910 & 0.552 & 0.552 & \textbf{0.565} \\
GPN-GD\textsubscript{EVI} & 0.900 & 0.812 & 0.812 & \textbf{0.855} & 0.410 & 0.712 & \textbf{0.719} & 0.689 & 0.910 & 0.563 & 0.563 & \textbf{0.573} \\
GPN-GD\textsubscript{LAT} & 0.899 & 0.810 & 0.810 & \textbf{0.855} & 0.790 & 0.561 & 0.561 & \textbf{0.583} & 0.900 & 0.554 & 0.554 & \textbf{0.579} \\
LOP-GPN & 0.892 & 0.865 & \textbf{0.897} & 0.569 & 0.880 & 0.763 & 0.683 & \textbf{0.951} & 0.870 & 0.679 & 0.584 & \textbf{0.848} \\
S-BGCN-K & 0.446 & 0.629 & \textbf{0.630} & 0.359 & 0.250 & 0.126 & 0.071 & \textbf{0.976} & 0.360 & 0.520 & \textbf{0.522} & 0.477 \\
S-BGCN-T-K & 0.717 & 0.664 & \textbf{0.665} & 0.361 & 0.210 & 0.196 & 0.084 & \textbf{0.997} & 0.640 & 0.551 & \textbf{0.553} & 0.488 \\
\midrule
G-Ens\textsubscript{APPNP} & 0.930 & 0.843 & \textbf{0.844} & 0.807 & 0.135 & 0.481 & 0.145 & \textbf{0.681} & 0.900 & 0.725 & \textbf{0.726} & 0.715 \\
G-Ens\textsubscript{GAT} & 0.916 & 0.786 & \textbf{0.801} & 0.703 & 0.290 & 0.895 & 0.289 & \textbf{0.999} & 0.900 & 0.497 & 0.498 & \textbf{0.520} \\
G-Ens\textsubscript{GCN} & 0.921 & \textbf{0.847} & \textbf{0.847} & 0.839 & 0.240 & 0.331 & 0.270 & \textbf{0.590} & 0.900 & 0.530 & \textbf{0.531} & 0.489 \\
\midrule
GB-Ens\textsubscript{APPNP} & 0.926 & 0.777 & \textbf{0.778} & 0.740 & 0.140 & 0.504 & 0.164 & \textbf{0.737} & 0.890 & 0.718 & \textbf{0.719} & 0.701 \\
GB-Ens\textsubscript{GAT} & 0.887 & 0.832 & \textbf{0.866} & 0.810 & 0.270 & 0.906 & 0.251 & \textbf{1.000} & 0.890 & 0.471 & 0.474 & \textbf{0.486} \\
GB-Ens\textsubscript{GCN} & 0.862 & 0.739 & \textbf{0.741} & 0.497 & 0.230 & 0.158 & 0.137 & \textbf{0.791} & 0.860 & 0.526 & 0.527 & \textbf{0.531} \\
\bottomrule
\end{tabular}
\end{table*}

\begin{table*}[t]
\centering
\scriptsize
\setlength{\tabcolsep}{3.5pt}
\renewcommand{\arraystretch}{1.08}
\caption{OOD detection on Coauthor CS measured by AUROC.}
\label{tab:app-auroc-coauthor-cs}
\begin{tabular}{l c ccc c ccc c ccc}
\toprule
Model & ID ACC & \multicolumn{3}{c}{LOC} & \multicolumn{4}{c}{FI Gaussian} & \multicolumn{4}{c}{FI Bernoulli} \\
\cmidrule(lr){3-5}\cmidrule(lr){6-9}\cmidrule(lr){10-13}
& & TU & AU & EU & OOD ACC & TU & AU & EU & OOD ACC & TU & AU & EU \\
\midrule
CUQ-APPNP & 0.910 & \textbf{0.770} & \textbf{0.770} & 0.741 & 0.060 & \textbf{1.000} & \textbf{1.000} & \textbf{1.000} & 0.830 & \textbf{0.553} & \textbf{0.553} & 0.507 \\
CUQ-GAT & 0.787 & 0.676 & \textbf{0.678} & 0.639 & 0.580 & \textbf{0.663} & 0.645 & \textbf{0.663} & 0.820 & \textbf{0.492} & \textbf{0.492} & 0.452 \\
CUQ-GCN & 0.797 & \textbf{0.730} & \textbf{0.730} & 0.658 & 0.110 & \textbf{0.979} & 0.940 & \textbf{0.979} & 0.770 & 0.538 & \textbf{0.541} & 0.519 \\
GPN & 0.884 & 0.755 & 0.755 & \textbf{0.931} & 0.860 & 0.531 & 0.531 & \textbf{0.557} & 0.860 & 0.537 & 0.537 & \textbf{0.550} \\
GPN-GD\textsubscript{EVI} & 0.872 & 0.771 & 0.760 & \textbf{0.919} & 0.840 & 0.522 & 0.522 & \textbf{0.539} & 0.840 & 0.524 & 0.524 & \textbf{0.538} \\
GPN-GD\textsubscript{LAT} & 0.885 & 0.716 & 0.716 & \textbf{0.928} & 0.880 & \textbf{0.541} & \textbf{0.541} & 0.515 & 0.880 & 0.545 & 0.545 & \textbf{0.556} \\
LOP-GPN & 0.886 & 0.706 & \textbf{0.883} & 0.149 & 0.880 & 0.678 & 0.629 & \textbf{0.723} & 0.870 & 0.620 & \textbf{0.626} & 0.588 \\
S-BGCN-K & 0.350 & 0.659 & \textbf{0.660} & 0.551 & 0.240 & 0.148 & 0.112 & \textbf{0.924} & 0.440 & 0.465 & 0.465 & \textbf{0.517} \\
S-BGCN-T-K & 0.709 & 0.642 & \textbf{0.643} & 0.543 & 0.190 & 0.144 & 0.090 & \textbf{0.994} & 0.730 & \textbf{0.542} & \textbf{0.542} & 0.488 \\
\midrule
G-Ens\textsubscript{APPNP} & 0.934 & \textbf{0.940} & \textbf{0.940} & 0.934 & 0.105 & 0.724 & 0.312 & \textbf{0.889} & 0.875 & \textbf{0.736} & \textbf{0.736} & 0.713 \\
G-Ens\textsubscript{GAT} & 0.915 & 0.914 & \textbf{0.920} & 0.855 & 0.155 & 0.805 & 0.156 & \textbf{0.958} & 0.870 & 0.580 & \textbf{0.583} & 0.551 \\
G-Ens\textsubscript{GCN} & 0.926 & 0.924 & 0.924 & \textbf{0.939} & 0.180 & 0.558 & 0.374 & \textbf{0.788} & 0.880 & \textbf{0.584} & \textbf{0.584} & 0.566 \\
\midrule
GB-Ens\textsubscript{APPNP} & 0.924 & 0.826 & \textbf{0.828} & 0.779 & 0.100 & 0.750 & 0.289 & \textbf{0.963} & 0.870 & \textbf{0.753} & 0.752 & 0.745 \\
GB-Ens\textsubscript{GAT} & 0.888 & 0.728 & \textbf{0.757} & 0.607 & 0.130 & 0.872 & 0.082 & \textbf{1.000} & 0.850 & 0.570 & \textbf{0.576} & 0.536 \\
GB-Ens\textsubscript{GCN} & 0.865 & 0.657 & \textbf{0.670} & 0.410 & 0.200 & 0.191 & 0.114 & \textbf{0.810} & 0.800 & \textbf{0.544} & 0.543 & 0.541 \\
\bottomrule
\end{tabular}
\end{table*}

\begin{table*}[t]
\centering
\scriptsize
\setlength{\tabcolsep}{3.5pt}
\renewcommand{\arraystretch}{1.08}
\caption{OOD detection on Coauthor Physics measured by AUROC.}
\label{tab:app-auroc-coauthor-physics}
\begin{tabular}{l c ccc c ccc c ccc}
\toprule
Model & ID ACC & \multicolumn{3}{c}{LOC} & \multicolumn{4}{c}{FI Gaussian} & \multicolumn{4}{c}{FI Bernoulli} \\
\cmidrule(lr){3-5}\cmidrule(lr){6-9}\cmidrule(lr){10-13}
& & TU & AU & EU & OOD ACC & TU & AU & EU & OOD ACC & TU & AU & EU \\
\midrule
CUQ-APPNP & 0.948 & \textbf{0.861} & \textbf{0.861} & 0.814 & 0.240 & \textbf{1.000} & 0.999 & \textbf{1.000} & 0.920 & \textbf{0.594} & \textbf{0.594} & 0.544 \\
CUQ-GAT & 0.945 & \textbf{0.848} & \textbf{0.848} & 0.784 & 0.260 & 0.998 & 0.996 & \textbf{1.000} & 0.930 & \textbf{0.487} & \textbf{0.487} & 0.417 \\
CUQ-GCN & 0.944 & \textbf{0.826} & \textbf{0.826} & 0.782 & 0.300 & \textbf{0.992} & \textbf{0.992} & \textbf{0.992} & 0.900 & \textbf{0.524} & \textbf{0.524} & 0.474 \\
GPN & 0.938 & 0.876 & 0.876 & \textbf{0.957} & 0.900 & 0.562 & 0.562 & \textbf{0.607} & 0.900 & \textbf{0.609} & \textbf{0.609} & 0.600 \\
GPN-GD\textsubscript{EVI} & 0.934 & 0.912 & 0.897 & \textbf{0.957} & 0.920 & 0.613 & 0.609 & \textbf{0.643} & 0.920 & 0.618 & 0.615 & \textbf{0.641} \\
GPN-GD\textsubscript{LAT} & 0.926 & 0.902 & 0.902 & \textbf{0.961} & 0.920 & 0.547 & 0.547 & \textbf{0.581} & 0.920 & 0.561 & 0.561 & \textbf{0.578} \\
LOP-GPN & 0.931 & 0.832 & \textbf{0.877} & 0.698 & 0.890 & 0.751 & 0.702 & \textbf{0.767} & 0.870 & 0.701 & \textbf{0.712} & 0.634 \\
S-BGCN-K & 0.863 & \textbf{0.905} & \textbf{0.905} & 0.814 & 0.430 & 0.284 & 0.197 & \textbf{0.959} & 0.820 & \textbf{0.593} & \textbf{0.593} & 0.568 \\
S-BGCN-T-K & 0.944 & \textbf{0.885} & \textbf{0.885} & 0.826 & 0.360 & 0.400 & 0.290 & \textbf{0.965} & 0.960 & \textbf{0.601} & \textbf{0.601} & 0.581 \\
\midrule
G-Ens\textsubscript{APPNP} & 0.960 & \textbf{0.940} & \textbf{0.940} & 0.934 & 0.210 & 0.409 & 0.279 & \textbf{0.514} & 0.960 & 0.750 & \textbf{0.751} & 0.684 \\
G-Ens\textsubscript{GAT} & 0.953 & 0.914 & \textbf{0.920} & 0.855 & 0.270 & 0.617 & 0.257 & \textbf{0.695} & 0.960 & \textbf{0.606} & \textbf{0.606} & 0.600 \\
G-Ens\textsubscript{GCN} & 0.956 & 0.924 & 0.924 & \textbf{0.939} & 0.330 & 0.574 & 0.467 & \textbf{0.730} & 0.960 & \textbf{0.610} & \textbf{0.610} & 0.592 \\
\midrule
GB-Ens\textsubscript{APPNP} & 0.954 & 0.928 & \textbf{0.931} & 0.896 & 0.210 & 0.750 & 0.420 & \textbf{0.868} & 0.900 & 0.763 & \textbf{0.764} & 0.745 \\
GB-Ens\textsubscript{GAT} & 0.932 & 0.907 & \textbf{0.916} & 0.844 & 0.280 & 0.817 & 0.307 & \textbf{0.990} & 0.900 & \textbf{0.600} & \textbf{0.600} & 0.595 \\
GB-Ens\textsubscript{GCN} & 0.924 & 0.885 & \textbf{0.892} & 0.733 & 0.400 & 0.407 & 0.362 & \textbf{0.701} & 0.900 & \textbf{0.604} & \textbf{0.604} & 0.597 \\
\bottomrule
\end{tabular}
\end{table*}

\begin{table*}[t]
\centering
\scriptsize
\setlength{\tabcolsep}{3.5pt}
\renewcommand{\arraystretch}{1.08}
\caption{OOD detection on OGBN-ArXiv measured by AUROC.}
\label{tab:app-auroc-ogbn-arxiv}
\begin{tabular}{l c ccc c ccc c ccc}
\toprule
Model & ID ACC & \multicolumn{3}{c}{LOC} & \multicolumn{4}{c}{FI Gaussian} & \multicolumn{4}{c}{FI Bernoulli} \\
\cmidrule(lr){3-5}\cmidrule(lr){6-9}\cmidrule(lr){10-13}
& & TU & AU & EU & OOD ACC & TU & AU & EU & OOD ACC & TU & AU & EU \\
\midrule
CUQ-APPNP & 0.697 & \textbf{0.685} & \textbf{0.685} & 0.662 & 0.390 & \textbf{0.823} & \textbf{0.823} & 0.811 & 0.680 & 0.581 & 0.578 & \textbf{0.676} \\
CUQ-GAT & N/A & 0.395 & \textbf{0.405} & 0.387 & N/A & N/A & N/A & N/A & N/A & N/A & N/A & N/A \\
CUQ-GCN & N/A & \textbf{0.523} & 0.516 & \textbf{0.523} & N/A & N/A & N/A & N/A & N/A & N/A & N/A & N/A \\
GPN & 0.696 & 0.706 & 0.705 & \textbf{0.747} & 0.660 & 0.540 & 0.539 & \textbf{0.604} & 0.660 & 0.540 & 0.539 & \textbf{0.604} \\
GPN-GD\textsubscript{EVI} & 0.675 & 0.690 & 0.690 & \textbf{0.730} & 0.680 & 0.534 & 0.534 & \textbf{0.601} & 0.680 & 0.533 & 0.533 & \textbf{0.601} \\
GPN-GD\textsubscript{LAT} & 0.701 & 0.689 & 0.689 & \textbf{0.753} & 0.690 & 0.548 & 0.548 & \textbf{0.597} & 0.690 & 0.548 & 0.548 & \textbf{0.597} \\
LOP-GPN & OOM & OOM & OOM & OOM & OOM & OOM & OOM & OOM & OOM & OOM & OOM & OOM \\
S-BGCN-K & 0.622 & 0.647 & \textbf{0.648} & 0.605 & 0.420 & 0.540 & 0.469 & \textbf{0.775} & 0.640 & 0.591 & \textbf{0.601} & 0.438 \\
S-BGCN-T-K & 0.645 & 0.637 & \textbf{0.639} & 0.590 & 0.480 & 0.600 & 0.499 & \textbf{0.792} & 0.660 & 0.586 & \textbf{0.597} & 0.450 \\
\midrule
G-Ens\textsubscript{APPNP} & 0.706 & 0.642 & 0.641 & \textbf{0.708} & 0.340 & 0.540 & 0.508 & \textbf{0.942} & 0.740 & 0.598 & \textbf{0.598} & 0.442 \\
G-Ens\textsubscript{GAT} & 0.710 & 0.695 & 0.694 & \textbf{0.702} & 0.550 & 0.651 & 0.601 & \textbf{0.888} & 0.720 & 0.529 & \textbf{0.529} & 0.511 \\
G-Ens\textsubscript{GCN} & 0.714 & 0.667 & 0.666 & \textbf{0.705} & 0.580 & 0.596 & 0.584 & \textbf{0.680} & 0.710 & \textbf{0.509} & 0.509 & 0.493 \\
\midrule
GB-Ens\textsubscript{APPNP} & 0.661 & 0.607 & 0.589 & \textbf{0.608} & 0.150 & 0.279 & \textbf{0.904} & 0.364 & 0.530 & 0.716 & 0.614 & \textbf{0.718} \\
GB-Ens\textsubscript{GAT} & 0.663 & \textbf{0.658} & 0.551 & 0.649 & 0.150 & 0.232 & \textbf{0.966} & 0.470 & 0.590 & 0.555 & 0.542 & \textbf{0.558} \\
GB-Ens\textsubscript{GCN} & 0.644 & \textbf{0.636} & 0.598 & 0.634 & 0.380 & 0.514 & \textbf{0.701} & 0.558 & 0.610 & \textbf{0.573} & \textbf{0.573} & 0.572 \\
\bottomrule
\end{tabular}
\end{table*}

\section{Computational Complexity}
\label{app:sec-comp-comp}
Graph EDL methods can themselves incur substantial computational costs, with different methods trading off preprocessing, training, inference, and memory requirements. Ensemble-based methods require repeated model fitting, but retain manageable per-model computation and peak memory and are naturally parallelizable. GPN and CUQ-GCN admit \( O(N+M) \) implementations, whereas GPN-LOP and S-BGCN-K can incur \( O(N^2) \) computation due to dense propagation and pairwise GKDE prior construction, respectively. For GCN-based ensembles, training costs \( O(N+M) \) per member; GB-Ens additionally requires bootstrap generation, whose cost depends on the \(k\)-NN construction and can reach \( O(N^2) \) for general similarity search, although cheaper alternatives are available, such as the open Jaccard similarity metric used in our experiments. Bootstrap generation is repeated for each bootstrap sample. Across ensemble members, training cost scales linearly for both GB-Ens and G-Ens but is naturally parallelizable. Moreover, processing members independently and accumulating uncertainty statistics online keeps peak memory at \( O(N+M) \), comparable to GPN, CUQ-GNN, and S-BGCN-K, whereas GPN-LOP can require \( O(N^2) \) memory due to dense propagation.

Tables~\ref{tab:compl-CoraML}, \ref{tab:compl-CoPhysics} and ~\ref{tab:compl-OGBN-ArXiv} complement the asymptotic analysis with empirical measurements on CoraML, CoauthorPhysics and OGBN-ArXiv, obtained on a dedicated NVIDIA A100 64\,GB GPU. GB-Ens uses 100 models in the \(25 \times 4\) bootstrap--seed decomposition, while G-Ens uses 10 randomly initialized models. Preprocessing time reports the total preprocessing cost required by each method, including generation of all bootstrap samples for GB-Ens. Epoch time is reported per model and training epoch, while inference time reports the total cost of producing one uncertainty estimate, aggregating the 100 forward passes required by GB-Ens and S-BGCN variants and the 10 forward passes required by G-Ens. Peak inference memory is measured when processing ensemble members sequentially. GB-Ens therefore trades increased total preprocessing, training, and inference costs for moderate peak memory requirements: graph bootstrapping becomes particularly costly on larger graphs, while its 100 members can be trained independently and in parallel, and sequential inference limits peak memory. This trade-off differs across EDL baselines: some avoid repeated model fitting and inference, whereas preprocessing-heavy methods such as S-BGCN-K incur comparable preprocessing costs on OGBN-ArXiv, and dense propagation makes GPN-LOP considerably more expensive on CoauthorPhysics and infeasible on OGBN-ArXiv under the available GPU memory.

%%%%%%%%%%%%%%%%%%%%%%%
%%%%%%% CORAML %%%%%%%%
%%%%%%%%%%%%%%%%%%%%%%%
\begin{table}[htbp]
\centering\scriptsize
\caption{Computational cost on CoraML.}
\label{tab:compl-CoraML}
\setlength{\tabcolsep}{5pt}
\begin{tabular}{lrrrrrrr}
\toprule
Model & Params & Models & Preproc.\ (s) & Epoch (ms) & Fwd & Infer.\ (ms) & Peak (MiB) \\
\midrule

CUQ-APPNP & 186,620 & 1 & -- & 17.0 & 1 & 7.0 & 268 \\

CUQ-GAT & 195,196 & 1 & -- & 16.9 & 1 & 6.6 & 253 \\

CUQ-GCN & 194,940 & 1 & -- & 16.4 & 1 & 6.9 & 235 \\

GPN & 186,620 & 1 & -- & 16.4 & 1 & 6.8 & 173 \\

GPN-GD\textsubscript{EVI} & 186,620 & 1 & -- & 17.8 & 1 & 7.0 & 202 \\

GPN-GD\textsubscript{LAT} & 186,620 & 1 & -- & 16.7 & 1 & 6.9 & 188 \\

GPN-LOP & 186,620 & 1 & -- & 24.7 & 1 & 15.3 & 320 \\

% S-BGCN & 46,199 & 1 & -- & 2.0 & 100 & 80.0 & 91 \\

S-BGCN-K & 46,199 & 1 & 0.1 & 2.4 & 100 & 80.0 & 92 \\

S-BGCN-T-K & 46,199 & 1 & 0.5 & 2.7 & 100 & 80.0 & 92 \\

\midrule

G-Ens\textsubscript{APPNP} & 184,775 & 10 & -- & 3.5 & 10 & 17.0 & 55 \\

G-Ens\textsubscript{GAT} & 184,917 & 10 & -- & 3.3 & 10 & 13.0 & 64 \\

G-Ens\textsubscript{GCN} & 184,775 & 10 & -- & 2.1 & 10 & 9.0 & 177 \\

\midrule

GB-Ens\textsubscript{APPNP} & 184,775 & 100 & 32.5 & 3.6 & 100 & 170.0 & 167 \\

GB-Ens\textsubscript{GAT} & 184,917 & 100 & 32.5 & 3.3 & 100 & 140.0 & 176 \\

GB-Ens\textsubscript{GCN} & 184,775 & 100 & 32.5 & 2.1 & 100 & 100.0 & 175 \\

\bottomrule
\end{tabular}
\end{table}

%%%%%%%%%%%%%%%%%%%%%%%
%%% COAUTHORPHYSICS %%%
%%%%%%%%%%%%%%%%%%%%%%%
\begin{table}[htbp]
\centering\scriptsize
\caption{Computational cost on CoauthorPhysics.}
\label{tab:compl-CoPhysics}
\setlength{\tabcolsep}{5pt}
\begin{tabular}{lrrrrrrr}
\toprule
Model & Params & Models & Preproc.\ (s) & Epoch (ms) & Fwd & Infer.\ (ms) & Peak (MiB) \\
\midrule

CUQ-APPNP & 540,564 & 1 & -- & 46.6 & 1 & 24.6 & 2,886 \\

CUQ-GAT & 549,140 & 1 & -- & 36.4 & 1 & 19.0 & 2,760 \\

CUQ-GCN & 548,884 & 1 & -- & 34.9 & 1 & 18.8 & 2,633 \\

GPN & 540,564 & 1 & -- & 32.2 & 1 & 17.1 & 2,127 \\

GPN-GD\textsubscript{EVI} & 540,564 & 1 & -- & 32.3 & 1 & 16.1 & 2,127 \\

GPN-GD\textsubscript{LAT} & 540,564 & 1 & -- & 31.9 & 1 & 16.2 & 2,249 \\

GPN-LOP & 540,564 & 1 & -- & 1835.1 & 1 & 1824.7 & 28,643 \\

% S-BGCN & 538,949 & 1 & -- & 11.5 & 100 & 620.0 & 2,535 \\

S-BGCN-K & 538,949 & 1 & 18.4 & 12.0 & 100 & 620.0 & 2,540 \\

S-BGCN-T-K & 538,949 & 1 & 20.5 & 11.7 & 100 & 620.0 & 2,545 \\

\midrule

G-Ens\textsubscript{APPNP} & 2,155,781 & 10 & -- & 21.0 & 10 & 91.0 & 1,883 \\

G-Ens\textsubscript{GAT} & 1,078,159 & 10 & -- & 18.6 & 10 & 72.0 & 2,370 \\

G-Ens\textsubscript{GCN} & 2,155,781 & 10 & -- & 26.0 & 10 & 119.0 & 2,905 \\

\midrule

GB-Ens\textsubscript{APPNP} & 2,155,781 & 100 & 617.5 & 21.0 & 100 & 900.0 & 1,858 \\

GB-Ens\textsubscript{GAT} & 1,078,159 & 100 & 617.5 & 18.5 & 100 & 720.0 & 2,370 \\

GB-Ens\textsubscript{GCN} & 2,155,781 & 100 & 617.5 & 26.0 & 100 & 1180.0 & 2,905 \\

\bottomrule
\end{tabular}
\end{table}

%%%%%%%%%%%%%%%%%%%%%%%
%%%%%% OGBN-ARXIV %%%%%
%%%%%%%%%%%%%%%%%%%%%%%
\begin{table}[htbp]
\centering\scriptsize
\caption{Computational cost on OGBN-ArXiv.}
\label{tab:compl-OGBN-ArXiv}
\setlength{\tabcolsep}{5pt}
\begin{tabular}{lrrrrrrr}
\toprule
Model & Params & Models & Preproc.\ (s) & Epoch (ms) & Fwd & Infer.\ (ms) & Peak (MiB) \\
\midrule

CUQ-APPNP & 111,152 & 1 & -- & 778.6 & 1 & 505.4 & 48,500 \\

CUQ-GAT & 243,760 & 1 & -- & 930.0 & 1 & 380.3 & 43,686 \\

CUQ-GCN & 242,736 & 1 & -- & 537.0 & 1 & 379.7 & 38,870 \\

GPN & 111,152 & 1 & -- & 509.5 & 1 & 367.4 & 38,664 \\

GPN-GD\textsubscript{EVI} & 111,152 & 1 & -- & 575.0 & 1 & 367.4 & 47,412 \\

GPN-GD\textsubscript{LAT} & 111,152 & 1 & -- & 522.6 & 1 & 367.0 & 42,598 \\

GPN-LOP & OOM & OOM & OOM & OOM & OOM & OOM & OOM \\

% S-BGCN & 10,856 & 1 & -- & 20.1 & 100 & 770.0 & 1,568 \\

S-BGCN-K & 10,856 & 1 & 7331.9 & 22.0 & 100 & 770.0 & 1,594 \\

S-BGCN-T-K & 10,856 & 1 & 7328.0 & 21.9 & 100 & 780.0 & 1,620 \\

\midrule

G-Ens\textsubscript{APPNP} & 110,120 & 10 & -- & 67.2 & 10 & 3040.0 & 1,024 \\

G-Ens\textsubscript{GAT} & 43,896 & 10 & -- & 67.3 & 10 & 2540.0 & 5,452 \\

G-Ens\textsubscript{GCN} & 109,096 & 10 & -- & 80.8 & 10 & 3820.0 & 5,541 \\

\midrule

GB-Ens\textsubscript{APPNP} & 110,120 & 100 & 8362.5 & 67.3 & 100 & 2970.0 & 1,025 \\

GB-Ens\textsubscript{GAT} & 43,896 & 100 & 8362.5 & 67.1 & 100 & 2550.0 & 5,452 \\

GB-Ens\textsubscript{GCN} & 109,096 & 100 & 8362.5 & 81.3 & 100 & 3820.0 & 5,541 \\

\bottomrule
\end{tabular}
\end{table}

\end{document}